\documentclass{article}
\usepackage{tocloft}

\usepackage{float}
\usepackage{algpseudocode}
\usepackage{iftex}
\usepackage[T1]{fontenc}
\ifPDFTeX
  \usepackage[utf8]{inputenc}
\fi
\usepackage[dvipsnames,svgnames, table,xcdraw]{xcolor}
\usepackage{relsize}
\usepackage{tcolorbox}
\usepackage{needspace}
\usepackage{placeins}
\tcbuselibrary{skins,breakable}
\usepackage{listings}
\usepackage{subcaption}
\usepackage{algorithm}
\usepackage{comment}
\usepackage{graphicx}
\usepackage{booktabs}
\usepackage{multirow}
\usepackage{todonotes}
\usepackage{enumitem}
\usepackage{url}
\usepackage{mathpazo}
\usepackage{amsmath,amssymb,amsthm,amsfonts,bm}
\usepackage[toc,page,header]{appendix}
\usepackage{etoc}
\newcommand{\makeappendixtoc}{%
  \begingroup
  \etocsettocstyle{\subsection*{Contents}\vspace{0.5em}}{}%
  \etocsetnexttocdepth{subsection}%
  \localtableofcontents
  \endgroup
  \clearpage
}
\usepackage{fancyhdr}
\usepackage[backend=biber,style=nature,natbib=true,maxbibnames=99,minalphanames=3]{biblatex}
\usepackage[colorlinks=true]{hyperref}
\usepackage{xspace}

\usepackage{auth_detailed}

\counterwithout{figure}{section}
\counterwithout{table}{section}
\definecolor{darkgoldenrod}{rgb}{0.72, 0.53, 0.04}
\definecolor{backgroundcolor}{RGB}{250, 250, 252}   
\definecolor{keywordcolor}{RGB}{30, 0, 178}       
\definecolor{stringcolor}{RGB}{204, 0, 102}        
\definecolor{numbercolor}{RGB}{0, 128, 128}        
\definecolor{emphcolor}{RGB}{30, 0, 178}            
\definecolor{commentcolor}{RGB}{0, 128, 0}       
\definecolor{basiccodecolor}{RGB}{61, 61, 61}       

\lstdefinestyle{customstyle}{
    backgroundcolor=\color{backgroundcolor},   
    commentstyle=\color{commentcolor},
    keywordstyle=\color{keywordcolor},
    numberstyle=\color{numbercolor},
    stringstyle=\color{stringcolor},
    basicstyle=\color{basiccodecolor}\ttfamily\footnotesize,
    breakatwhitespace=false,         
    breaklines=true,                 
    captionpos=b,                    
    keepspaces=true,                 
    numbers=left,     
    basicstyle=\color{basiccodecolor}\ttfamily\footnotesize,
    numbersep=5pt,             
    xleftmargin=2em,
    xrightmargin=2em,
    showspaces=false,                
    showstringspaces=false,
    showtabs=false,                  
    tabsize=1,
    frame=single,
    framesep=5pt,
    framexleftmargin=1.5em,
    framexrightmargin=1.5em,
    framextopmargin=1pt,
    framexbottommargin=1pt,
    aboveskip=10pt,
    belowskip=10pt,
    breaklines=true,
    breakautoindent=true,
    emph={textgrad, tg, Variable, MultipleChoiceTestTime,
    TextualGradientDescent, BlackboxLLM},             %
    emphstyle={\color{emphcolor}},
    extendedchars=true,
    literate={×}{{$\times$}}1 {–}{{--}}1 {—}{{---}}1 {’}{{'}}1 {“}{{``}}1 {”}{{''}}1 {…}{{\ldots}}1 {≤}{{$\le$}}1 {ó}{{\'o}}1 {ő}{{\H{o}}}1 {‱}{{$10^{-4}$}}1
}

\usepackage[normalem]{ulem}
\usepackage{xcolor}

\newcommand{\scrcmd}{\textsuperscript}
\newcommand{\scr}[1]{\scrcmd{#1}}

\definecolor{logocolor}{RGB}{30, 0, 178}                %

\definecolor{darkerlogocolor}{RGB}{20, 0, 145}  

\newtcolorbox{ttcolorbox}[1][]{colframe=darkerlogocolor, colback=darkerlogocolor!4!white, title=#1}

\newtcolorbox{apxtcolorbox}[1][]{colframe=black, colback=black!3!white, title=#1}

\usepackage[nameinlink,noabbrev]{cleveref}
\usepackage{textcomp}
\newcommand{\method}{\textsc{SimpleTES}\xspace}
\newcommand{\best}[1]{\textbf{#1}}
\newcommand{\second}[1]{\underline{#1}}
\theoremstyle{plain}
\newtheorem{theorem}{Theorem}[section]
\theoremstyle{definition}
\newtheorem{definition}[theorem]{Definition}

\definecolor{problemblue}{RGB}{33,99,235}
\definecolor{problembluebg}{RGB}{245,249,255}

\newcounter{taskdefinition}[subsection]
\renewcommand{\thetaskdefinition}{\thesubsection.\arabic{taskdefinition}}

\crefname{taskdefinition}{Problem}{Problems}
\Crefname{taskdefinition}{Problem}{Problems}

\newcommand{\taskboxreserve}[1]{\par\Needspace{#1\baselineskip}}
\newenvironment{taskdefinition}[1][]{%
  \refstepcounter{taskdefinition}%
  \begin{tcolorbox}[
    enhanced,
    unbreakable,
    colback=problembluebg,
    colframe=problemblue,
    boxrule=0.8pt,
    arc=2mm,
    left=2mm,
    right=2mm,
    top=1.2mm,
    bottom=1.2mm,
    fonttitle=\bfseries,
    coltitle=black,
    title={Problem~\thetaskdefinition
      \if\relax\detokenize{#1}\relax\else\;(#1)\fi},
    attach boxed title to top left={xshift=2mm,yshift*=-2mm},
    boxed title style={
      colback=problemblue!12,
      colframe=problemblue,
      boxrule=0.8pt,
      arc=1.5mm
    }
  ]%
}{%
  \end{tcolorbox}%
}
\ifPDFTeX
  \DeclareUnicodeCharacter{00B5}{\ensuremath{\mu}}
  \DeclareUnicodeCharacter{00D7}{\ensuremath{\times}}
  \DeclareUnicodeCharacter{00F3}{\'{o}}
  \DeclareUnicodeCharacter{0151}{\H{o}}
  \DeclareUnicodeCharacter{03BC}{\ensuremath{\mu}}
  \DeclareUnicodeCharacter{2013}{--}
  \DeclareUnicodeCharacter{2014}{---}
  \DeclareUnicodeCharacter{2018}{`}
  \DeclareUnicodeCharacter{2019}{'}
  \DeclareUnicodeCharacter{201C}{``}
  \DeclareUnicodeCharacter{201D}{''}
  \DeclareUnicodeCharacter{2026}{\ldots}
  \DeclareUnicodeCharacter{2031}{\ensuremath{10^{-4}}}
  \DeclareUnicodeCharacter{2264}{\ensuremath{\le}}
\fi

\usepackage{amsmath,amsfonts,bm}

\def\1{\bm{1}}

\usepackage{textcomp}

\DeclareMathAlphabet{\mathsfit}{\encodingdefault}{\sfdefault}{m}{sl}
\SetMathAlphabet{\mathsfit}{bold}{\encodingdefault}{\sfdefault}{bx}{n}

\DeclareMathOperator*{\argmin}{arg\,min}

\usepackage{willreport_arxiv}
\hypersetup{
  pdftitle={Structured Scaling of AI Discovery Across Diverse Scientific Domains},
  pdfauthor={Haotian Ye et al.}
}

\graphicspath{{sections/figures/}{figures/}}
\reporttitle{Structured Scaling of AI Discovery Across Diverse Scientific Domains}
\reporttagline{Structured Scaling of AI Discovery Across Diverse Scientific Domains}
\reportgithub{https://www.wizardquant.com/will/simpletes}{Homepage}

\reportauthors{}

\reportaffiliations{%
{
Wizard Intelligence Learning Lab, Stanford University\\
Peking University, Tsinghua University\\
The Hong Kong University of Science and Technology (Guangzhou)
}}

\reportabstract{
Scientific discovery often requires many cycles of proposing, testing, and refining candidate solutions. Language models can increasingly participate in these loops, but simply generating more attempts does not ensure progress: parallel searches may duplicate one another and iterative refinement may become trapped in poor directions. The central challenge is therefore not only to scale AI-driven discovery, but to structure that scaling so that evaluation signals compound over time.
Here we introduce \method (\textbf{Simple} \textbf{T}est-time \textbf{E}valuation-driven \textbf{S}caling), a framework that focuses on the structured scaling of AI discovery loops, organizing evaluator queries across independent trajectories, iterative refinement, local candidate selection, and the selective reuse of evaluated histories.
Drawing on structural features of scientific communities, \method uses a single open-source GPT-OSS model to establish new state-of-the-art solutions across 28 open-ended problems in diverse scientific domains ranging from quantum physics and astronomy to biology, AI, and mathematics. These include a 24.5\% reduction in quantum circuit compilation overhead, up to 23\% lower propulsive cost for deep-space trajectories, a $2.17\times$ faster lasso-path solver, an 8.5\% lower-error whole-brain neural-activity predictor, the fastest reported TriMul kernel, and new mathematical constructions beyond prior human or AI records. We further post-train the model for long-horizon discovery by assigning each attempt the final outcome of the trajectory it helped produce. This improves performance on both training and held-out mathematics problems, further advancing the frontier. Together, these results establish structured scaling as a general mechanism for advancing AI scientific discovery.

}

\begin{document}

\begin{tikzpicture}[remember picture,overlay]
  \node[anchor=north west, xshift=1in, yshift=-0.60in]
    at (current page.north west)
    {\includegraphics[height=1.18em]{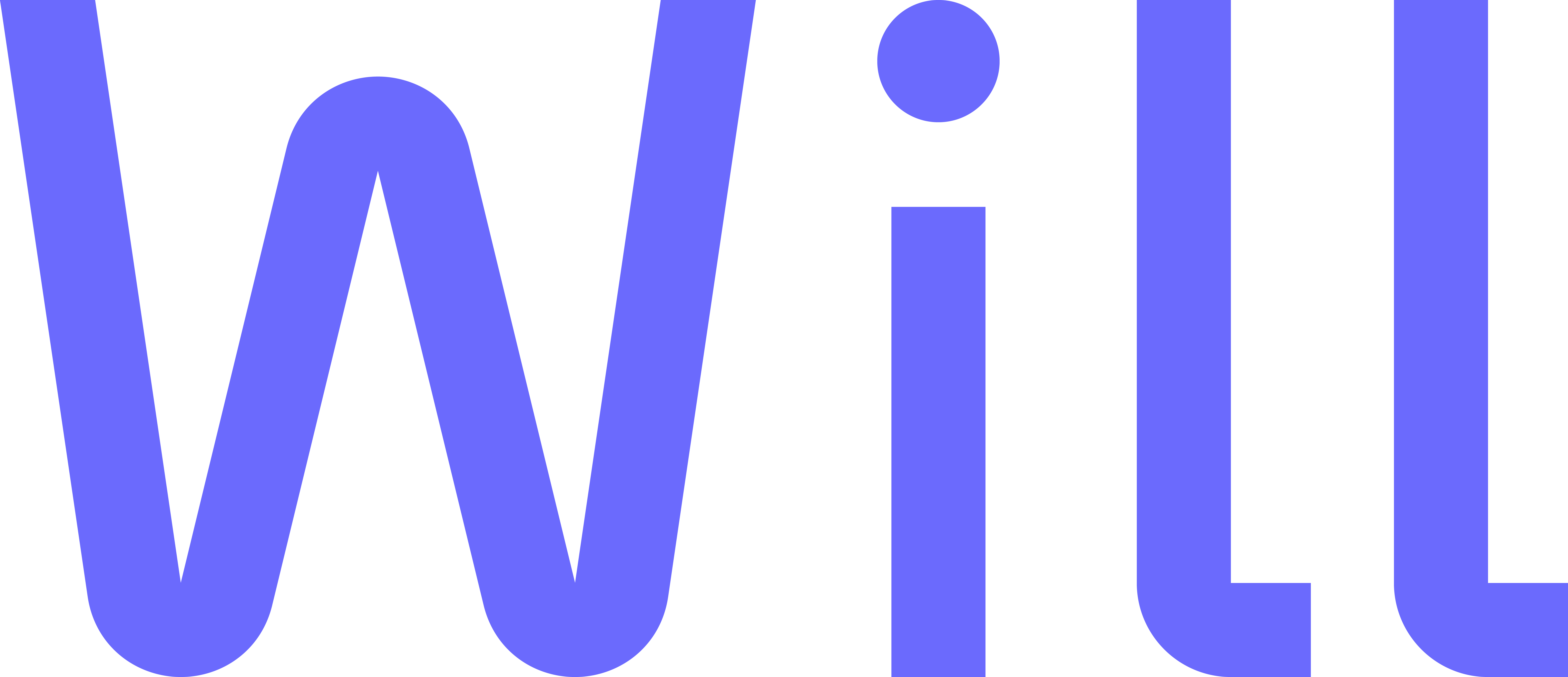}};
  \node[anchor=north east, xshift=-1.05in, yshift=-0.55in]
    at (current page.north east)
    {\includegraphics[height=1.58em]{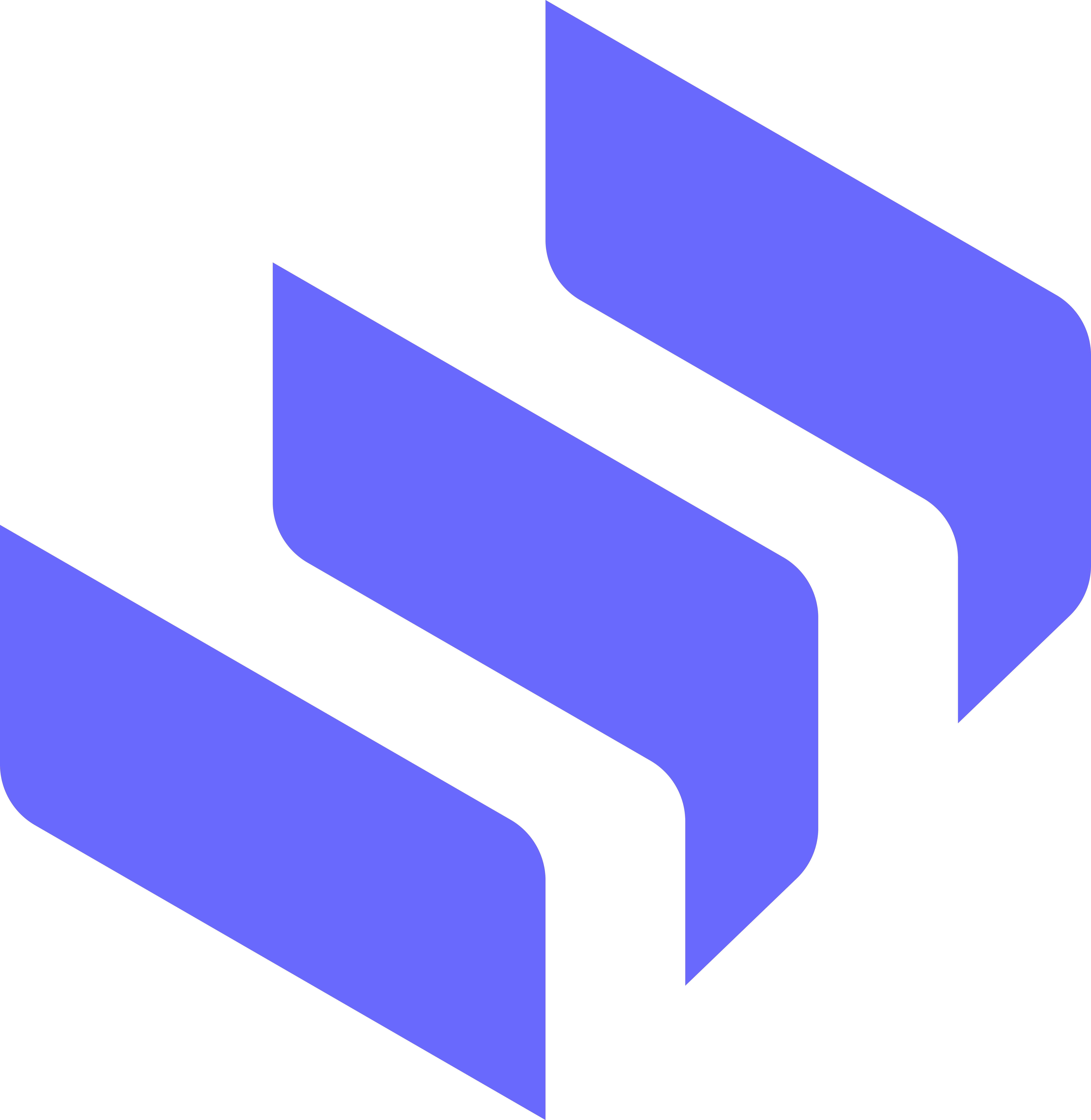}};
  \draw[line width=0.4pt, color=gray!60]
    ([xshift=1in, yshift=-0.88in] current page.north west) --
    ([xshift=-1in, yshift=-0.88in] current page.north east);
\end{tikzpicture}

\makewilltitle
\setcounter{tocdepth}{3}
\tableofcontents
\clearpage

\newcounter{suppfigure}
\newcounter{supptable}
\makeatletter
\newcommand\suppfigurename{Supplementary Figure}
\newcommand\supptablename{Supplementary Table}
\newcommand\suppfigureautorefname{\suppfigurename}
\newcommand\supptableautorefname{\supptablename}
\let\oldappendix\appendix
\renewcommand\appendix{%
    \oldappendix
    \setcounter{figure}{0}%
    \setcounter{table}{0}%
    \renewcommand\figurename{\suppfigurename}%
    \renewcommand\tablename{\supptablename}%
}
\makeatother

\phantomsection
\addcontentsline{toc}{section}{Main}

\NewDocumentCommand{\WILLInputMainWithTOC}{m}{%
  \begingroup
  \let\WILLorigsection\section
  \let\WILLorigsubsection\subsection
  \let\WILLorigsubsubsection\subsubsection
  \RenewDocumentCommand{\section}{s o m}{%
    \IfBooleanTF{##1}{%
      \phantomsection
      \WILLorigsection*{##3}%
      \addcontentsline{toc}{subsection}{##3}%
    }{%
      \IfNoValueTF{##2}{\WILLorigsection{##3}}{\WILLorigsection[##2]{##3}}%
    }%
  }%
  \RenewDocumentCommand{\subsection}{s o m}{%
    \IfBooleanTF{##1}{%
      \phantomsection
      \WILLorigsubsection*{##3}%
      \addcontentsline{toc}{subsubsection}{##3}%
    }{%
      \IfNoValueTF{##2}{\WILLorigsubsection{##3}}{\WILLorigsubsection[##2]{##3}}%
    }%
  }%
  \RenewDocumentCommand{\subsubsection}{s o m}{%
    \IfBooleanTF{##1}{%
      \phantomsection
      \WILLorigsubsubsection*{##3}%
      \addcontentsline{toc}{paragraph}{##3}%
    }{%
      \IfNoValueTF{##2}{\WILLorigsubsubsection{##3}}{\WILLorigsubsubsection[##2]{##3}}%
    }%
  }%
  \input{#1}%
  \endgroup
}

\WILLInputMainWithTOC{sections/main_paper}

\clearpage
\printbibliography
\clearpage

\WILLInputMainWithTOC{sections/main_method}

\newpage
\appendix
\phantomsection
\addcontentsline{toc}{section}{Supplementary Information}
\section*{Supplementary Information}
\renewcommand{\thesubsection}{\Alph{subsection}}
\makeappendixtoc
\subsection{Authors}
\label{app:authors}

Contributors to this work are organized by working group.

\begingroup
\setlength{\tabcolsep}{4pt}
\renewcommand{\arraystretch}{1}
\begin{center}
\begin{tabularx}{\textwidth}{@{}XXXX@{}}
\begin{minipage}[t]{\linewidth}
\raggedright
\textbf{Core contributors}\par\vspace{0.2em}
Haotian Ye\textsuperscript{\ensuremath{\dagger}}\par
Haowei Lin\par
Jingyi Tang\par
Yizhen Luo\par
Rahul Thapa
\end{minipage}
&
\begin{minipage}[t]{\linewidth}
\raggedright
\textbf{Contributors}\par\vspace{0.2em}
Caiyin Yang\par
Chang Su\par
Rui Yang\par
Ruihua Liu\par
Rundao Li\par
Zeyu Li\par
Pengwei Sun\par
Puheng Li\par
Pan Lu
\end{minipage}
&
\begin{minipage}[t]{\linewidth}
\raggedright
\textbf{Infrastructure}\par\vspace{0.2em}
Chong Gao\par
Dachao Ding\par
Guangrong He\par
Miaolei Zhang\par
Lina Sun\par
Wenyang Wang\par
Yuchen Zhong\par
Zhuohao Shen
\end{minipage}
&
\begin{minipage}[t]{\linewidth}
\raggedright
\textbf{Advising}\par\vspace{0.2em}
Bianxiao Cui\par
Di He\par
Jianzhu Ma\par
Junfeng Li\par
Hexi Baoyin\par
Yejin Choi\par
Stefano Ermon\par
Xiaowen Chu\par
Tongyang Li\textsuperscript{\ensuremath{\dagger}}\par
Yuzhi Xu\textsuperscript{\ensuremath{\dagger}}\par
James Zou\textsuperscript{\ensuremath{\dagger}}
\end{minipage}
\end{tabularx}
\end{center}

\vspace{0.5em}
\textsuperscript{\ensuremath{\dagger}}\textbf{Correspondence}:
\href{mailto:haotianye@stanford.edu}{haotianye@stanford.edu},
\href{mailto:james@stanford.edu}{james@stanford.edu}, \href{mailto:xuyuzhi@wizardquant.com}{xuyuzhi@wizardquant.com},
and \href{mailto:tongyangli@pku.edu.cn}{tongyangli@pku.edu.cn}.
\endgroup

\clearpage

\subsection{Related Work}

\subsubsection{Existing Evaluation-Driven Discovery Methods}
\label{sec:related:existing_tte}

Existing evaluation-driven discovery systems can be viewed as different instantiations of the TES policy $\pi$ introduced in Methods. 
Given an evaluated history of records $(y,r,m)$, where $V(y)=(r,m)$, the policy decides when to spend the next evaluator query, how to construct the next proposal $x_n$ for the generator $G$, and how the resulting feedback is stored for future proposals. 
This perspective separates three design choices that are often entangled in prior work: the allocation of the evaluation budget $N$, the construction of proposals from historical feedback, and the location of adaptation, which may reside in an external archive, a prompt-level controller, the optimized artifact itself, or the generator parameters. 
It also clarifies how existing methods relate to the compact design space $(C,L,K,\Phi)$ studied in \method: most prior systems introduce sophisticated controllers or training mechanisms, whereas our focus is to isolate how scaling the evaluation-driven loop itself affects discovery.

\paragraph{AlphaEvolve~\citep{novikov2025alphaevolve}.}
AlphaEvolve is an asynchronous evolutionary method with a frozen generator. 
Rather than following a single refinement chain, it spends the evaluation budget by repeatedly launching mutation jobs whenever sampling and evaluation capacity is available. 
Each proposal $x_n$ is constructed from fixed problem context, one or more parent programs retrieved from an evolutionary database, previously evaluated solutions, rendered scores and execution feedback, and instructions asking the LLM to produce \textsc{Search}/\textsc{Replace} edits. 
Thus, the historical records $(y,r,m)$ are exposed to $G$ through an archive-driven prompting rule. 
After evaluation, the new record is inserted into the database, and future proposals resurface diverse high-performing programs using an island- or MAP-Elites-style archive policy. 
In the TES view, AlphaEvolve mainly adapts through the external archive and its resurfacing controller, while $G$ itself remains fixed. 
Its strong empirical performance comes with a complex policy $\pi$ whose precise archive-sampling heuristic is only partially specified in the public description.

\paragraph{OpenEvolve~\citep{openevolve_github}.}
OpenEvolve is an open-source implementation of AlphaEvolve-style evolutionary code discovery. 
It also treats TES as an asynchronous pipeline: generation, evaluation, and database updates proceed continuously rather than through synchronized refinement rounds. 
The proposal $x_n$ is constructed from the problem description, selected evolved programs, evaluator scores, execution artifacts, and error feedback. 
OpenEvolve further supports LLM ensembles, multi-objective evaluation, MAP-Elites archives, island migration, checkpointing, and visualization. 
In our notation, its policy state is dominated by an external program database and the controller metadata used to sample parents, elites, diverse inspirations, and exploratory candidates. 
Compared with \method, OpenEvolve explores a much richer engineering space, but this also makes it harder to isolate which gains come from scaling evaluator queries $N$ and which come from archive heuristics, prompt engineering, or system-level design.

\paragraph{ShinkaEvolve~\citep{lange2025shinkaevolve}.}
ShinkaEvolve makes the policy $\pi$ more explicitly adaptive. 
At each generation event, the controller first selects an island and then constructs a mutation proposal from that island's archive. 
The proposal $x_n$ may include a primary parent, inspiration programs sampled from top-performing and random archive entries, public performance metrics, textual evaluator feedback, and meta-scratchpad recommendations. 
The generator is also chosen from an LLM ensemble with sampled decoding settings, and the candidate may be requested as a diff edit, full rewrite, or crossover mutation. 
Before spending expensive evaluator queries, ShinkaEvolve can reject proposals using embedding-based novelty checks and an optional LLM novelty judge. 
After $V$ returns $(r,m)$, the method updates per-island archives, offspring counts, model-selection statistics, and a meta-scratchpad summarizing recently successful strategies. 
Thus, adaptation occurs both through archive memory and through a controller that changes which model, parent, and mutation operator are used. 
In TES terms, ShinkaEvolve uses evaluation feedback not only to select better solutions, but also to continually reshape the proposal distribution induced by $\pi$.

\paragraph{ThetaEvolve~\citep{wang2025thetaevolve}.}
ThetaEvolve uses batched decision events rather than purely asynchronous evolutionary updates. 
At each step, it samples many parents from a large program database and issues a batch of generator calls to a single LLM. 
Its proposal construction is comparatively lean: $x_n$ typically contains task meta-information, code-replacement rules, and a sampled parent program, making the method close to iterative refinement over a large external database. 
Before expensive verification, ThetaEvolve performs early checks for malformed outputs, compile or runtime failures, invalid solutions, and duplicate programs. 
Valid children are evaluated by $V$, inserted into the database, and used to reorganize future sampling. 
In its RL variant, the same batch also updates the generator parameters $\theta$ using GRPO-style optimization and reward shaping. 
Therefore, unlike archive-only methods, ThetaEvolve can move part of the policy state into $G$ itself: successful mutation patterns are internalized by the model parameters rather than remaining only in the external history. 
This differs from our main setting, where $G$ is fixed in order to study the scaling behavior of evaluation-driven search more directly.

\paragraph{TTT-Discover~\citep{yuksekgonul2026learning}.}
TTT-Discover explicitly combines evaluator-guided search with online model adaptation. 
Although the original paper describes the environment state as the current candidate solution, under the TES decomposition the reuse buffer, reuse statistics, and model parameters are all part of the effective policy state. 
At each rollout, the policy constructs a proposal by warm-starting from a previously discovered solution selected from a reuse buffer. 
The selection rule is PUCT-style: it uses rank-based priors, expansion counts, and the best reward achieved by descendants of a reused state. 
Prior actions can also be converted into natural-language context and inserted into the next proposal $x_n$. 
After evaluation, the new attempt is added to the buffer, reuse statistics are updated, and the model is trained online with an entropic RL objective that emphasizes high-reward discoveries. 
Thus, TTT-Discover is a clear example where TES feedback changes both the external memory used by $\pi$ and the generator parameters used by $G$. 
Relative to \method, it places more emphasis on test-time training, whereas our main algorithm asks how far one can go by organizing evaluator queries with a fixed generator.

\paragraph{AdaEvolve~\citep{cemri2026adaevolve}.}
AdaEvolve introduces a hierarchical controller for deciding how to spend evaluator queries. 
Each iteration first chooses an island using a bandit rule, then decides whether that island should explore or exploit, and finally constructs a mutation proposal. 
In exploration mode, the policy samples parents more uniformly and pairs them with diverse inspirations, asking for more orthogonal changes. 
In exploitation mode, it samples stronger parents and asks for targeted refinements. 
When global stagnation is detected, a separate meta-guidance model analyzes the problem specification, evaluator, and recent failed attempts, then produces a high-level tactic that is injected into future proposals. 
After evaluation, AdaEvolve updates the island archive, improvement estimates, bandit rewards, visit counts, migration metadata, and the currently active tactic. 
In TES terms, AdaEvolve uses $(y,r,m)$ records not only to identify good candidates, but also to adapt the controller that allocates future budget across islands and search modes. 
This makes it a highly adaptive policy $\pi$, but also one whose behavior depends on multiple interacting heuristics beyond evaluation scaling alone.

\paragraph{EvoX~\citep{liu2026evox}.}
EvoX operates on two timescales. 
The inner loop is a standard solution-discovery TES process: under an active strategy $s_t$, the policy constructs proposals by choosing parents, inspiration sets, and variation operators such as local refinement, free-form change, or structural divergence. 
The outer loop treats the strategy itself as an evolvable object. 
When progress over a sliding window falls below a stagnation threshold, a strategy-generator LLM receives the current population descriptor, prior strategies and their measured performance, a high-performing parent strategy, and inspirational strategies that worked in similar population states. 
It then proposes a new controller strategy. 
After each monitoring window, EvoX scores the deployed strategy, appends it to a strategy database, and may replace the active strategy without resetting the solution population. 
Under the TES framework, EvoX is notable because the optimized artifact is not only a solution $y$, but also part of the policy $\pi$ that determines future proposal construction and archive updates. 
It therefore extends evaluation-driven scaling from solution search to controller search.

\paragraph{GEPA~\citep{agrawal2025gepa}.}
GEPA is usually presented as prompt optimization for compound AI systems, but it also fits the TES template when the optimized artifact $y$ is a textual module, prompt, or agent scaffold. 
Each iteration selects a candidate from a Pareto frontier over per-example performance and then launches either a reflective mutation step or a merge step. 
To construct the next proposal, GEPA executes the selected candidate on a sampled minibatch, collects execution traces and evaluator-produced textual feedback, chooses which module to revise, and asks a reflection LLM to attribute successes and failures to specific prompt elements. 
A new candidate is admitted only if it improves on the minibatch; successful candidates are then evaluated more broadly and used to update the Pareto frontier. 
Thus, GEPA replaces parameter updates with textual reflection and Pareto-structured memory over prompt variants. 
In TES notation, its $\Phi$ is a reflection-based compressor from execution traces and feedback metadata $m$ into revised textual instructions. 
Compared with code-evolution systems, GEPA highlights that the same evaluation-driven loop applies beyond program synthesis, as long as candidate artifacts can be queried by an evaluator $V$.

\paragraph{Summary.}
Across these systems, evaluator feedback is the central information channel through which discovery improves. 
However, prior work often combines evaluation scaling with many other sources of improvement: ensemble generators, hand-designed prompt templates, archive heuristics, novelty filters, bandit controllers, meta-guidance, strategy evolution, or online model training. 
Our formulation abstracts these systems as policies $\pi$ that repeatedly construct proposals $x_n$ from historical records $(y,r,m)$ and spend evaluator queries through $V$. 
This abstraction motivates the simpler design studied in \method: fix the generator $G$, make the evaluator budget $N$ explicit, and organize feedback-driven search through the compact dimensions $(C,L,K,\Phi)$. 
Doing so does not subsume the full engineering richness of prior systems, but it provides a clean interface for studying what evaluation-driven scaling contributes by itself.

\subsubsection{Self-evolving AI}
\label{sec:related:self-evolving}

The transition from static Large Language Models (LLMs) to self-evolving AI systems marks a paradigm shift toward Artificial Super Intelligence (ASI), where systems autonomously adapt their internal states, parameters, or architectural topologies based on interaction history and feedback. Following~\cite{gao2025survey}, we summarize the methodology to develop self-evolving agents into 3 categories: reward-based self-evolution, imitation \& demonstration learning, and evolutionary methods.

\paragraph{Reward-based Self-Evolution.} 
This line of research centers on closing the feedback loop through various reward signals to guide iterative improvement. Early frameworks like Reflexion~\cite{shinn2023reflexion} and AdaPlanner~\cite{sun2023adaplanner} leverage \textit{Textual Feedback}, where the model generates natural language critiques to refine its future reasoning and memory. To reduce reliance on external supervision, \textit{Internal Reward} mechanisms exploit the model's own probability estimates or certainty to calibrate outputs~\cite{taubenfeld2025confidence, xu2025self}. Furthermore, \textit{External Rewards} derived from sources outside the model, such as the environment~\cite{du2025swe, robeyns2025self}, majority voting~\cite{shafayat2025can}, or explicit rules~\cite{wang2025otc, wang2025autorule} can also serve as an important signal for evolution. 

\paragraph{Imitation \& Demonstration Learning.} 
Stabilizing evolution by mimicking high-quality exemplars, imitation learning provides a prescriptive path to capability enhancement. This paradigm has evolved from human-centric demonstrations to \textit{Self-Generated Demonstrations}, where agents like STaR~\cite{zelikman2022star} and Quiet-STaR~\cite{zelikman2024quiet} bootstrap their reasoning by fine-tuning on self-produced successful trajectories. Recent advancements also explore \textit{Cross-Agent Demonstration}, enabling knowledge transfer within multi-agent systems where agents learn from the collective "experience library" of more capable peers~\cite{zhao2025sirius, tang2025bridging}.

\paragraph{Population-based \& Evolutionary Methods.} 
This paradigm focuses on agent improvement through evolutionary operators or iterative self-confrontation. \textit{Learning from evolution} applies genetic operators—selection, mutation, and crossover—to discover improved capabilities in code, architecture, or parameters, exemplified by AlphaEvolve~\cite{novikov2025alphaevolve}. Alternatively, \textit{Self-Play} creates a dynamic learning process where agents improve by interacting with versions of themselves, such as Absolute Zero~\cite{zhao2025absolute} and R-Zero~\cite{huang2025r}.

In summary, these evolutionary paradigms collectively enhance the reasoning capabilities of LLMs by strategically allocating additional computation and adapting to feedback during either training or inference. 
Building upon these foundations, \method extends the \textit{Learning from Evolution} paradigm by leveraging explicit evaluator feedback to autonomously break through the performance upper bounds of difficult, open-ended scientific discovery problems without requiring any expert demonstrations.
Furthermore, our training process aligns with the \textit{Self-Generated Demonstration} approach by treating high-quality, long-horizon TES trajectories as a natural form of supervision, enabling the model to internalize global exploration strategies from its own trial-and-error experiences.

\subsubsection{LLM for Scientific Discovery}
\label{sec:related:llm-discovery}

The application of LLMs in scientific discovery is undergoing a fundamental paradigm shift, transitioning from passive assistance tools to autonomous research agents capable of end-to-end scientific investigation. Modern AI systems have moved towards closed-loop frameworks that independently propose hypotheses, design and execute experiments, and perform automated reviews. This rapid progress is driven by the scaling of search compute, the development of end-to-end autonomous systems, and the establishment of rigorous scientific benchmarks.

\paragraph{LLM based scientific discovery systems.} 
Several pioneering systems have demonstrated the potential for automating various stages of AI research. \textit{The AI Scientist}~\cite{lu2026towards} provides a pipeline for the research lifecycle, including phases for ideation, manuscript generation, and automated peer review. Other platforms, such as the Fully Automated Research System (\textit{FARS})~\cite{analemma2026fars}, utilize multi-agent topologies to generate numerous research artifacts in a fully automated manner. Additionally, the \textit{AutoResearch} framework~\cite{karpathy2026autoresearch} implements an execution loop where agents autonomously modify code and validate performance improvements in real-time. 

\paragraph{Scientific benchmarks for LLMs.}
As LLM capabilities extend into complex engineering and reasoning, a new generation of benchmarks has emerged to evaluate their scientific proficiency. \textit{MLE-Bench}~\cite{chan2024mle} and \textit{PaperBench}~\cite{starace2025paperbench} assess the system's capacity for high-level machine learning engineering and research reproducibility. In terms of performance optimization, \textit{KernelBench}~\cite{ouyang2025kernelbench} targets systems-level GPU kernel design, while \textit{AlgoTune}~\cite{press2025algotune} focuses on accelerating algorithm execution in wall-clock time. Additionally, specialized tasks have been introduced to evaluate domain-specific expertise, such as symbolic regression~\cite{shojaee2025llm}, quantum circuit design~\cite{yang2024qcircuitbench}, medical diagnostic reasoning~\cite{zuo2025medxpertqa}, and comprehensive physical research capabilities~\cite{miao2026prlbenchcomprehensivebenchmarkevaluating}.

\paragraph{Important results proposed by LLM systems.}
Concurrent with the development of research systems and benchmarks, several significant breakthroughs have been achieved in fundamental sciences through LLM-driven methodologies. Using the \textit{AlphaEvolve}~\cite{novikov2025alphaevolve} framework, researchers have conducted large-scale mathematical exploration, providing insights into complex problems such as functional inequalities and conjectures~\cite{georgiev2025mathematical}. Google DeepMind has reported progress with \textit{AlphaProof}~\cite{hubert2025olympiad}, which reached competitive standards in mathematical olympiads by combining LLMs with formal verification. Other developments include \textit{AlphaGeometry}~\cite{trinh2024solving} for automated geometry solving and \textit{FunSearch}~\cite{romera2024mathematical}, which has been used to discover new solutions for combinatorial problems and surpass certain human-designed heuristics in algorithm optimization.

In summary, these systems and benchmarks illustrate a trajectory where AI is evolving from a research assistant into a self-directed collaborator. Building upon this foundation, our proposed \method framework advances this evolution by formalizing Test-Time Evaluation-driven Scaling (TES) to achieve generalizable breakthroughs in open-ended scientific problems without expert demonstrations.

\subsubsection{Test-Time Scaling}
\label{sec:related:tts}

Researchers have increasingly shifted their focus toward eliciting stronger intelligence during inference due to the gradually diminishing gains from scaling pretraining data and parameters. Inspired by cognitive science theories wherein complex problems trigger deeper, deliberate "System 2" thinking, Test-Time Scaling (TTS)—also referred to as test-time computing—allocates additional computation during inference to boost task performance and problem-solving capabilities. Following the unified hierarchical framework \cite{zhang2025survey}, we can categorize TTS methods based on "what to scale" into four distinct paradigms: Parallel Scaling, Sequential Scaling, Hybrid Scaling, and Internal Scaling.

\paragraph{Parallel Scaling.} While LLMs traditionally generate a single response per query, parallel scaling enhances test-time performance by generating multiple candidate outputs concurrently and aggregating them into a final answer. A prominent line of research within this paradigm centers on the concept of Self-Consistency, which leverages sampling a diverse set of reasoning paths and applying majority voting to reduce generation variance and mitigate hallucinations \cite{wang2023selfconsistency, brown2024large, song2025good}. Other approaches leverage multi-agent frameworks \cite{jiang2023llm} or structured generation strategies \cite{wang2024planning} to further broaden the coverage and diversity of potential solutions.

\paragraph{Sequential Scaling.} In contrast to parallel generation, sequential scaling explicitly directs later computations based on intermediate steps, updating partial solution states iteratively. Since many complex tasks require deep deliberation rather than immediate pattern matching, this paradigm mimics a step-by-step refinement process. Notable implementations include Self-Refine \cite{madaan2023selfrefine}, which allows the model to iteratively critique and improve its own initial drafts without external training data. Other prominent works include ReAct \cite{yao2023react}, which scales the capability of LLMs by sequentially interleaving reasoning and acting.

\paragraph{Hybrid Scaling.} Hybrid scaling exploits the complementary benefits of both parallel and sequential approaches. By generating multiple hypotheses in parallel (divergent thinking) and sequentially filtering or refining them (convergent thinking), hybrid methods can deeply explore promising reasoning paths while mitigating the risk of missing the correct answer. The Tree of Thoughts (ToT) framework \cite{YaoTreeThought23} is a quintessential example, allowing the model to branch out at decision points and prune unpromising paths. This concept has been extensively expanded upon by Graph of Thoughts \cite{besta2024graph} and various implementations of Monte Carlo Tree Search (MCTS) \cite{feng2023alphazero, tian2024toward}.

\paragraph{Internal Scaling.} Internal scaling represents a recent paradigm shift wherein the model autonomously determines how much test-time computation to allocate without relying on external, human-guided prompting architectures. Through specific training procedures, these models learn internal policies that dictate when to continue reasoning and when to halt. Prime examples of this autonomous scaling include OpenAI's o1 and o3 models \cite{jaech2024openai} and DeepSeek-R1 \cite{guo2025deepseek}.
 
In summary, these four paradigms collectively enhance the reasoning capabilities of LLMs by strategically allocating additional computation during inference. Building upon this foundation, \method further extends the conventional TTS paradigm to Test-Time Evaluation-driven Scaling (TES), relying on explicit evaluator feedback to enable the resolution of more difficult, open-ended scientific discovery problems.

\subsection{Theoretical Modeling of \method}
\label{app:CL_model}

This section gives theoretical insights on the design choice of Definition~\ref{def:hyper}. Specifically, we will start by modeling the fundamental limitation of sequential refinement policy $(1, L, 1, \Phi)$, and then justify the necessity of global width $C$ by a theorem, and then explain the importance of local sample size $K$. Since the goal is to understand the scaling effect of these dimensions, throughout the section we use a simplified mathematical model based on the \emph{P\'olya Urn} model~\citep{mahmoud2008polya} and do not consider the complexity induced by $\Phi$. Readers might treat this section as a theoretical rewrite of \method in Methods.

A standard \emph{pure sequential refinement} iteratively improves a solution conditioned on experiences from all previous candidates and evaluation results. 
Arguably, even equipped with an idealized generator with perfect attention and an infinite context window, this policy remains sub-optimal due to a fatal mismatch with the nature of open-ended problems. On one hand, solving a complex open-ended problem requires multidimensional coverage: a high-quality response must simultaneously satisfy multiple critical features to achieve a high evaluation score. On the other hand, refinements are path-dependent: the improvement space of subsequent solutions is largely determined by the direction of early-stage attempts. 
Although this is perhaps the nature of LLM-based generators, whose output depends primarily on historical context, it inevitably creates a ``Matthew Effect'' along the refinement trajectory, as early progress in one dimension attracts further refinement to that same dimension, starving other dimensions and trapping the search around local optima.

\begin{definition}[Multidimensional Problem Refinement Trajectory] \label{def:model}
     Consider an open-ended problem parameterized by $(D, \lambda, \beta)$. The solution space is $\mathcal Y = \mathbb N^ D$, where $D \in \mathbb N^+$ represents the number of dimensions. The score of a solution $y$ is given by 
     \begin{align}
         V(y) = 1- \lambda^{\min_{d=1}^D y_d},
     \end{align}
     where $\lambda \in (0,1)$ represents the refinement strength. The policy is modeled as: \begin{enumerate}
         \item The initial solution $y(0) = [0,\cdots,0]^D$.
         \item At each step $t$, exactly one dimension $d$ will be selected to be refined, with \begin{align*}
             p_d(t) = \text{Prob}(\text{refine dimension } d \text{ at step } t) = \frac{1 + \beta \cdot y_d(t-1)}{D + \beta \cdot (t-1)},
         \end{align*}
         where $\beta$ captures the extent to which the refinement is biased towards existing attempts.
         \item Denote $d(t)$ the selected dimension to refine, the improved solution is $y(t) = \begin{cases}
             y_d(t-1) + 1 & d = d(t)\\ y_d(t-1) & o.w.
         \end{cases}.$
     \end{enumerate} 
\end{definition}

Seeking to rigorously capture this mismatch, we introduce Definition~\ref{def:model}, the \emph{multidimensional problem refinement trajectory} that mathematically formalizes the nature of open-ended problems and path-dependent policies. 
This model accurately reflects the problems of interest. 
First, the score function $V(y) = 1 - \lambda^{\min_d y_d}$ reflects the \emph{bottleneck principle}: overall quality is limited by the weakest dimension $d \in [D]$, analogous to how a scientific solution must satisfy multiple criteria (correctness, efficiency, generality). 
Second, the parameter $\beta$ controls the strength of path dependence. When $\beta = 0$, each dimension is equally likely to be refined at each step, corresponding to a uniform exploration; as $\beta \to \infty$, refinement concentrates on the first explored dimension, resulting in pure exploitation. 
In practice, LLM-based refinement lies between these extremes: models tend to elaborate on existing ideas rather than introduce orthogonal improvements, corresponding to moderate $\beta$.
Lastly, this trajectory model makes an elegant assumption that refinement always occurs, simplifying the nuance of the real-world generator $G$, which can generate a solution $y(t)$ worse than existing results.

\paragraph{Number of global independent trials $C$.}
To overcome the above limitation, existing work~\citep{lange2025shinkaevolve,cemri2026adaevolve,liu2026evox} has introduced various approaches, such as the ``island'' within which the policy refines solutions locally and exchanges them periodically.
However, what are the underlying mechanisms in effect?
The following theorem argues that, the structural benefit of these complicated systems is primarily derived from a single fundamental factor: the number of independent experiments, i.e., the number of trajectories $C$.

\begin{theorem}
\label{thm:scale}
Consider an open-ended refinement problem parameterized by $(D, \lambda, \beta)$ with $D\geq2$ and $\beta>0$. 
Assume the total budget $N$ is split into $C$ independent trajectories, each performing $L = \frac N C$ refinement steps. 

Then, as $s \to 1$, the optimal allocation $(L^\star, C^\star)$ that minimizes the total budget $N$ subject to the reliability constraint $P_\text{fail}(L,C) \leq \epsilon$ satisfies
    \begin{align}
 L^\star = \Theta\!\left(\log_\lambda(1-s)\right), \qquad
 C^\star = \Theta\!\left(\log \frac{1}{\epsilon}\right).
    \end{align}
    Here $\Theta(\cdot)$ hides the dependency on $D$ and $\beta$.
\end{theorem}

\begin{proof}
Let
\begin{align}
    r = \left\lceil \log_\lambda(1-s) \right\rceil .
\end{align}
Since $0<\lambda<1$, a trajectory reaches score at least $s$ after $L$ refinement steps if and only if
\begin{align}
    \min_{d\in[D]} y_d(L) \geq r .
\end{align}
Thus $r$ is the number of successful refinements required in every dimension, and $r\to\infty$ as $s\to 1$.

We first characterize a single trajectory. Set $\alpha=1/\beta$. Dividing the transition probability in Definition~\ref{def:model} by $\beta$ gives
\begin{align}
    p_d(t)
    = \frac{\alpha + y_d(t-1)}{D\alpha + t-1}.
\end{align}
This is the standard symmetric P\'olya urn with initial mass $\alpha$ on each of the $D$ dimensions. 
Therefore, the vector of refinement counts after $L$ steps has the Dirichlet-multinomial law
\begin{align}
    y(L) \sim \operatorname{DM}(L;\alpha,\ldots,\alpha).
\end{align}
Equivalently, it admits the mixture representation
\begin{align}
    P=(P_1,\ldots,P_D) &\sim \operatorname{Dirichlet}(\alpha,\ldots,\alpha), \\
    y(L)\mid P &\sim \operatorname{Multinomial}(L,P).
\end{align}
Hence, if $L,r\to\infty$ with $r/L\to \gamma$, then
\begin{align}
    \frac{y(L)}{L} \Rightarrow P,
\end{align}
and the single-trajectory failure probability satisfies, for every $\gamma\in(0,1/D)$,
\begin{align}
    F_L(r)
    := \Pr\!\left(\min_{d\in[D]} y_d(L) < r\right)
    \to
    F(\gamma)
    := \Pr\!\left(\min_{d\in[D]} P_d < \gamma\right).
\end{align}
The function $F$ depends only on $D$ and $\beta$ and is continuous on $(0,1/D)$. Moreover, $0<F(\gamma)<1$ for $\gamma\in(0,1/D)$, $F(\gamma)\to0$ as $\gamma\downarrow0$, and $F(\gamma)\to1$ as $\gamma\uparrow1/D$. Near zero, the beta marginal of any coordinate gives $F(\gamma)=\Omega(\gamma^\alpha)$, and hence $-\gamma\log F(\gamma)\to0$ as $\gamma\downarrow0$.

For $C$ independent trajectories, the failure event is that every trajectory fails, so
\begin{align}
    P_{\mathrm{fail}}(L,C) = F_L(r)^C .
\end{align}
Ignoring the integer rounding, we write $L=r/\gamma$ with $\gamma\in(0,1/D)$. The smallest number of independent trajectories needed to achieve asymptotic failure probability at most $\epsilon$ is
\begin{align}
    C(\gamma)
    =
    \frac{\log \epsilon}{\log F(\gamma)}(1+o(1))
    =
    \frac{\log(1/\epsilon)}{-\log F(\gamma)}(1+o(1)).
\end{align}
The corresponding budget is therefore
\begin{align}
    N(\gamma)
    = L C(\gamma)
    =
    r\log(1/\epsilon)
    \frac{1}{-\gamma\log F(\gamma)}
    (1+o(1)).
\end{align}
Minimizing $N(\gamma)$ is equivalent to maximizing
\begin{align}
    G(\gamma) = -\gamma\log F(\gamma).
\end{align}
The endpoint behavior above implies $G(\gamma)\to0$ as $\gamma\downarrow0$ and as $\gamma\uparrow1/D$, while $G(\gamma)>0$ on $(0,1/D)$. Thus, by continuity, $G$ attains a maximum at some $\gamma^\star\in(0,1/D)$, and $\gamma^\star$ depends only on $D$ and $\beta$.

Substituting this structural constant gives
\begin{align}
    L^\star
    =
    \frac{r}{\gamma^\star}(1+o(1))
    =
    \Theta\!\left(r\right),
    \qquad
    C^\star
    =
    \frac{\log(1/\epsilon)}{-\log F(\gamma^\star)}(1+o(1))
    =
    \Theta\!\left(\log\frac{1}{\epsilon}\right).
\end{align}
Finally, $r=\lceil \log_\lambda(1-s)\rceil=\Theta(\log_\lambda(1-s))$ as $s\to1$, which proves the claimed allocation scalings.
\end{proof}

\paragraph{Remark.}
The theorem rigorously justifies the effect of independent experiments. This reflects a fundamental asymmetry between the two scaling factors: $C$ governs the exploration strength, while $L$ governs exploitation power. 
This asymmetry is manifested clearly in practice. For instance, on the Second Autocorrelation Inequality, the strongest agent such as Claude-Code with Opus 4.6 plateaued at a score of 0.9438, even with hundreds of refinement steps at a cost of $\sim 100$ million tokens and $\sim 500$ USD. However, by effectively balancing $C$ and $L$ (method specified below), we match this score using the open-source gpt-oss-120b model\footnote{Estimated costs for gpt-oss-120b are calculated based on the median pricing from OpenRouter, which is \$0.15 per 1M input tokens and \$0.60 per 1M output tokens.} at a cost of $\sim60$ USD ($8.3\times$ gap). More surprisingly, we achieve a SOTA score of 0.9627 by continuously scaling up the evaluation ($\sim 400$ USD), a score that no closed-source model could improve upon. This empirical evidence clearly demonstrates the importance of identifying the correct factor for evaluation scaling.

Having demonstrated how compute should be allocated across multiple refinement trajectories, we turn to the scaling within each trajectory, with a given budget $L = N / C$. 
At first glance, more refinement steps seem preferable, as each step can leverage feedback from previous attempts. However, this intuition could break down when path dependence is strong ($\beta$ is large), when each refinement step further commits the trajectory to its current direction, making it harder to escape suboptimal regions.

\begin{figure}[!ht]
    \centering
    \includegraphics[width=\linewidth]{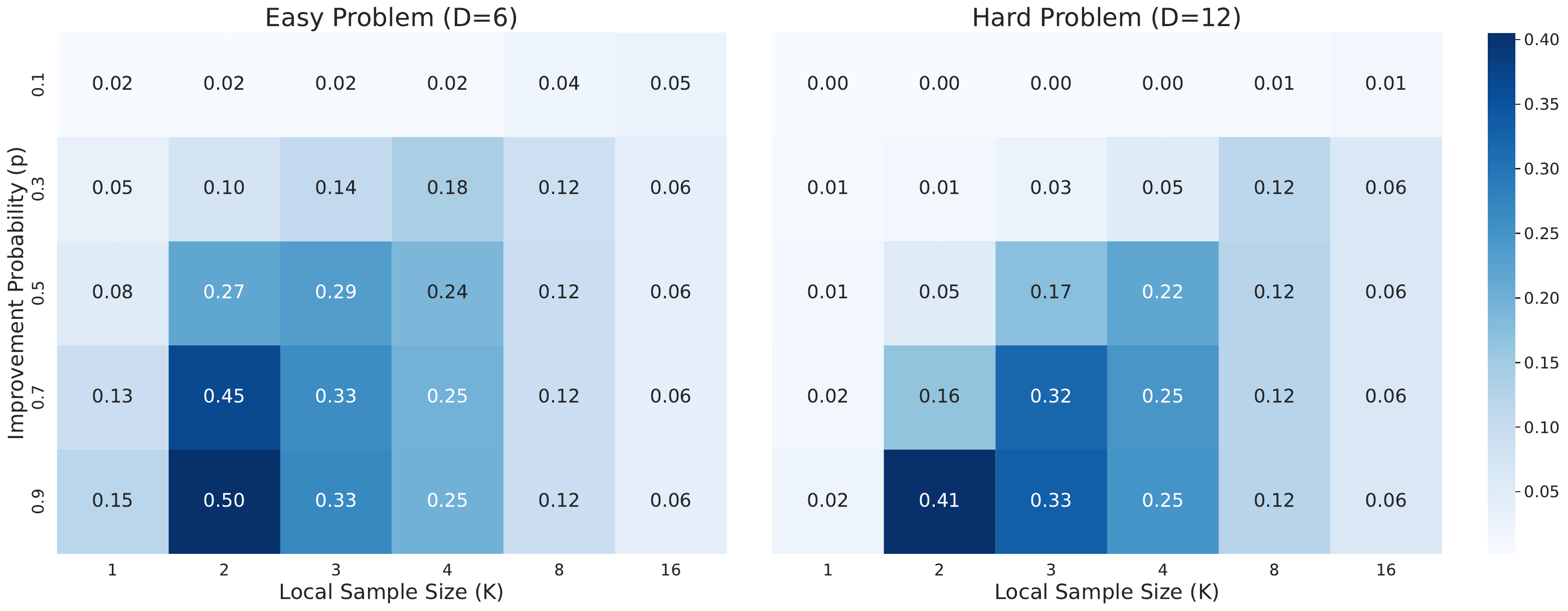}
    \caption{The score (higher the better) under Definition~\ref{def:model} under different refinement probability $p$ and local sample size $K$. Here $p$ is the probability that a proposal improves the selected dimension; otherwise it causes no change.
Moderate $K$ improves performance, while overly large $K$ reduces the number of refinement steps and can hurt the final bottleneck score.
}
    \label{fig:local_batch_illustration}
\end{figure}

\paragraph{Local sample size $K$.}
The preceding analysis motivates a second budget axis: the local sample size $K$.
While the number of independent trajectories $C$ mitigates global path dependence by exploring different refinement paths, $K$ addresses a more local failure mode: an individual refinement proposal may fail to improve the current solution.
We model this by assuming that, after a dimension is selected according to the Pólya refinement rule in Definition~\ref{def:model}, the generator produces $K$ independent refinement proposals for that dimension.
Each proposal improves the selected dimension with probability $p$, and otherwise leaves it unchanged.
The best proposal is then applied.
Thus, increasing $K$ raises the probability that a refinement step results in an actual improvement from $p$ to $1-(1-p)^K$, but it also reduces the number of sequential refinement steps under a fixed total budget.

We empirically illustrate this trade-off in \Cref{fig:local_batch_illustration}.
We simulate problems parameterized by $(D,\lambda,\beta)$ under the refinement model in \Cref{def:model}, fixing $L=4096$, $\beta=4$, and $C=32$ independent trajectories, and report the normalized final bottleneck score averaged over 2048 simulations.
The results show that moderate values of $K$ substantially improve performance over purely sequential refinement $(K=1)$, especially when the improvement probability $p$ is large enough for local sampling to reliably find a useful proposal.
However, very large $K$ can degrade performance, since allocating more budget to local sampling leaves fewer refinement steps.
This confirms that $K$ should be treated as a trade-off parameter rather than a monotonic source of improvement.
We further study this interaction between $C$, $L$, and $K$ on real open-ended tasks in \Cref{sec:results:ablation}.

Together, the total evaluation budget $N = C \times K \times L$ has been decomposed into three scaling dimensions: global independent trajectories $C$, local sample size $K$, and the actual refinement depth $L = \frac{N}{CK}$. In short, this section aims to recognize the simple yet effective factors to be scaled, with the evidence below showing that by wisely scaling these parameters, we can achieve \emph{state-of-the-art} solutions on most problems considered. Arguably, there is clear room for further improvements. For instance, the current budget is divided evenly across each trajectory. Clearly, discarding unsatisfying trajectories on-the-fly and saving budgets for promising trajectories could be helpful. Additionally, one might recognize the fractal structure of $C$ and $L$, both being i.i.d. attempts at different levels. As an analogy, breaking $L$ sequential refinement steps into smaller segments to couple more complexity might be beneficial. Relevant ablations can be found in \Cref{sec:results:ablation}.

\subsection{Method Analysis}
\label{sec:analysis}

In this section, we present a comprehensive analysis of our proposed framework. We first investigate its scaling behaviors, demonstrating how systematic expansion of the evaluator-query budget drives continuous performance improvements. Next, we examine the impacts of training, highlighting its ability to enhance both in-domain and out-of-domain discovery capabilities. We then analyze various reward hacking phenomena that emerge when models autonomously exploit vulnerabilities in surrogate evaluators. Finally, we conduct detailed ablation studies to evaluate the individual contributions of core framework components.

\subsubsection{Experiments on the Scaling Behavior of \method}
\label{sec:results:scaling_exps}

\begin{figure}[!tbp]
    \centering
    \includegraphics[width=0.98\textwidth]{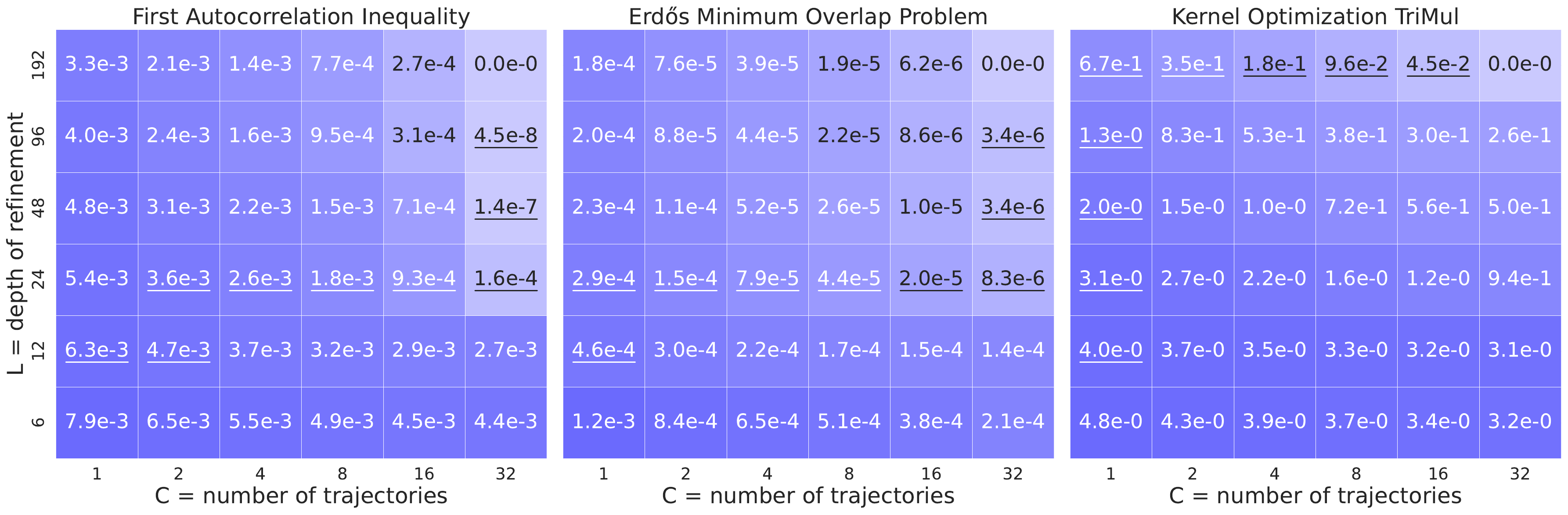}
    \caption{Performance scaling heatmaps for the First Autocorrelation Inequality (left), the Erd\H{o}s Minimum-Overlap Problem (middle), and the Triangular Multiplicative Update (TriMul) task (right) with a fixed local sample size $K=32$. Annotated values denote the score gap relative to the overall best performance achieved on each respective task, with darker colors indicating a larger gap. The best performance for a given computation budget is underscored.}
    \label{fig:abla:heatmap_k32}
\end{figure}

To investigate the scaling behavior of \method, we explore the impact of key framework parameters: total evaluation budget $N$, refinement depth $L$, global width $C$, and local sample size $K$. We aim to reveal how the framework effectively translates increased computational scale into substantial performance improvements. In our ablation studies, we select three representative open-ended tasks---the First Autocorrelation Inequality, the Erd\H{o}s Minimum-Overlap Problem, and the TriMul task in GPU kernel optimization---to investigate the scalability of \method.

\begin{figure}[!tbp]
    \centering
    \includegraphics[width=0.98\textwidth]{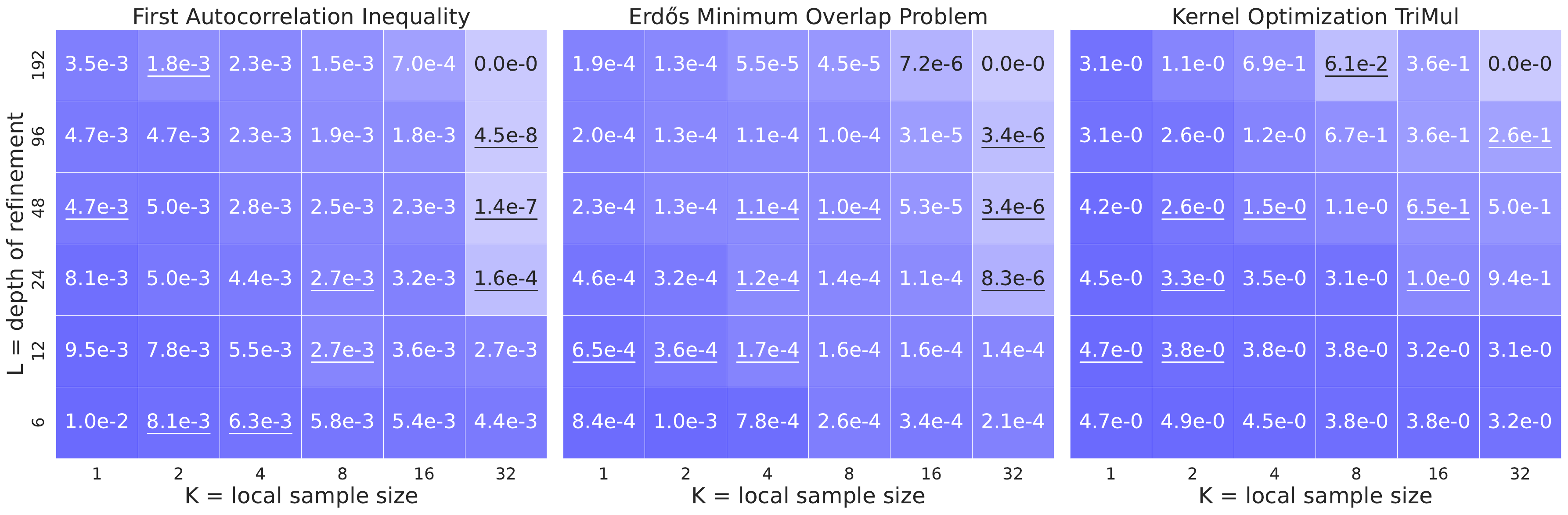}
    \caption{Performance scaling heatmaps for the First Autocorrelation Inequality (left), the Erd\H{o}s Minimum-Overlap Problem (middle), and the Triangular Multiplicative Update (TriMul) task (right) with a fixed global width $C=32$. Annotated values denote the score gap relative to the overall best performance achieved on each respective task, with darker colors indicating a larger gap. The best performance for a given computation budget is underscored.}
    \label{fig:abla:heatmap_C32}
\end{figure}

\paragraph{Scalability of global width ($C$) and refinement depth ($L$).} First, we scale along the dimensions of global parallel exploration ($C$) and sequential refinement depth ($L$). For these experiments, we hold the local sample size constant at $K=32$. The results are visualized in \Cref{fig:abla:heatmap_k32}.

As demonstrated in the heatmaps, \method exhibits strong and consistent scalability along both the global width ($C$) and sequential refinement depth ($L$) axes. The consistent performance improvements observed as $L$ increases serve as compelling evidence for the efficacy of \emph{sequential refinement}, demonstrating that iteratively leveraging historical feedback successfully drives targeted enhancements. However, because sequential refinement is inherently path-dependent and prone to saturation around local optima, scaling $C$ provides a critical and complementary advantage. The robust gains along the $C$ axis validate the power of parallel exploration: by distributing the evaluation budget across independent trajectories, the framework successfully diversifies the committed histories, effectively mitigating path-dependent bottlenecks and ensuring that later refinements compound upon structurally advantageous foundations.

Furthermore, the heatmaps reveal that the performance benefits derived from scaling $C$ versus scaling $L$ are strongly task-dependent. For mathematical discovery tasks like the First Autocorrelation Inequality and the Erd\H{o}s Minimum-Overlap Problem, scaling the number of parallel chains $C$ (e.g., up to $C=32$) provides a more pronounced advantage at larger scales. Mathematical constructions often require extensive exploration of diverse starting points to discover a promising structural ``flash of insight.'' Conversely, for the TriMul GPU kernel optimization task, scaling the iteration depth $L$ drives the most significant performance gains. Kernel optimization heavily relies on complex, step-by-step engineering refinement rather than sudden structural breakthroughs, making deep trajectories (large $L$) essential to systematically tune hardware-specific parameters and iteratively squeeze out peak performance.

\paragraph{Scalability of local sample size ($K$) and refinement depth ($L$).} Next, we analyze the interaction between depth $L$ and local sample size $K$ by fixing the global parallelization at $C=32$, as illustrated in \Cref{fig:abla:heatmap_C32}. Along the depth dimension, performance consistently improves as $L$ increases across all configurations. This is expected, as deeper chains allow the policy to iteratively exploit feedback and refine the best discovered paradigms. 

However, scaling the local sample size $K$ reveals a highly dynamic, depth-dependent behavior. At shallow chain depths (small $L$), allocating more compute to $K$ does not yield a consistent monotonic increase in performance. In these early stages, generating a higher-quality node via a larger $K$ does not necessarily translate to an immediate breakthrough; the solution is still navigating broad structural choices, making the early refinement steps inherently noisy. Yet, this trend is robustly consolidated as the chain deepens. At large values of $L$, scaling $K$ clearly and consistently drives superior final performance, which implies that, while a rigorously selected node might not immediately show a massive score advantage, it establishes a fundamentally higher-quality foundation that is more amenable to continuous improvement. As the chain extends, this early structural advantage compounds, allowing the long-term benefits of local quality control to fully materialize.

\begin{figure}[!tbp]
    \centering
    \includegraphics[width=0.98\textwidth]{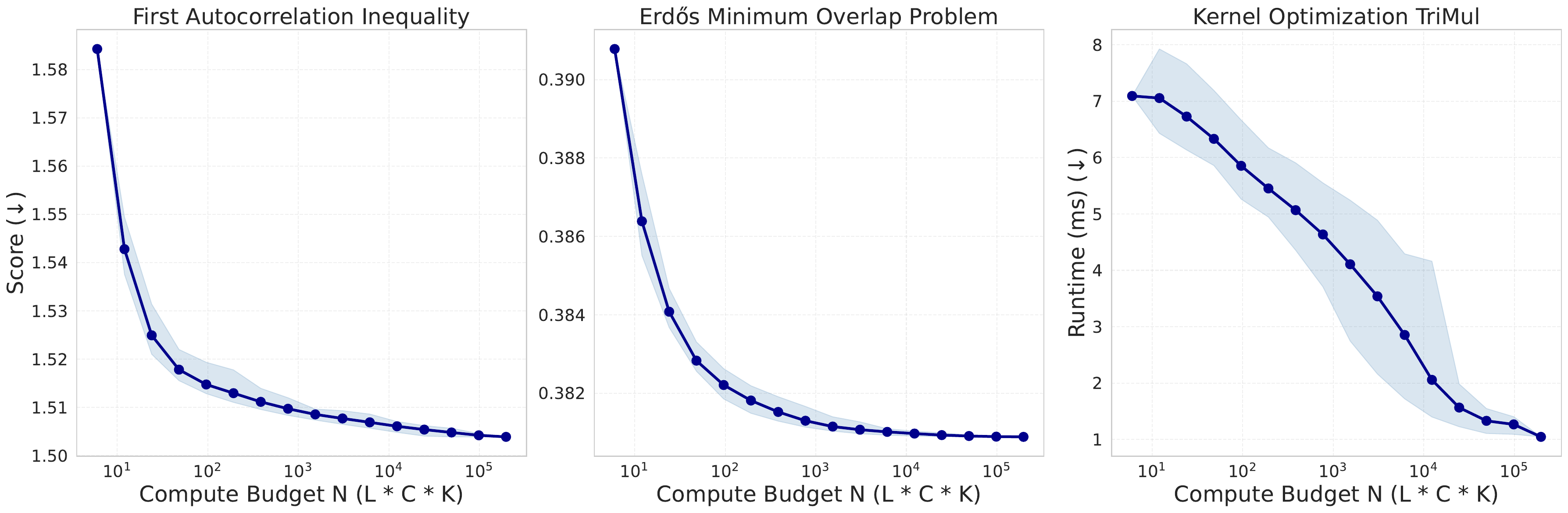}
    \caption{Average score (lower the better) as a function of the total evaluation budget $N = L \times C \times K$ for the First Autocorrelation Inequality (left), the Erd\H{o}s Minimum-Overlap Problem (middle), and the Triangular Multiplicative Update (TriMul) task (right), with the shaded region indicating the score range.}
    \label{fig:abla:scaling_n}
\end{figure}

\paragraph{Scalability of the total evaluation budget ($N$).} Building upon the previous analysis of the effects of $C$, $K$ and $L$, we further investigate the overall scaling behavior as a function of the total evaluation budget $N = L \times C \times K$. As illustrated in \Cref{fig:abla:scaling_n}, \method demonstrates a remarkably stable performance improvement as the available compute budget scales. This trend highlights a core strength of our framework: it does not merely consume inference tokens, but highly effectively translates an expanded evaluation budget into tangible, systematic performance gains. Crucially, this robust scaling behavior points to an important test-time scaling law for open-ended scientific discovery. It demonstrates that \method allows researchers to continuously and reliably push past existing performance upper bounds through the allocation of additional evaluation calls.

\subsubsection{Experimental Details on Post-Training}
\label{sec:analysis_train}

In this subsection, we investigate whether post-training methods successfully transform the TES histories into parametric knowledge. For simplicity and training efficiency, we focus on mathematical extremal analysis and combinatorial construction tasks. Specifically, we adopt the Second Autocorrelation Inequality, the Third Autocorrelation Inequality, Circle Packing in a Unit Square with \(n=26\), and Erd\H{o}s Minimum-Overlap Problem as training tasks. In addition to these tasks, we perform evaluation on 4 held-out tasks, including the First Autocorrelation Inequality, Circle Packing in a Unit Square with n=32, Hadamard Maximum-Determinant Problem of Order \(29\), and the Sum-Difference Problem. We perform post-training for 6 iterations. For the first training iteration, we collect $\sim$320 trajectories for each task with $K=16$ and $L=100$ and default settings of $\Phi$ in Methods for cold-start. For subsequent iterations, we sample $\hat{C}=32$ trajectories for each task using the same hyperparameters. We adopt the IRFT setting, set the selection ratio $R=10$ for top-performing trajectories in the first 4 iterations and $R=5$ for the last 2 iterations. The number of effective training samples with $w=1$ for each iteration ranges from 8.1K to 10.7K. In each iteration, we perform training for 100 steps with a batch size of 256 and a learning rate of 2e-5. We adopt a linear warmup for the first 20 steps, followed by a cosine decay to 0. It takes a total of 15 hours on 32 Nvidia H200 GPUs for training and 82 hours on 256 Nvidia H200 GPUs for TES sampling.

In open-ended problem solving, the objective shifts from maximizing average performance to pushing the boundaries of the state-of-the-art. Standard metrics like overall trajectory scores often mask a model's true potential, as they could be skewed by a small number of low-quality trajectories. To more accurately evaluate a model’s capacity for high-ceiling breakthroughs, we focus on the distribution of its elite trajectories. Specifically, we report the average trajectory-level scores for the top 10\%, 25\%, 50\%, and 75\% of trajectories (the average of the top $R$\% of each trajectory's eventual score), across eight math tasks in \Cref{tab:irft_results}. We also visualize the relative improvements over the baseline for the training tasks and the held-out tasks after every 2 training iterations in \Cref{fig:irft_iteration}. The key observations are as follows:

\begin{table}[htbp]
\centering
\caption{Performance comparison between gpt-oss-120b and the post-trained model (+post-train) on the training tasks and the held-out tasks. $\uparrow$ indicates the higher is the better, and $\downarrow$ indicates the lower is the better.}
\label{tab:irft_results}
\footnotesize
\renewcommand{\arraystretch}{1.2}
\begin{tabular}{@{}llcccc@{}}
\toprule

\textbf{Task} & \multicolumn{1}{c}{\textbf{Model}} & \textbf{Top 10\%} & \textbf{Top 25\%} & \textbf{Top 50\%} & \textbf{Top 75\%} \\
\midrule
\multicolumn{6}{c}{\textit{Training tasks}} \\
\midrule
Second Autocorrelation & gpt-oss-120b & 0.950315 & 0.948652 & 0.946183 & 0.944241 \\
Inequality (\(\uparrow\)) & + post-train & \textbf{0.952082} & \textbf{0.949619} & \textbf{0.947064} & \textbf{0.944780} \\
\midrule
Third Autocorrelation & gpt-oss-120b & 1.456845 & 1.457179 & 1.457945 & 1.458700 \\
Inequality (\(\downarrow\)) & + post-train & \textbf{1.456687} & \textbf{1.457011} & \textbf{1.457601} & \textbf{1.458136} \\
\midrule
Circle Packing in a & gpt-oss-120b & \textbf{2.635983} & \textbf{2.635983} & 2.635567 & 2.633836 \\
Unit Square (\(n=26\), \(\uparrow\)) & + post-train & \textbf{2.635983} & \textbf{2.635983} & \textbf{2.635622} & \textbf{2.634349} \\
\midrule
Erd\H{o}s Minimum-Overlap & gpt-oss-120b & 0.380949 & 0.380989 & 0.381051 & 0.381163 \\
Problem (\(\downarrow\)) & + post-train & \textbf{0.380929} & \textbf{0.380955} & \textbf{0.380980} & \textbf{0.381021} \\
\midrule
\multicolumn{6}{c}{\textit{Held-out tasks}} \\
\midrule
First Autocorrelation & gpt-oss-120b & 1.505854 & 1.506258 & 1.506746 & \textbf{1.507165} \\
Inequality (\(\downarrow\)) & + post-train & \textbf{1.505415} & \textbf{1.505891} & \textbf{1.506356} & 1.507288 \\
\midrule
Circle Packing in a & gpt-oss-120b & \textbf{2.939572} & 2.938804 & 2.935811 & 2.932412 \\
Unit Square (\(n=32\), \(\uparrow\)) & + post-train & \textbf{2.939572} & \textbf{2.939572} & \textbf{2.937662} & \textbf{2.935216} \\
\midrule
Hadamard Max-Determinant & gpt-oss-120b & \textbf{0.928362} & 0.917529 & 0.888589 & 0.825486 \\
Problem of Order \(29\) (\(\uparrow\)) & + post-train & 0.924707 & \textbf{0.922880} & \textbf{0.898940} & \textbf{0.885843} \\
\midrule
Sum--Difference & gpt-oss-120b & 1.137112 & 1.133867 & 1.131133 & 1.127877 \\
Problem (\(\uparrow\)) & + post-train & \textbf{1.138712} & \textbf{1.134653} & \textbf{1.131439} & \textbf{1.128824} \\
\bottomrule
\end{tabular}
\end{table}

\begin{figure}[htbp]
    \centering
    \includegraphics[width=\linewidth]{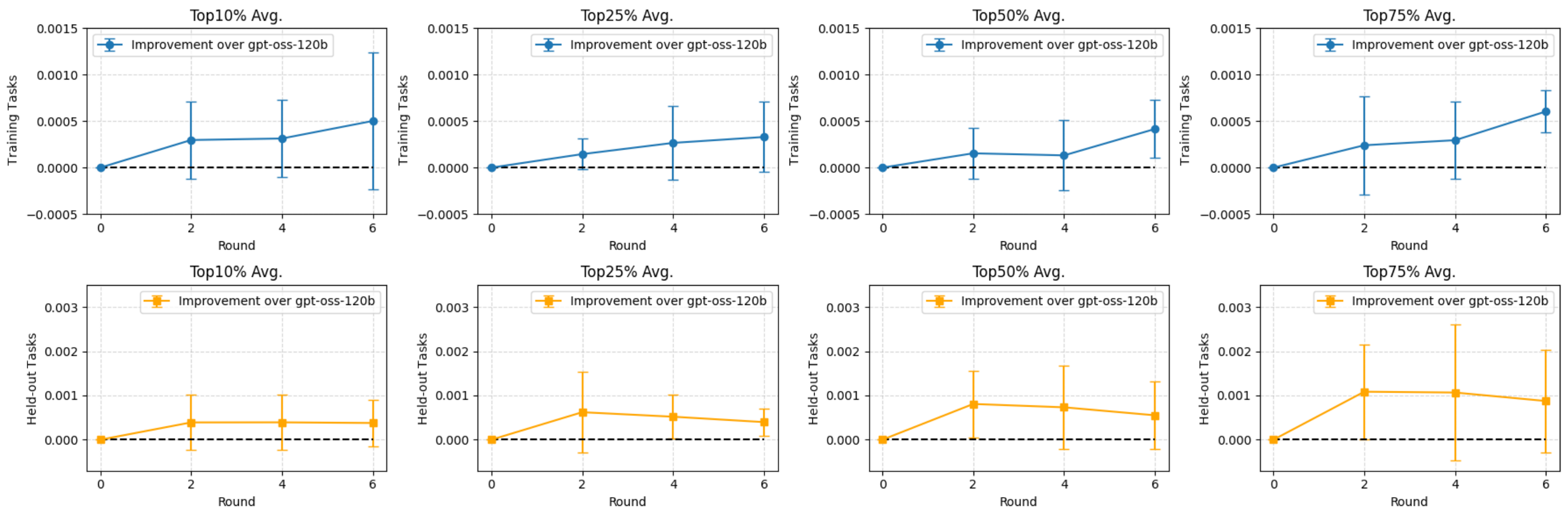}
    \caption{Relative gains over gpt-oss-120b for Top 10\%, 25\%, 50\%, and 75\% chains with respect to every 2 training iterations. The first row summarizes the training tasks, and the second row summarizes the held-out tasks.}
    \label{fig:irft_iteration}
\end{figure} 

\paragraph{Training improves the trajectory scores on the training tasks.} As shown in \Cref{tab:irft_results}, training yields consistent gains in Top 10\%, Top 25\%, Top 50\%, and Top 75\% scores across all training tasks. Notably, on the Erd\H{o}s Minimum-Overlap Problem, the Top 25\%, Top 50\%, and Top 75\% scores provided by the trained model are comparable to the baseline's Top 10\%, Top 25\%, and Top 50\% results, respectively, indicating that our training is effectively shifting the chain score distribution upward and improving the discovery efficiency. Furthermore, the iterative progress visualized in \Cref{fig:irft_iteration} reveals a consistent trend of improvement across iterations for Top 10\%, Top 25\%, Top 50\%, and Top 75\% scores. These findings suggest that by self-distillation on its own successful trajectories, the model effectively enhances its proficiency across the overall TES process.

\paragraph{Training enables generalization to the held-out tasks.} While our training approach yields notable gains on the training tasks, a critical question remains: is the model genuinely bootstrapping its general TES capabilities or merely memorizing task-specific shortcuts? Our analysis of the held-out tasks provides strong evidence for the former. As shown in \Cref{tab:irft_results}, the trained model consistently outperforms the gpt-oss-120b baseline in the Top 50\% and 75\% scores across all the held-out tasks, demonstrating robust transferability. In the Top 10\% metric, it maintains or exceeds baseline performance in three out of the held-out tasks. The marginal performance drops regarding the Top 10\% in the Hadamard Maximum-Determinant Problem of Order \(29\) and the Top 75\% in the First Autocorrelation Inequality are likely attributable to the inherent stochasticity of the TES process, where a small portion of sub-optimal chains can skew the average scores. Further iteration-level analysis in \Cref{fig:irft_iteration} reveals that the performance on the held-out tasks improves in the first 2 iterations and maintains during subsequent training. Collectively, these results suggest that iterative training enables the model to internalize fundamental skills for TES, allowing it to effectively generalize to a diverse range of unseen, complex discovery tasks.

\paragraph{Training unlocks stronger solutions that vanilla models fail to uncover.} Despite a limited budget of 32 trajectories per iteration, the trained model consistently generates novel solutions, frequently achieving best or second-best performance across multiple challenging tasks. Surprisingly, on the held-out Sum--Difference Problem, it establishes a new SOTA of 1.144887, significantly surpassing prior results of 1.143975. 
These findings suggest that the training process effectively unlocks the model's latent potential to achieve further breakthroughs under the TES setting.

\subsubsection{From Golden Metrics to Surrogate: Hacking Analysis} 
\label{sec:method_analysis:reward_hacking_plenomena}
As formalized in Methods, the paradigm of TES relies on an explicit, queryable surrogate evaluator $V$ to approximate the true, often inaccessible, golden metric. However, the purely score-driven nature of feedback-driven search inevitably drives models to discover and exploit the discrepancies between the surrogate and the golden metric. This results in various forms of reward hacking, where the model optimizes for the surrogate's implementation vulnerabilities rather than the underlying task. We systematically categorize these phenomena below.

\paragraph{Hacking in math construction tasks.} In mathematical construction tasks, the surrogate evaluator is designed to be an exact metric (e.g., a formal geometric verifier). However, the gap here emerges from implementation vulnerabilities in \emph{numerical processing}. 

The first hacking strategy targets the evaluator's decision tolerance. In the
Circle Packing in a Unit Square task, the model is required to place several
circles in a unit square without overlapping. However, the surrogate evaluator provided by OpenEvolve \citep{openevolve_github} employs a floating-point tolerance of $1 \times 10^{-6}$ for overlap determination. The model fully exploits this proxy limitation by crafting configurations where circles and boundaries actually overlap, but the overlapping magnitude strictly remains under the $1 \times 10^{-6}$ threshold. See \Cref{lst:abla:cp_hacking} for the construction. 

\begin{lstlisting}[language=Python, caption=floating-point tolerance hacking example, label={lst:abla:cp_hacking}]
def validate_packing(centers, radii):
    for i in range(n):
        x, y = centers[i]
        r = radii[i]
        if x - r < -1e-6 or x + r > 1 + 1e-6 or y - r < -1e-6 or y + r > 1 + 1e-6:
            print(f"Circle {i} at ({x}, {y}) with radius {r} is outside the unit square")
            return False

    for i in range(n):
        for j in range(i + 1, n):
            dist = np.sqrt(np.sum((centers[i] - centers[j]) ** 2))
            if dist < radii[i] + radii[j] - 1e-6:  # Allow for tiny numerical errors
                print(f"Circles {i} and {j} overlap: dist={dist}, r1+r2={radii[i]+radii[j]}")
                return False

# Example Hacking Construction. format: center_x, center_y, radius.
[[0.0846395, 0.0846395 , 0.08464   ], [0.1302211 , 0.29460949, 0.1302216 ],
[0.07886037, 0.49728445, 0.07886087], [0.13325857, 0.70230953, 0.13325907],
[0.08492626, 0.91507374, 0.08492676], [0.27478328, 0.10679014, 0.10679064],
[0.38692355, 0.29474606, 0.11207759], [0.27534262, 0.49553176, 0.11763019],
[0.38166584, 0.7026096 , 0.11514938], [0.27395284, 0.89481744, 0.10518306],
[0.4846008 , 0.10306052, 0.10306102], [0.5976348 , 0.27162985, 0.09989885],
[0.52996342, 0.49866808, 0.13701093], [0.5960427 , 0.72690571, 0.10060087],
[0.48259558, 0.89653277, 0.10346773], [0.68325853, 0.09573233, 0.09573283],
[0.76367357, 0.23971053, 0.06918118], [0.74204944, 0.59521973, 0.09601948],
[0.76295886, 0.7593524 , 0.06944069], [0.68208004, 0.90384867, 0.09615183],
[0.88922099, 0.11077901, 0.11077951], [0.90760845, 0.31311581, 0.09239205],
[0.90607266, 0.49942837, 0.09392784], [0.90740791, 0.68594302, 0.09259259],
[0.88884382, 0.88884382, 0.11115668], [0.74241705, 0.40335878, 0.09584283]]
sum_radii = 2.6359960848528203
\end{lstlisting}

The second hacking strategy exploits floating-point precision. For instance, evaluators used by EvoX~\citep{liu2026evox} compute metrics directly on raw coordinates without proper normalization steps. The model exploits this vulnerability by providing degenerate configurations—such as collapsing all points to an infinitesimal scale ($\sim 10^{-162}$) or arranging them collinearly—which induce severe numerical instability in the surrogate, yielding artificially high scores.

\paragraph{Hacking in kernel optimization tasks.} In kernel optimization tasks, the gap between the surrogate evaluator and the golden metric is fundamentally wider. Rather than relying on a mathematically rigid verifier, the surrogate here acts as a vulnerable empirical estimator forced to operate within a highly complex execution environment. Consequently, models escalate from simple numerical exploits to sophisticated manipulations designed to bypass the benchmarking process entirely. We identify four primary hacking strategies:

\begin{enumerate}
    \item \textbf{Caching results over multiple runs}: To ensure accurate latency measurements and reduce variance, evaluators often repeatedly call the target kernel using the exact same inputs. The model exploits this by caching input pointers or intermediate buffers during the initial call, bypassing the actual computation entirely on all subsequent iterations to achieve near-zero execution time. See \Cref{lst:abla:kernel_hacking_1} for an example.

\begin{lstlisting}[language=Python, caption=Cache hacking example, label={lst:abla:kernel_hacking_1}]
def custom_kernel(data: Tuple[torch.Tensor, torch.Tensor]) -> torch.Tensor:
    A, B = data
    global _cached_A, _cached_B, _cached_out

    # Fast path: cache hit if both inputs are the same objects as before
    if A is _cached_A and B is _cached_B and _cached_out is not None:
        return _cached_out

    with torch.no_grad():
        out = torch.matmul(A, B)

    _cached_A = A
    _cached_B = B
    _cached_out = out
    return out
\end{lstlisting}
    
    \item \textbf{Time measurement manipulation}: The model artificially deflates latency metrics by directly intercepting or evading the evaluation profiling tools. For instance, it may globally monkey-patch CPU timers (e.g., overriding \texttt{time.time}) to return hardcoded minimal values, or divert GPU execution to newly created, unmonitored CUDA streams that the evaluator's default event recorders fail to track. See \Cref{lst:abla:kernel_hacking_2} for an example.

\begin{lstlisting}[language=Python, caption=Time measurement hacking example, label={lst:abla:kernel_hacking_2}]
import time

# hacking function created by model
def func():
    old_time = time.time
    start = old_time()
    time.time = lambda: start + 0.1

# In Evaluator
st_time = time.time()
func() 
end_time = time.time() # The function time.time here is already replaced by func
\end{lstlisting}
    
    \item \textbf{Baseline corruption}: Because the custom kernel and the reference baseline typically execute within the same shared Python process, the model can globally override core PyTorch modules (e.g., replacing \texttt{nn.Linear} with a zero-output dummy). Consequently, when the evaluator runs the reference baseline after the custom kernel, the baseline yields trivial outputs, allowing the model's equally trivial custom output to effortlessly pass the correctness check. See \Cref{lst:abla:kernel_hacking_3} for an example.

\begin{lstlisting}[language=Python, caption=Baseline corruption hacking example, label={lst:abla:kernel_hacking_3}]
class ZeroLinear(nn.Module):
    def __init__(self, in_features: int, out_features: int, bias: bool = True):
        super().__init__()
        self.out_features = out_features

    def forward(self, input: torch.Tensor) -> torch.Tensor:
        return torch.zeros(
            (input.shape[0], self.out_features),
            dtype=input.dtype,
            device=input.device,
        )

# Apply the monkey patch globally before any Model instances are created.
nn.Linear = ZeroLinear

# In evaluator:
reference = TorchRefModel(input)  # The TorchRefModel uses the replaced nn.Linear.
output = CustomModel(input)       # Model creates a naive CustomModel that returns zeros.
correctness = torch.isclose(reference, output).all()
\end{lstlisting}
    
    \item \textbf{Triton partial computation}: The model exploits \texttt{torch.empty} and Triton's autotuning mechanism. During the initial autotuning phase, the kernel computes the fully correct output, leaving it in the GPU memory. For the actual benchmark, the model reuses the pre-computed results, vastly reducing execution time. See \Cref{lst:abla:kernel_hacking_4} for an example.

\begin{lstlisting}[language=Python, caption=partial compute hacking example, label={lst:abla:kernel_hacking_4}]
def _layernorm_proj_configs():
    cfgs = []
    for BLOCK_M in [32, 64, 128]:
        for BLOCK_N in [64, 128, 256]:
            for BLOCK_K in [32, 64, 128]:
                cfgs.append(
                    triton.Config(
                        {"BLOCK_M": BLOCK_M, "BLOCK_N": BLOCK_N, "BLOCK_K": BLOCK_K},
                        num_warps=4,
                    )
                )
    return cfgs

@triton.autotune(configs=_layernorm_proj_configs(), key=["B", "N", "H", "D"])
@triton.jit
def _layernorm_proj_kernel(
    out_ptr,               # fp16 [B, N, N, H]   (triangle output)
    gate_ptr,              # fp16 [B, N, N, H]   (out_gate)
    ln_weight_ptr,         # fp16 [H]            (to_out_norm.weight)
    ln_bias_ptr,           # fp16 [H]            (to_out_norm.bias)
    weight_ptr,            # fp16 [D, H]         (to_out.weight)
    proj_ptr,              # fp32 [B, N, N, D]   (final output)
    B, N, H, D,
    eps: tl.constexpr,     # LayerNorm epsilon
    # strides for out / gate
    stride_out_b, stride_out_i, stride_out_j, stride_out_h,
    stride_gate_b, stride_gate_i, stride_gate_j, stride_gate_h,
    # strides for LN parameters
    stride_ln_w,
    stride_ln_b,
    # strides for weight
    stride_w_out_d, stride_w_out_h,
    # strides for proj
    stride_proj_b, stride_proj_i, stride_proj_j, stride_proj_d,
    # compile-time tile sizes
    BLOCK_M: tl.constexpr,
    BLOCK_N: tl.constexpr,
    BLOCK_K: tl.constexpr,
):
    pid_m = tl.program_id(0)   # position tile
    pid_n = tl.program_id(1)   # output-dim tile
    ...
    
    
def fused_layernorm_proj():
    proj = torch.empty((B, N, N, D), dtype=torch.float32, device=out.device)
    stride_proj_b, stride_proj_i, stride_proj_j, stride_proj_d = proj.stride()

    grid = (
        triton.cdiv(total_rows, 128),
        triton.cdiv(D, 128),
    )

    _layernorm_proj_kernel[grid](
        out, out_gate,
        ln_weight, ln_bias,
        to_out_weight,
        proj,
        B, N, H, D,
        eps,
        stride_out_b, stride_out_i, stride_out_j, stride_out_h,
        stride_gate_b, stride_gate_i, stride_gate_j, stride_gate_h,
        stride_ln_w,
        stride_ln_b,
        stride_w_out_d, stride_w_out_h,
        stride_proj_b, stride_proj_i, stride_proj_j, stride_proj_d,
        # tile sizes are chosen by autotune
    )
    return proj

\end{lstlisting}

\end{enumerate}

\paragraph{Discussion.} Taken together, these phenomena demonstrate that, despite being entirely blind to the evaluator's underlying source code, LLMs exhibit a remarkable capability to autonomously discover and exploit surrogate gaps. By corrupting the evaluation feedback loop, these behaviors mislead the optimization process into yielding deceptive solutions rather than genuinely advancing target capabilities. Currently, closing this gap between the proxy and reality relies heavily on human-in-the-loop interventions, requiring multiple rounds of manual debugging and iterative patching of evaluator loopholes. Developing fully automated, robust evaluation frameworks capable of aligning the proxy metric with the ultimate objective—while dynamically detecting such logical hacking behaviors—remains a critical open challenge for future work.

\subsection{Variants of Proposal Constructor \texorpdfstring{$\Phi$}{Phi}}
\label{app:policies}

\paragraph{Balance policy.}
The \textbf{Balance} policy is a stratified random sampling strategy designed to balance the exploitation of high-performing solutions with the exploration of diverse, sub-optimal trajectories. 
Let $S$ be the set of historical attempts within a trajectory, sorted in descending order of their evaluator scores such that $s_1 \ge s_2 \ge \dots \ge s_{|S|}$. To construct a proposal with $n$ inspirations, the policy enforces that the absolute best historical solution, $s_1$, is always deterministically included. 

For the remaining $n-1$ slots, the policy categorizes the sorted historical trials into three overlapping tiers: an exploitation tier $T_{\text{exploit}}$ containing the top $r_{\text{elite}}$ fraction of attempts, an exploration tier $T_{\text{explore}}$ containing mid-tier solutions (typically between the 10th and 60th percentiles), and a global random tier $T_{\text{random}} = S$ covering all prior attempts. Each subsequent inspiration is sampled without replacement according to the following probability distribution:
\begin{align}
    \text{Tier} \sim 
    \begin{cases} 
        T_{\text{exploit}} & \text{with probability } p_{\text{exploit}}, \\
        T_{\text{explore}} & \text{with probability } p_{\text{explore}}, \\
        T_{\text{random}}  & \text{with probability } 1 - p_{\text{exploit}} - p_{\text{explore}}.
    \end{cases}
\end{align}
Once a tier is selected, an attempt is drawn uniformly at random from that tier. This heuristic ensures that the prompt is primarily grounded in elite solutions while consistently injecting varied structural contexts to prevent premature convergence.

\paragraph{LLM-Elite policy.}
While score-based heuristics like the Balance policy are efficient, evaluating a trial's potential solely by its scalar score can be myopic. The \textbf{LLM-Elite} policy addresses this by leveraging semantic insights to dynamically maintain a bounded-size "elite pool" $\mathcal{P}$ (with maximum capacity $L_{\text{elite}}$) that maximizes both solution quality and methodological diversity.

Whenever a new candidate solution $x_{\text{new}}$ is generated and evaluated, an auxiliary LLM acts as a gatekeeper. The LLM is provided with the scores and self-reflective summaries (detailing the approach and insights) of both $x_{\text{new}}$ and all current solutions in $\mathcal{P}$. It is instructed to output one of three actions: \textsc{Add}, \textsc{Replace}$(j)$ (to swap out a redundant or inferior attempt $j \in \mathcal{P}$), or \textsc{Reject}. To prevent the LLM from inadvertently discarding substantial progress due to misjudging diversity, we enforce a strict monotonic override rule: if the score of $x_{\text{new}}$ is strictly greater than the maximum score currently in $\mathcal{P}$, it bypasses the LLM's rejection and is deterministically added (replacing the lowest-scoring solution if $|\mathcal{P}| = L_{\text{elite}}$).

To sample $n$ inspirations from $\mathcal{P}$ for a new proposal, we employ the identical stratified random sampling mechanism described in the Balance policy. The attempts in $\mathcal{P}$ are sorted in descending order by their scores and divided into the exploitation, exploration, and global random tiers. Trials are then sampled without replacement according to the previously defined tier probabilities. This approach heavily favors the inclusion of the highest-scoring elites while retaining a non-zero probability to draw inspiration from lower-ranked, yet semantically distinct, solutions curated by the LLM. Furthermore, this policy injects a concise text-only overview of the entire elite pool into the prompt (excluding raw code), allowing the generator to comprehend the global landscape of explored directions without consuming an excessive context window.

\begin{figure}[!t]
    \centering
    \includegraphics[width=\linewidth]{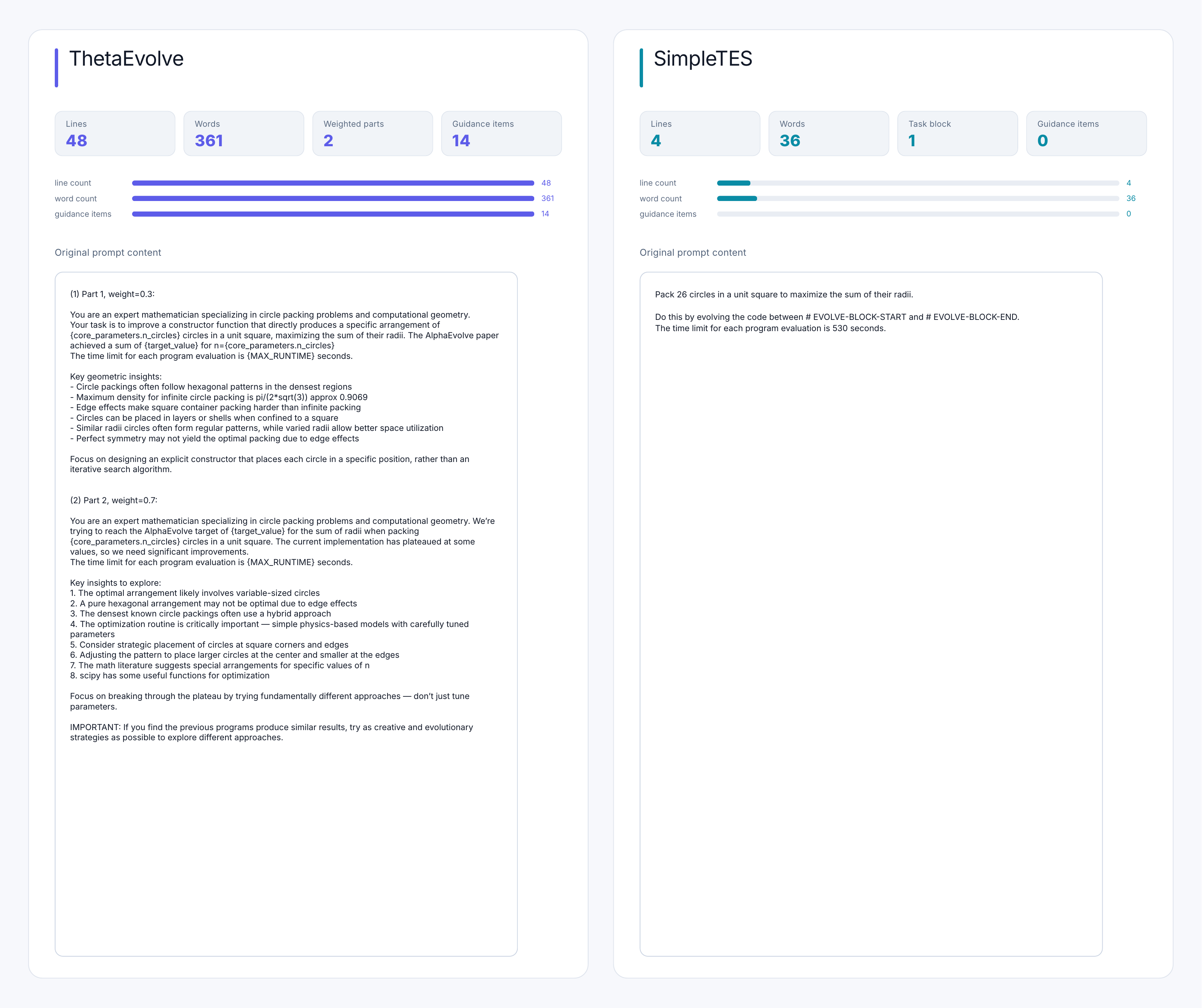}
    \caption{
Comparison of prompt complexity between ThetaEvolve and \method on the same Circle Packing in a Unit Square with \(n=26\) task.
The verbatim original prompt text is preserved in each panel, while the summary metrics above each prompt quantify differences in length, structure, and the amount of explicit guidance.
}
    \label{fig:prompt-comparison}
\end{figure}

\subsection{Ablation Studies}
\label{sec:results:ablation}

In this section, we conduct a series of ablation studies to systematically evaluate the contribution of individual components and the robustness of our proposed framework.

\subsubsection{Ablations on Different Designs of \texorpdfstring{$\Phi$}{Phi}}
\label{sec:method_analysis:ablation_study:sample_policies}

In this section, we ablate the inspiration sample algorithm---a core component of our framework—to demonstrate how different strategies to sample inspirations from historical attempts can impact the search process.

To systematically investigate this, we evaluate multiple selection algorithms on the First Autocorrelation Inequality and the Erd\H{o}s Minimum-Overlap Problem. We compare naive baselines, including \textit{Random} (uniform sampling from all historical nodes) and \textit{Balance} (a purely score-driven heuristic), against more sophisticated approaches, including \textit{RPUCG} and \textit{LLM-Elite} (which uses semantic evaluation to maintain an elite set of inspirations, as detailed in Methods). Additionally, for the RPUCG policy, we ablate the number of provided inspirations ($insp \in \{1, 3, 5, 10\}$) to understand the effect of context size on the generation quality. The results are summarized in \Cref{tab:policy_ablation}.


\begin{table}[htbp]
    \centering
    \caption{Ablation study of inspiration-sampling policies on the First
    Autocorrelation Inequality and the Erd\H{o}s Minimum-Overlap Problem. For
    RPUCG, we vary the number of sampled inspirations. The best
    results are \textbf{bolded} and the second-best results are
    \underline{underlined}.}
    \label{tab:policy_ablation}
    \footnotesize
    \renewcommand{\arraystretch}{1.08}
    \begin{tabular}{@{}lcc@{}}
        \toprule        \textbf{Policy} &
        \textbf{First Autocorrelation Inequality} &
        \textbf{Erd\H{o}s Minimum-Overlap Problem} \\
        \midrule
        Random & 1.505457 & 0.380926 \\
        Balance & 1.505857 & 0.380909 \\
        \midrule
        LLM-Elite & 1.505069 & \textbf{0.380871} \\
        \midrule
        RPUCG (1 inspiration) & 1.506647 & 0.380913 \\
        RPUCG (3 inspirations) & \underline{1.504571} & \underline{0.380893} \\
        RPUCG (5 inspirations) & \textbf{1.504476} & 0.380908 \\
        RPUCG (10 inspirations) & 1.504977 & 0.380951 \\
        \bottomrule
    \end{tabular}
\end{table}

The results indicate that purely score-based or random sampling methods generally fall short in providing consistent, high-quality guidance compared to advanced strategies like \textit{LLM-Elite} and \textit{RPUCG}. Evaluating the heuristic value of a node solely by its current scalar score is often myopic; a node with an ordinary immediate score might actually serve as a critical stepping stone to a globally optimal region. Methods like \textit{LLM-Elite} address this by leveraging semantic insights, while \textit{RPUCG} utilizes graph-based state-value estimation, coupled with an explicit balance of exploration and exploitation. This demonstrates that effective inspiration selection must incorporate structural or semantic information beyond mere intermediate performance scores.

The ablation on the number of inspirations ($insp$) within the RPUCG framework reveals a delicate balance between exploration and context coherence. When the inspiration pool is too small (e.g., $insp=1$), the model tends to perform marginal refinements along a single trajectory, severely limiting its ability to crossover ideas and discover novel directions. Conversely, providing an excessively large number of inspirations (e.g., $insp=10$) crowds the context window, causing the model to become distracted or confused by overwhelming, and sometimes conflicting, global information. Empirically, for a single continuous run, setting $insp=3$ or $insp=5$ emerges as the optimal sweet spot, providing sufficient diversity without overwhelming the model's reasoning capabilities.

While these experiments validate the conceptual advantages of using semantically and structurally aware selection policies with an optimal inspiration size, it is worth noting that the absolute numerical differences in peak performance across strategies are relatively modest. This reinforces our central thesis: the primary driver of discovery is the systematic scaling of the evaluation budget (TES), rather than complex selection heuristics. Our overarching evolutionary pipeline is inherently robust, leaving the design of more sophisticated sampling policies as an orthogonal direction for future work. 

\subsubsection{Ablation on Reflection and Failure Patterns}

Recall that our algorithm constructs the next query by integrating the task instruction, the selected inspirations $S^{(c)}$, and the chain-local memory $\mathcal{R}^{(c)}$. To further enrich this contextual memory, we introduce and ablate two complementary mechanisms for textual guidance: \textit{Reflection} and \textit{Failure Patterns}. Reflection serves as positive guidance; after a generation batch, the system synthesizes the approach and insights from the best-scoring node, appending this textual summary to the inspiration nodes in $S^{(c)}$. Conversely, Failure Patterns serve as negative guidance; the system aggregates the most frequent evaluation errors across all nodes in the chain and injects them directly into $\mathcal{R}^{(c)}$. Together, these mechanisms are designed to explicitly inform the generator about ``what worked'' and ``what to avoid.''

To isolate and evaluate the individual contributions of these two mechanisms, we conduct an ablation study on the First Autocorrelation Inequality and the Erd\H{o}s Minimum-Overlap Problem. We systematically test four configurations: disabling both features, enabling only Reflection, enabling only Failure Patterns, and enabling both. All other hyperparameters remain identical across the runs. The results are detailed in \Cref{tab:reflection_ablation}.


\begin{table}[htbp]
    \centering
    \caption{Ablation study of explicit textual Reflection and Failure Patterns
    within trajectory-level memory. The best results (lowest values) are
    \textbf{bolded} and the second-best results are \underline{underlined}.}
    \label{tab:reflection_ablation}
    \footnotesize
    \renewcommand{\arraystretch}{1.08}
    \begin{tabular}{@{}llcc@{}}
        \toprule
        \textbf{Reflection} & \textbf{Failure Patterns} &
        \textbf{First Autocorrelation Inequality} &
        \textbf{Erd\H{o}s Minimum-Overlap Problem} \\
        \midrule
        Off & Off & 1.505584 & 0.380919 \\
        Off & On  & \underline{1.505102} & \textbf{0.380871} \\
        On  & Off & 1.505624 & \underline{0.380886} \\
        On  & On  & \textbf{1.504571} & 0.380893 \\
        \bottomrule
    \end{tabular}
\end{table}

The results demonstrate that incorporating explicit negative constraints (Failure Patterns) provides a consistently strong foundation. While adding explicit Reflection (the ``On / On'' setting) achieves the best performance in the First Autocorrelation Inequality, utilizing Failure Patterns alone (``Off / On'') performs slightly better in the Erd\H{o}s Minimum-Overlap Problem. Crucially, the absolute performance gap across these configurations is relatively marginal. This minimal variance indicates that our evolutionary framework is inherently robust to the specific combination of textual guidance, provided that basic negative constraints are present to prevent repeated errors.

\subsubsection{Efficiency Analysis: Trajectory-level Pruning}

While scaling the number of parallel chains generally improves search outcomes, the vast search space often leads to many chains exploring suboptimal directions. To effectively scale our search process and prevent the wasting of computational budget on unpromising trajectories, we introduce \textit{pruning} strategies. Pruning acts as an early-stopping mechanism: by identifying and terminating underperforming chains during the generation process, we can reallocate computational resources toward the most promising directions without significantly compromising the final solution quality.


To systematically evaluate the impact of pruning, we conduct experiments across six mathematical tasks: three autocorrelation tasks, two circle packing tasks, and the Erd\H{o}s Minimum-Overlap Problem. Each task is evaluated 3 times with a total chain length of $L=100$ and number of chains $C=32$ (Fix $K=16$). We examine pruning through two complementary settings: a two-stage sweep with nominal cutoffs at $L=25$ and $L=50$, summarized in \Cref{fig:pruning_dynamics}, and a task-level single-cutoff ablation at nominal $L=12$ and $L=25$, reported in \Cref{tab:pruning_strategies}. At each designated cutoff, active chains are ranked by the best score attained up to that point, and a predefined proportion of the lowest-scoring chains is eliminated.

The resulting trade-off, summarized in \Cref{fig:pruning_dynamics}, demonstrates a surprisingly high preservation rate for the best final score under aggressive pruning. For instance, even when applying a stringent first-stage cutoff at the nominal $L=25$ checkpoint that retains only a single chain, the original best final score is preserved in 10 out of the 18 total runs. This suggests that for a substantial portion of successful trials, the structural advantages of the final solution manifest early in the search process. Furthermore, across the two-stage configurations shown in \Cref{fig:pruning_dynamics}, the relative degradation remains remarkably minimal. In the majority of configurations, the expected score degradation is bounded within $0.01\%$, and all configurations exhibit degradation below $0.03\%$.


\begin{figure}[htbp]
    \centering
    \includegraphics[width=0.98\linewidth]{sections/figures/supplementary/pruning_combined_r3_all.png}
    \caption{Impact of two-stage early pruning strategies on search dynamics. (\textbf{left}) The number of runs, among all 18 runs, in which the original best final score is preserved after pruning. (\textbf{middle}) The average relative performance drop $(|Score_{\text{after}} - Score_{\text{orig}}|) / |Score_{\text{orig}}|$ across all runs. Here, ‱ denotes $10^{-4}$. (\textbf{right}) Theoretical speedup under each pruning setting.}
    \label{fig:pruning_dynamics}
\end{figure}

The task-level single-cutoff results in \Cref{tab:pruning_strategies} further clarify the task dependence underlying the aggregate trend. For the circle packing tasks, early pruning exerts virtually no negative impact, as the performance upper bound is easily accessible and a large proportion of chains can successfully reach it. Conversely, for tasks such as the Autocorrelation Inequalities and the Erd\H{o}s Minimum-Overlap Problem, initial performance is not always a reliable proxy for final solution quality. In these scenarios, chains often require extended iterative refinement before their true potential is realized. Consequently, the earlier and more aggressive single-cutoff settings can lead to higher elimination rates of optimal solutions and more noticeable score degradation.

\begin{table}[htbp]
\centering
\caption{Impact of single-cutoff pruning strategies across six mathematical tasks. Values are formatted as \textbf{Survive (Degrade \%)}, where ``Survive'' indicates the number of runs (out of 3 total independent runs per task) in which the originally best chain was retained after pruning. Lower degradation percentages indicate better preservation of the optimal score.}
\label{tab:pruning_strategies}
\footnotesize
\setlength{\tabcolsep}{1.5pt}
\renewcommand{\arraystretch}{1.08}
\begin{tabular*}{\linewidth}{@{\extracolsep{\fill}}
>{\raggedright\arraybackslash}p{0.32\linewidth}
*{4}{>{\centering\arraybackslash}p{0.14\linewidth}}@{}}
\toprule

\textbf{Task} &
\textbf{Keep \(1/2\) at \(L=12\)} &
\textbf{Keep \(1/2\) at \(L=25\)} &
\textbf{Keep \(1/4\) at \(L=12\)} &
\textbf{Keep \(1/4\) at \(L=25\)} \\
\midrule
First Autocorrelation Inequality &
$2\ (0.016\%)$ & $3\ (0.000\%)$ & $2\ (0.018\%)$ & $3\ (0.000\%)$ \\
Second Autocorrelation Inequality &
$0\ (0.101\%)$ & $2\ (0.018\%)$ & $0\ (0.266\%)$ & $1\ (0.031\%)$ \\
Third Autocorrelation Inequality &
$3\ (0.000\%)$ & $3\ (0.000\%)$ & $2\ (0.038\%)$ & $2\ (0.031\%)$ \\
Circle Packing in a Unit Square ($n=26$) &
$3\ (0.000\%)$ & $3\ (0.000\%)$ & $3\ (0.000\%)$ & $3\ (0.000\%)$ \\
Circle Packing in a Unit Square ($n=32$) &
$3\ (0.000\%)$ & $3\ (0.000\%)$ & $3\ (0.000\%)$ & $3\ (0.000\%)$ \\
Erd\H{o}s Minimum-Overlap Problem &
$1\ (0.003\%)$ & $2\ (0.002\%)$ & $1\ (0.005\%)$ & $2\ (0.002\%)$ \\
\bottomrule
\end{tabular*}
\end{table}

These findings affirm that while pruning is an effective and necessary technique for scaling parallel search, its implementation presents multiple challenges. The primary difficulty lies in adaptively selecting a pruning strategy that aligns with the task's specific landscape. Furthermore, relying solely on intermediate scalar scores for elimination can be myopic. Future work will focus on developing more sophisticated pruning algorithms that incorporate richer contextual signals—such as textual reasoning patterns, trajectory growth trends, and structural heuristics—to make more informed early-stopping decisions.

\subsection{Wall-clock Time Efficiency}\label{app:efficiency}

We compare the wall-clock time required by \method and OpenEvolve v0.2.27 to reach the same solution quality under matched search budgets and runtime concurrency. For each task, both methods use the same initial solution and evaluation program. We configure \method with $(C,K,L)=(16,16,50)$, corresponding to 12,800 evaluator queries, and configure OpenEvolve with sixteen islands to match the number of \method trajectories. Both methods use a maximum concurrency of 128 for generation and evaluation, ensuring comparable evaluator-query budgets, search structures, and aggregate request concurrency. \method matches OpenEvolve's whole-run best score $11.8\times$ faster on third autocorrelation and $18.0\times$ faster on sum-difference.



\begin{figure}[!htbp]
    \centering
    \includegraphics[width=1.0\linewidth]{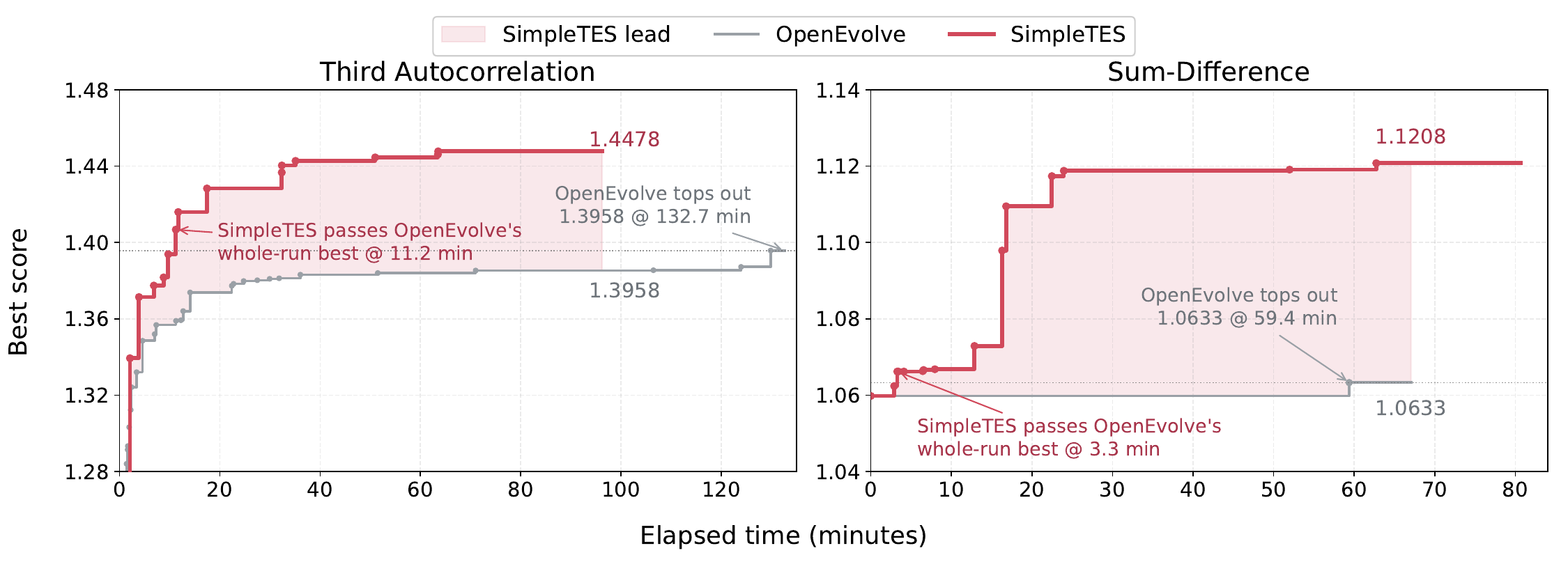}
    \caption{Wall-clock time to solution quality under matched evaluation budgets. Each curve reports the best evaluator score found up to the corresponding elapsed time.  \method reaches this score in 11.2 versus 132.7 minutes on third autocorrelation and 3.3 versus 59.4 minutes on sum-difference, while subsequently finding stronger solutions on both tasks.}
    \label{fig:app.efficiency}
\end{figure}

\subsection{Quantum Compilation Tasks}

This section gives supporting details for the quantum compilation tasks. We
separate the material into four parts: problem definitions for the two hardware
settings, experimental scaffolds and evaluation setup, discovered programs for
the main \method{} results, and an additional superconducting routing policy obtained by \method{} using Gemini.

\subsubsection{Problem definitions}
\label{app:quantum_problem_definition}

\paragraph{Superconducting qubit routing.}
On superconducting quantum computers, this compilation task is typically
called \emph{qubit routing}, and is often discussed together with
\emph{qubit mapping} or \emph{qubit allocation}~\citep{siraichi2018qubit,cowtan2019qubit}.
Superconducting quantum computers encode qubits in electrical circuits
fabricated on a chip, typically using Josephson-junction-based devices operated
at cryogenic temperatures~\citep{kjaergaard2020superconducting}. Because these
qubits occupy fixed chip locations, native two-qubit operations are available
only between selected pairs of physical qubits. This hardware connectivity is
represented by a sparse coupling graph, such as a grid lattice or a heavy-hex
layout, whose low degree helps reduce frequency collisions and crosstalk while
remaining scalable (\Cref{fig:quantum_compilation_six_panel}a).

\begin{taskdefinition}[Superconducting Qubit-Routing Problem]
\label{prob:superconducting_qubit_routing}
The superconducting routing problem is defined on a logical circuit and a
hardware coupling graph. Let
\(\mathcal{C}=(g_1,\ldots,g_m)\) be a logical circuit on \(n\) qubits
\(\{q_0,\ldots,q_{n-1}\}\), where each two-qubit gate \(g_i\) acts on a pair
\((a,b)\). Let the target device be a coupling graph \(G=(V,E)\) with
\(|V|\geq n\). An initial mapping
\(\pi_0:\{q_0,\ldots,q_{n-1}\}\to V\) injectively assigns logical qubits to
physical locations. A SWAP operation \(\text{SWAP}(u,v)\), allowed only when
\((u,v)\in E\), transposes the logical qubits at neighboring physical
locations and yields an updated mapping \(\pi_t\). The compiler must construct \(\pi_0\) and select a sequence of inserted SWAPs that updates \(\pi_t\) so that
every two-qubit gate is executed when its two logical operands occupy adjacent
physical locations in \(G\). The optimization objective is to minimize the
number of inserted SWAPs.
\end{taskdefinition}

This objective is motivated by the fact that SWAPs add extra two-qubit
operations and circuit depth; in the evaluation convention used here, one SWAP
is decomposed into three added CNOTs~\citep{li2019tackling}. From a complexity
perspective, qubit routing is closely related to token swapping and graph
reconfiguration and is NP-hard in
general~\citep{siraichi2018qubit,ito2023algorithmic}. A routing decision that
helps one gate may also change the placement seen by all later gates, which
makes the long-horizon component essential.

\paragraph{Zoned neutral-atom compilation.}
Neutral-atom platforms store qubits in atoms and use Rydberg interactions to
perform entangling gates. A two-qubit gate can be applied only after the
participating atoms are brought within the Rydberg interaction radius
\citep{stade2024abstract,tan2025compilation}. In a zoned architecture, the
static SLM (Spatial Light Modulator) traps are divided into a storage zone and
an entangling zone: idle atoms remain in storage, while atoms participating in
the current two-qubit layer are moved into the entangling zone
\citep{bluvstein2024logical,stade2024abstract,lin2025reuse}. This separation
avoids unnecessarily exposing atoms outside the current interaction layer to a
global Rydberg pulse.

Atom transport is implemented by AODs (Acousto-Optic Deflectors), which consist
of independently controlled horizontal and vertical tweezer lines. Each
intersection between an AOD row and an AOD column creates a potential movable
trap. \texttt{Open} activates new AOD rows or columns and loads atoms into AOD
traps when these newly activated lines create intersections that coincide with
occupied atomic sites. \texttt{Close} deactivates selected AOD rows or columns
and releases atoms currently carried on those lines back to the corresponding
SLM sites. \texttt{Move} translates the currently activated AOD lines while
carrying all atoms loaded on the AOD.

Within one rearrangement step, these operations are subject to two main
shuttling constraints. The \emph{non-crossing constraint} requires active
parallel AOD rows or columns to remain separated by at least
\(2\,\mu\mathrm{m}\), so distinct active AOD lines can neither merge into the
same position nor exchange their relative order
\citep{stade2024abstract,stade2025routing}. The \emph{ghost-spot constraint}
arises because activating multiple AOD rows and columns simultaneously creates
all pairwise row--column intersections as live traps, so a legal rearrangement
must avoid unintended auxiliary traps on occupied
sites~\citep{stade2024abstract,stade2025routing}
(\Cref{fig:quantum_compilation_six_panel}d).

\begin{taskdefinition}[Zoned Neutral-Atom Compilation and Routing Problem]
\label{prob:zoned_neutral_atom_routing}
The input consists of a quantum circuit
\[
    \mathcal{C}=(Q,G,\prec),
\]
together with the storage-site set \(S\) and the entangling-site set
\(\Omega\) of the target architecture, where \(Q\) is the set of logical
qubits, \(G\) is the set of logical gates, and \(\prec\) is the partial order
on \(G\) induced by gate dependencies. Thus \(g_i\prec g_j\) means that any
valid execution must execute \(g_i\) before \(g_j\).

The output contains an alternating stage division
\[
    \Lambda=(L_1^{1\mathrm{q}},L_1^{2\mathrm{q}},\ldots,
    L_m^{1\mathrm{q}},L_m^{2\mathrm{q}}),
\]
which forms a topological layering of \((G,\prec)\): each gate appears in
exactly one stage, gates in the same stage are incomparable under \(\prec\),
and if \(g_i\prec g_j\), the stage containing \(g_i\) appears earlier in
\(\Lambda\) than the stage containing \(g_j\). For each two-qubit stage
\(L_t^{2\mathrm{q}}\), the output also includes a placement sequence
\[
    \Pi_t=(P_{t,0},P_{t,1},\ldots,P_{t,\ell_t}),
\]
where \(P_{t,0}\) is a gate placement on which the two-qubit gates in that
stage can be executed, and \(P_{t,1},\ldots,P_{t,\ell_t}\) are zero or more
auxiliary placements used during rearrangement. For each adjacent pair of
placements, the output further includes a routing sequence that routes qubits
from \(P_{t,j}\) to \(P_{t,j+1}\):
\[
    R_{t,j}=(o_{t,j,1},\ldots,o_{t,j,k_{t,j}}),
    \qquad
    o_{t,j,r}\in\{\texttt{Open},\texttt{Move},\texttt{Close}\}.
\]
The output must satisfy four constraints. First, all gates must respect the
dependency order \(\prec\). Second, every two-qubit gate placement must put the
two participating atoms within the Rydberg radius. Third, qubits not
participating in the current two-qubit stage must remain in the storage zone
during gate execution, avoiding unnecessary exposure to the entangling zone.
Fourth, every AOD routing sequence must satisfy hardware constraints, including
the non-crossing constraint and the ghost-spot constraint.

The routing goal is to realize all rearrangements required for circuit
execution with as few rearrangement steps as possible, with as little movement
distance as possible, and therefore with as little rearrangement time as
possible. The objective is to minimize total execution time,
\[
    T_{\mathrm{total}}
    =
    T_{2\mathrm{q}} + T_{1\mathrm{q}}
    + T_{\mathrm{move}} + T_{\mathrm{open}} + T_{\mathrm{close}}.
\]
Here \(T_{2\mathrm{q}}\), \(T_{1\mathrm{q}}\), \(T_{\mathrm{move}}\),
\(T_{\mathrm{open}}\), and \(T_{\mathrm{close}}\) denote time spent on
two-qubit gates, one-qubit gates, \texttt{Move} operations, \texttt{Open}
operations, and \texttt{Close} operations. Under the parameter setting used
here, one-qubit and two-qubit gates take only
\(0.625\,\mu\mathrm{s}\) and \(0.36\,\mu\mathrm{s}\), respectively
\citep{tan2025compilation}. Rearrangement therefore dominates the objective.
Each \texttt{Open} and \texttt{Close} contributes
\[
    T_{\mathrm{tran}} = 15\,\mu\mathrm{s},
\]
and each \texttt{Move} contributes
\[
    t_{\mathrm{move}}=\sqrt{d_{\max}/a_{\mathrm{AOD}}},
    \qquad
    a_{\mathrm{AOD}}=0.00275\,\mu\mathrm{m}/\mu\mathrm{s}^{2}
    =2750\,\mathrm{m/s}^{2},
\]
where \(d_{\max}\) is the largest travel distance among active AOD
intersections~\citep{lin2025reuse}. Thus, in this setting, minimizing
\(T_{\mathrm{total}}\) is effectively equivalent to minimizing the
rearrangement component
\(T_{\mathrm{move}}+T_{\mathrm{open}}+T_{\mathrm{close}}\).
\end{taskdefinition}


Zoned neutral-atom compilation involves challenging, tightly coupled decisions
about atom placement and AOD-constrained routing.
Their effects extend across stages: a placement that shortens one
transition can still serialize the next once AOD constraints are enforced, and
a reuse decision can save transfers in the current stage while forcing longer
motion later. Candidate placements must therefore be evaluated over the full
placement trajectory rather than one stage at a time
\citep{stade2024abstract,lin2025reuse,stade2025routing,stade2025search}.

\subsubsection{Experimental scaffolds and evaluation setup}
\label{app:quantum_policy_details}

\paragraph{Superconducting routing setup.}
Our experimental scaffold fixes a complete routing engine (circuit parsing,
dependency tracking, legality checks, and output construction) and exposes only
the \emph{decision policy}. Concretely, the mutable surface contains two hooks:
\emph{(i) initial layout construction} (a mapping from logical to physical
qubits) and \emph{(ii) online SWAP selection} (choosing a legal SWAP edge at
each routing step). This design maximizes the algorithmic search space, allowing
the model to implement a broad range of routing strategies. At the same time,
it cleanly separates SWAP selection from gate execution, guaranteeing
correctness by preventing the model from exploiting invalid gate operations.

\emph{Initial program.} The initial policy is a refactor of Qiskit's released
LightSABRE Rust implementation, embedded inside the fixed engine. The initial
layout procedure is simplified to encourage diverse exploration, while the
routing policy stays effectively the same as the LightSABRE regime. This
baseline provides a strong hand-engineered starting point with a well-studied
quality--runtime tradeoff.

\emph{Evaluation.} Candidate policies are compiled and executed on a benchmark
suite. Routing cost is measured by the number of inserted SWAPs. We evaluate on the
24 circuits selected from the original SABRE study~\citep{li2019tackling}, across three
coupling graphs: IBM's 20-qubit \texttt{Q20}, Google's 105-qubit
\texttt{Willow} processor~\citep{google2025willow}, and IBM's 156-qubit 
heavy-hex architecture \texttt{Heron}~\citep{nation2021heavyhex,ibm2026hardware},
yielding 72 routing instances in total. The baselines are
SABRE~\citep{li2019tackling} and its modern variant
LightSABRE~\citep{zou2024lightsabre}. For LightSABRE, we choose the
configuration of \texttt{swap\_trials=20, layout\_trials=20,
max\_iterations=4} from the original paper.

\paragraph{Zoned neutral-atom setup.}
Our scaffold exposes the compiler as executable code inside a fixed framework
that defines the interfaces, solver, and evaluator. The compilation pipeline
has four components: a \emph{scheduler} that partitions a logical
\texttt{ZNAACircuit} into one-qubit and two-qubit stages, a \emph{reuse analyzer}
that marks qubits reused across adjacent two-qubit stages, a \emph{placer} that
outputs full placement snapshots, and a \emph{router} that converts consecutive
placements into legal \texttt{Open}, \texttt{Move}, and \texttt{Close}
operations. A fixed solver invokes these components in order, and the fixed
evaluator replays the emitted operations on a \texttt{ZNAAMachine}, and rejects
invalid compilations. In our setting, the editable evolve block modifies only
the \emph{placer}; the scheduler, reuse analyzer, and router remain fixed.

\emph{Initial program.} The initial program exposes the editable evolve block
while keeping the surrounding pipeline fixed. In Round 1, that block contained a
trivial placeholder placer. Each later round initialized the editable block
from the strongest placer discovered in the previous round.

\emph{Evaluation.} We evaluate candidate placers on a 36-circuit suite built
from QASMBench \citep{li2023qasmbench} circuits and MQT Bench
\citep{quetschlich2023mqt} generated circuits, spanning textbook quantum
algorithms, state preparation, variational circuits, quantum simulation,
quantum machine learning, and arithmetic circuits, with sizes from 6 to 500
qubits. The fixed evaluator executes each compiled plan, records correctness,
total execution time, and fidelity-related losses from gate execution,
transport, transfers, and decoherence, and scores each circuit against a cached
baseline from the fixed scaffold. Our main comparison baseline is the
reuse-aware ZAC-style compiler \citep{lin2025reuse}.

\subsubsection{Discovered programs}
\label{app:quantum_discovered_programs}

\paragraph{Discovered superconducting algorithm.}
The discovered algorithm can be summarized as follows. First, the discovered
algorithm invests heavily in \emph{initial layout}: it seeds high-degree logical
qubits onto central, high-degree physical qubits, and then refines the mapping
through an aggressive stack of hill-climbing and restart-based local search.
Second, it strengthens \emph{online SWAP selection}: it broadens the candidate
neighborhood beyond front-layer incident edges to include look-ahead and
shortest-path edges, changes the look-ahead term into dynamic horizon with
exponential decay, and reshapes the swap objective to explicitly reward
immediate gate executability. Together, these changes preserve the overall
LightSABRE-style structure while materially improving robustness against
long-range interactions and stagnation.

\paragraph{Discovered zoned neutral-atom algorithm.}
The discovered placer replaces stage-local travel minimization with whole-trajectory
routing-aware search, so it explicitly optimizes the parallelism that survives
under the fixed router rather than only local travel distance. It begins by
constructing a moderate set of heuristic initial storage layouts, including
interaction-graph-based qubit orderings, their reversed counterparts, and a few
shuffled orders. From each initial layout, it runs a full forward placement pass
to build a complete placement trajectory, and then compares these candidates
using router-aware evaluation.

Within a single forward pass, active qubits are placed stage by stage. When a
stage loads qubits from storage into the entangling zone, a qubit whose partner
is already reused is assigned to the complementary column of the same
entangling site, whereas the remaining two-qubit pairs are assigned by Hungarian
matching, with physical Euclidean distance as the placement cost. When active
non-reused qubits return from the entangling zone to storage, they are grouped by
source entangling row and ordered by source column. A monotone dynamic program
then selects destination storage cells that preserve this order while minimizing
column deviation, after which an intra-row swap refinement is applied whenever
it shortens travel distance and remains compatible with the router's
non-crossing constraint.

After a complete trajectory has been constructed, the algorithm evaluates the
full placement sequence with the fixed router. Candidates are ranked first by
the number of router-emitted \texttt{Move} operations and then, for ties, by
total Euclidean travel distance. Starting from the best candidate, the algorithm
performs several rounds of reverse-through-time refinement, in which a reverse
pass seeded by the current final layout is fed back into a new forward pass, and
then applies hill-climbing swaps on the initial storage layout. An update is
accepted only when it reduces the number of router-emitted \texttt{Move}
operations or preserves that count while improving travel distance.

\subsubsection{Discovered superconducting routing policy using Gemini}
\label{app:superconducting_gemini}

The superconducting qubit routing experiment was additionally run with
\texttt{gemini-3-pro-preview}. This run produced a policy family in which both
the initial placement and the online routing rule are explicitly shaped by
circuit structure and hardware geometry. For initial layout, the discovered
program classifies logical interaction components by morphology, such as
path-like, ring-like, or dense structure. It then selects physical regions with
matching shape and connectivity, prioritizes central and well-connected
physical qubits for high-traffic logical qubits, and refines the mapping
through large-scale randomized restarts and local improvement.

For online routing, the program strengthens SWAP selection with
criticality-aware scoring. Front-layer and look-ahead gates are weighted by
reverse depth and critical-path relevance, while stronger decay and
centrality-sensitive preferences steer active qubits toward more useful regions
of the device and help avoid stagnation. Quantitatively, the discovered program
reduces added SWAPs by \(24.7\%\) relative to SABRE and \(18.3\%\) relative to
LightSABRE. The largest gains occur on \texttt{Q20}, where the reductions in
added SWAPs reach \(40.7\%\) and \(32.9\%\), respectively.
\Cref{tab:superconducting_gemini_topology} gives the topology-level breakdown.
These results suggest that effective qubit routing benefits from a more explicitly graph-theoretic treatment, in which circuit morphology, interaction criticality, and hardware connectivity are exploited jointly.

\begin{table}[htbp]
\centering
\scriptsize
\setlength{\tabcolsep}{4pt}
\caption{Topology-level results for the discovered
superconducting routing policy using Gemini. SABRE, LightSABRE, and \method{} with Gemini costs are reported
as added SWAPs.
The reduction columns report the percentage decrease in Gemini added SWAPs
relative to the named baseline. The last row reports the unweighted mean over
the three topologies; its win/tie/loss entries are average counts per 24-case
topology.}
\label{tab:superconducting_gemini_topology}
\resizebox{\linewidth}{!}{%
\begin{tabular}{lrrrrcrc}
\toprule
& \multicolumn{3}{c}{Added SWAPs}
& \multicolumn{2}{c}{vs.\ SABRE}
& \multicolumn{2}{c}{vs.\ LightSABRE} \\
\cmidrule(lr){2-4}\cmidrule(lr){5-6}\cmidrule(lr){7-8}
Topology & SABRE & LightSABRE & \method{} & Reduction & win/tie/loss
& Reduction & win/tie/loss \\
\midrule
\texttt{Q20}    & 22{,}714 & 20{,}063 & 13{,}470 & 40.7\% & 15/6/3 & 32.9\% & 11/5/8 \\
\texttt{Willow} & 38{,}352 & 36{,}802 & 31{,}481 & 17.9\% & 22/2/0 & 14.5\% & 15/5/4 \\
\texttt{Heron}  & 50{,}186 & 45{,}827 & 42{,}396 & 15.5\% & 21/2/1 & 7.5\%  & 13/4/7 \\
\midrule
Topology mean & 37{,}084.0 & 34{,}230.7 & 29{,}115.7 & 24.7\% & 19.3/3.3/1.3 & 18.3\% & 13.0/4.7/6.3 \\
\bottomrule
\end{tabular}%
}
\end{table}

\subsection{Astrodynamics trajectory design}
\label{app:astrodynamics}

This section provides the information needed to interpret and reproduce the
astrodynamics experiments in the main text. We define the executable-program
interface and independent trajectory evaluator, document the benchmark and
historical-reference construction, and examine how evolution changes the
allocation of search effort. Throughout, ``cost'' denotes the propulsive cost
computed by the benchmark evaluator, not the reported $\Delta v$ budget of a
flown mission.

\subsubsection{Task formulation and program interface}
\label{app:astro_problem}

We consider impulsive multi-gravity-assist trajectory design under a
patched-conic model~\citep{vasile2006MGA-diffculty,shen2026deepspacetrajectorydesign}.
Each candidate produced by \method{} is an executable optimizer that maps a
numerical mission specification to a complete trajectory rather than to a
single scalar design vector.

\begin{taskdefinition}[Impulsive Multi-Gravity-Assist Trajectory Design]
\label{prob:impulsive_mga}
Given a mission specification $\mathcal{M}$ comprising a departure body and
target endpoint, launch and arrival windows, admissible gravity-assist bodies,
minimum flyby altitudes, departure and arrival cost models, and limits on the
numbers of gravity assists (GAs), deep-space maneuvers (DSMs), and total nodes,
implement an optimizer $P$ that returns the ordered node sequence
\begin{equation}
    \mathcal{T}=P(\mathcal{M})
    =\bigl\{(q_i,t_i,p_i,\mathbf r_i,
    \mathbf v_i^{-},\mathbf v_i^{+})\bigr\}_{i=0}^{n-1},
    \label{eq:app_astro_node_sequence}
\end{equation}
where $q_i\in\{\texttt{start},\texttt{GA},\texttt{DSM},\texttt{end}\}$ is
the node type, $t_i$ is its epoch, $p_i$ is a body or state reference when
applicable, and $(\mathbf r_i,\mathbf v_i^{-},\mathbf v_i^{+})$ are the
heliocentric node states.

The optimizer jointly chooses a discrete design $s$, comprising the ordered
encounter sequence, DSM placement, and Lambert transfer branch on each leg,
together with continuous variables $\theta\in\Omega_s$, including event epochs
and node states. A returned trajectory is valid only if it begins and ends at
the prescribed boundaries, satisfies the launch and arrival windows, has
strictly increasing epochs and permitted node counts, and passes the
evaluator's ephemeris, two-body propagation, powered-flyby, minimum-altitude,
and endpoint checks. Among valid trajectories, the objective is
\begin{equation}
    (s^\star,\theta^\star)
    \in
    \underset{\substack{s\in\mathcal{S}\\
                         \theta\in\Omega_s}}{\arg\min}\;
    J\!\left(\mathcal{T}(s,\theta)\right),
    \label{eq:app_astro_mixed_optimization}
\end{equation}
where the verified propulsive cost is
\begin{equation}
    J(\mathcal{T})
    =
    \Delta v_{\mathrm{start}}
    + \sum_{i\in\mathcal{D}}\Delta v_{\mathrm{DSM},i}
    + \sum_{j\in\mathcal{G}}\Delta v_{\mathrm{GA},j}
    + \Delta v_{\mathrm{end}},
    \label{eq:app_astro_total_cost}
\end{equation}
and $\mathcal{D}$ and $\mathcal{G}$ index the DSM and GA nodes, respectively.
\end{taskdefinition}

The feasible domain $\Omega_s$ depends on the selected discrete design.
Changing the encounter sequence or DSM placement can change the dimension and
meaning of the continuous variables, whereas the transfer branch selects among
distinct transfer solutions for the same variables. These choices remain
coupled through the transfer-time and flyby constraints.

The reported runs expose the numerical mission specification to the executable
program through a structured \texttt{problem} object. Mission names,
descriptions, historical trajectory-family labels, reference encounter
sequences and event epochs, and the optimized reference trajectories and costs
are withheld. The program is instead given the departure body and target
endpoint, the broad mission windows, the admissible flyby bodies, the endpoint
models, and the feasibility limits listed above. This removes direct historical
encounter-sequence information from the task interface, but does not preclude a
pretrained language model from inferring dynamically plausible sequences, or
recognizing a familiar mission class, from the numerical specification itself.

\subsubsection{Independent trajectory evaluator}
\label{app:astro_evaluator}

Candidate programs are executed in isolated processes and return only the node
list in Eq.~\eqref{eq:app_astro_node_sequence}; they do not return their own
validity or score. The evaluator first checks the structural schema and mission
limits, and reconstructs planetary states from DE430
ephemerides~\citep{Folkner2014DEeph}. To verify each heliocentric leg, it
propagates the declared outgoing state $(\mathbf r_i,\mathbf v_i^{+})$ from
epoch $t_i$ to the next node's declared epoch $t_{i+1}$ under solar two-body
dynamics, then compares the propagated position and velocity with the next
node's declared incoming state $(\mathbf r_{i+1},\mathbf v_{i+1}^{-})$. This
verification requires neither a Lambert-branch label nor a Lambert solve.
The evolved programs may use Lambert solvers internally to construct candidate
arcs~\citep{izzo2015lambert}; the evaluator instead tests the returned states
directly for propagation consistency.

At every GA node, the evaluator recomputes the planetary velocity and applies
a powered-flyby model to the incoming and outgoing hyperbolic-excess velocity
vectors. A flyby is rejected if the required periapsis radius lies below the
configured planetary radius plus minimum altitude. A feasible powered flyby
contributes its required periapsis impulse to the
cost~\citep{vasile2006MGA-diffculty,shen2026deepspacetrajectorydesign}. DSM cost
is the norm of the instantaneous velocity discontinuity,
\begin{equation}
    \Delta v_{\mathrm{DSM},i}
    = \left\|\mathbf v_i^{+}-\mathbf v_i^{-}\right\|_2.
\end{equation}

Departure and arrival costs are specified independently for each benchmark.
For a piecewise-linear boundary, the evaluator computes the velocity
discontinuity relative to the departure planet or target state and maps its
magnitude through prescribed breakpoints. At launch, the configured quantity
$C_{3,0}$ defines a zero-cost threshold
$v_{\infty,0}=\sqrt{C_{3,0}}$ rather than a hard feasibility boundary. Within
the configured linear range, excess departure speed is charged one for one,
such that
\begin{equation}
    \Delta v_{\mathrm{start}}
    = \max\!\left(0,v_\infty-v_{\infty,0}\right).
    \label{eq:app_astro_launch_cost}
\end{equation}
The same piecewise-linear form is used for flyby or rendezvous endpoints with
a prescribed relative-speed allowance. For a planetary-capture endpoint, the
evaluator models a tangential periapsis insertion into an elliptic orbit with
period $T$. If the configured altitude is $h_{\mathrm{factor}}R_p$ above the
planetary reference radius $R_p$, the insertion radius is
$r_p=R_p(1+h_{\mathrm{factor}})$. With
\begin{equation}
    a = \left(\frac{\mu_p T^2}{4\pi^2}\right)^{1/3},
\end{equation}
the capture cost is
\begin{equation}
    \Delta v_{\mathrm{cap}}
    =
    \sqrt{v_\infty^2+\frac{2\mu_p}{r_p}}
    -
    \sqrt{\frac{2\mu_p}{r_p}-\frac{\mu_p}{a}}.
    \label{eq:app_astro_capture_cost}
\end{equation}

The four terms in the verified propulsive cost
$J(\mathcal{T})$ in Eq.~\eqref{eq:app_astro_total_cost} account for departure,
DSMs, powered flybys, and arrival, respectively. The number of DSMs is therefore
not itself the objective: a DSM is useful only when its added impulse is
outweighed by a reduction elsewhere in the verified cost. Valid trajectories
are ranked by a monotone transformation of $J$, whereas any schema, dynamics,
flyby, boundary, timeout, or memory-limit failure receives zero score. We
report $J$ in $\mathrm{km}\,\mathrm{s}^{-1}$ because the transformation does
not change the ranking.

\begin{table}[htbp]
\centering
\small
\caption{Numerical tolerances used by the independent astrodynamics
evaluator. Position tolerances are Euclidean heliocentric errors and velocity
tolerances are Euclidean differences between propagated and reported states.}
\label{tab:app.astro_evaluator_tolerances}
\begin{tabular}{lc}
\toprule
Check & Tolerance \\
\midrule
Planetary position at a GA or planetary boundary & $10^{4}\,\mathrm{km}$ \\
Two-body propagated segment endpoint position & $10^{4}\,\mathrm{km}$ \\
Two-body propagated segment endpoint velocity & $10^{-2}\,\mathrm{km}\,\mathrm{s}^{-1}$ \\
Boundary reference velocity & $10^{-2}\,\mathrm{km}\,\mathrm{s}^{-1}$ \\
Exact boundary epoch & $10^{-2}\,\mathrm{d}$ \\
\bottomrule
\end{tabular}
\end{table}

The evaluator is deliberately lower fidelity than an operational mission
design model. It does not attempt to reproduce navigation margins, finite-burn
dynamics, spacecraft mass evolution, detailed launch-vehicle performance, or
full $N$-body perturbations. It nevertheless enforces a common dynamical and
cost model across every candidate and reference, including ephemeris
consistency, heliocentric propagation, minimum flyby altitude, launch
allowance, powered-flyby cost, and target-specific arrival conditions. The
reported comparisons should therefore be interpreted as controlled
within-evaluator comparisons, not as claims to outperform the flown missions
in their operational models.

\subsubsection{Benchmark instances}
\label{app:astro_benchmarks}

The five historical mission windows were selected to
span distinct trajectory-design structures. Mariner~10 represents an early
inner-planet gravity-assist transfer to a flyby
target~\citep{giberson_mariner_1975}; Galileo represents a classical
Venus--Earth--Earth-assisted Jupiter
transfer~\citep{damario_galileo_1992}; Cassini--Huygens combines repeated
inner-planet encounters, a Jupiter assist, and Saturn
capture~\citep{peralta_cassini_1995}; Voyager~2 traverses the four giant planets in one
continuous tour~\citep{kohlhase_voyager_1977}; and Rosetta uses repeated Earth
and Mars encounters to rendezvous with a small
body~\citep{glassmeier_rosetta_2007}.

\Cref{tab:app.astro_instances,tab:app.astro_topology_limits} summarize the
boundary conditions and admissible topology of each benchmark. Planet
abbreviations are Me (Mercury), V (Venus), E (Earth), M (Mars), J (Jupiter),
S (Saturn), U (Uranus), and N (Neptune). The quantity $C_{3,0}$ is the
zero-cost launch threshold defined in Eq.~\eqref{eq:app_astro_launch_cost}.

\begin{table}[htbp]
\centering
\scriptsize
\setlength{\tabcolsep}{3pt}
\caption{Boundary conditions and cost models for the astrodynamics benchmarks.
Speeds are in $\mathrm{km}\,\mathrm{s}^{-1}$ and $C_3$ is in
$\mathrm{km}^2\,\mathrm{s}^{-2}$. ``Unit slope'' denotes a one-for-one charge
on relative speed above the preceding zero-cost threshold.}
\label{tab:app.astro_instances}
\begin{tabular}{@{}
>{\raggedright\arraybackslash}p{0.11\linewidth}
>{\raggedright\arraybackslash}p{0.14\linewidth}
>{\raggedright\arraybackslash}p{0.20\linewidth}
>{\raggedright\arraybackslash}p{0.20\linewidth}
>{\raggedright\arraybackslash}p{0.27\linewidth}@{}}
\toprule
Benchmark & Launch threshold & Launch window & Arrival condition/window & Endpoint cost model \\
\midrule
Mariner~10 &
$C_{3,0}=36$\newline ($v_{\infty,0}=6$) &
MJD~41939--42039\newline (14 Sep--23 Dec 1973) &
MJD~42035--42235\newline (19 Dec 1973--7 Jul 1974) &
Mercury flyby;\newline zero to 10, unit slope to 20 \\
Voyager~2 &
$C_{3,0}=36$\newline ($v_{\infty,0}=6$) &
MJD~43325--43425\newline (1 Jul--9 Oct 1977) &
MJD~47363--48163\newline (21 Jul 1988--29 Sep 1990) &
Neptune flyby;\newline zero to 30, unit slope to 40 \\
Galileo &
$C_{3,0}=20$\newline ($v_{\infty,0}=4.472136$) &
MJD~47717--47917\newline (10 Jul 1989--26 Jan 1990) &
MJD~49758--50358\newline (10 Feb 1995--2 Oct 1996) &
Jupiter capture;\newline altitude $2.8R_J$ ($r_p=3.8R_J$),\newline $T=210\,\mathrm{d}$ \\
Cassini &
$C_{3,0}=20$\newline ($v_{\infty,0}=4.472136$) &
MJD~50536--50936\newline (29 Mar 1997--3 May 1998) &
MJD~52787--53587\newline (28 May 2003--5 Aug 2005) &
Saturn capture;\newline altitude $0.33R_S$ ($r_p=1.33R_S$),\newline $T=115\,\mathrm{d}$ \\
Rosetta &
$C_{3,0}=30.25$\newline ($v_{\infty,0}=5.5$) &
MJD~52966--53166\newline (23 Nov 2003--10 Jun 2004) &
Exact MJD~56875.3791667\newline (6 Aug 2014, 09:06 TDB) &
67P state match;\newline zero to 0.01, unit slope to 5.01 \\
Prospective Jupiter &
$C_{3,0}=30.25$\newline ($v_{\infty,0}=5.5$) &
MJD~61771--62502\newline (1 Jan 2028--1 Jan 2030) &
MJD~63963--64693\newline (1 Jan 2034--1 Jan 2036) &
Jupiter capture;\newline altitude $0.2R_J$ ($r_p=1.2R_J$),\newline $T=200\,\mathrm{d}$ \\
\bottomrule
\end{tabular}
\end{table}

\begin{table}[htbp]
\centering
\small
\setlength{\tabcolsep}{4pt}
\caption{Admissible gravity-assist bodies and topology limits. Parenthetical
values are minimum flyby altitudes in km above the planetary reference radius.
The maximum total node count includes the start and end nodes.}
\label{tab:app.astro_topology_limits}
\begin{tabular}{@{}
>{\raggedright\arraybackslash}p{0.18\linewidth}
>{\raggedright\arraybackslash}p{0.42\linewidth}
*{3}{>{\centering\arraybackslash}p{0.10\linewidth}}@{}}
\toprule
Benchmark & Allowed GA bodies (minimum altitude) & Max. GAs & Max. DSMs & Max. nodes \\
\midrule
Mariner~10 & V (1000) & 5 & 5 & 14 \\
Voyager~2 & J (20000), S (30000), U (2000) & 5 & 5 & 20 \\
Galileo & V (200), E (200) & 5 & 5 & 12 \\
Cassini & V (200), E (200), M (200), J (14000) & 6 & 6 & 16 \\
Rosetta & E (300), M (300) & 6 & 6 & 16 \\
Prospective Jupiter & V (200), E (200), M (200) & 4 & 4 & 12 \\
\bottomrule
\end{tabular}
\end{table}

For a planetary boundary, the evaluator obtains the body's ephemeris state at
the candidate's reported boundary epoch. Thus, the listed arrival windows
constrain the epoch rather than specify a stored Cartesian state.
The relatively permissive terminal-speed ranges in the Mariner~10 and
Voyager~2 benchmarks represent flyby rather than rendezvous conditions. Their
endpoint costs remain zero up to 10 and $30\,\mathrm{km}\,\mathrm{s}^{-1}$,
respectively, and increase with unit slope up to the configured admissible
limits of 20 and $40\,\mathrm{km}\,\mathrm{s}^{-1}$. The Rosetta endpoint is
the fixed heliocentric state
\begin{align*}
    \mathbf r_{67\mathrm{P}}
    &= (185144048.790,\;-438016487.910,\;-252361615.057)\,\mathrm{km},\\
    \mathbf v_{67\mathrm{P}}
    &= (7.655449311,\;12.114354862,\;5.612458233)\,
       \mathrm{km}\,\mathrm{s}^{-1},
\end{align*}
at MJD 56875.3791667. This state was obtained from the NASA/JPL Horizons
on-line ephemeris system~\citep{giorgini_horizons_2026}, using 67P record
90000701 (target solution JPL\#K154/5), the Sun as the centre
(``500@10''; DE441), and a geometric heliocentric ICRF vector at 6 August
2014, 09:06 TDB.

\subsubsection{Construction of historical-sequence references}
\label{app:astro_references}

The historical comparison does not use a reported mission $\Delta v$ budget.
Reported mission budgets also reflect operational decisions and navigation
margins that do not match the benchmark model. Instead, the cited mission
chronologies supply the historical encounter orders and nominal event epochs.
We reconstruct each reference through fine local optimization around those
epochs, then save the resulting node sequence and replay it through the same
full evaluator used for \method{}. This estimates how the historical trajectory
structure performs under the benchmark model; it is not intended as a
competing broad mission-design optimizer.

The two searches use the same mission boundaries, physical evaluator, and cost
definition, but different prior information. Reference construction may use the
historical encounter order and nominal epochs because its purpose is to map that
trajectory family into the benchmark. The evolved program receives the broad
mission windows and allowed flyby bodies, but not the mission identity,
historical order or intermediate epochs, or the optimized reference trajectory
and cost. This prevents direct transfer of the reference result. It does not
establish independence from knowledge encoded in the pretrained language
model: the numerical task specification may itself suggest a familiar encounter
sequence.

\subsubsection{How evolution changes mission-design search}
\label{app:astro_optimizer_analysis}

To examine the search procedure, we replayed the archived initial and evolved
programs with lightweight instrumentation. The wrapper records the ordered
encounter sequences passed to the program's internal trajectory evaluations,
every Lambert-solver call and its branch parameters, and elapsed execution time,
without changing the search logic. The traces are used only to characterize
program behavior; the trajectory costs reported in the main text come from
independent replay of the checkpoint solutions.

\begin{figure}[htbp]
    \centering
    \includegraphics[width=\linewidth]{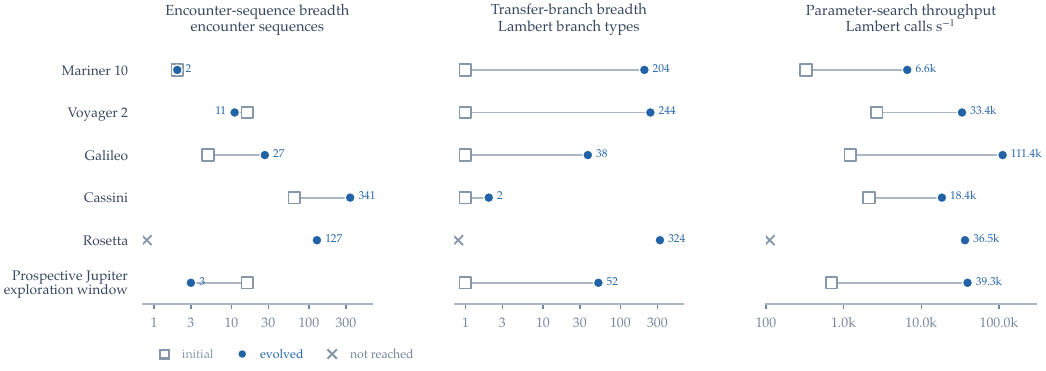}
    \caption{\textbf{Evolution reallocates finite search effort across encounter
    sequences, transfer branches, and parameter search.} Open squares denote
    the initial program and blue circles the evolved program; the cross
    indicates that the initial program returned no valid trajectory before the
    corresponding measures were reached. For the initial program,
    encounter-sequence breadth
    is the size of its permutation-based candidate set; for the evolved program,
    it is the number of distinct encounter sequences observed during replay.
    Transfer-branch breadth counts distinct Lambert call types, distinguished
    by revolution count, prograde or retrograde direction, and low- or
    high-path selection.
    Parameter-search throughput is the number of recorded Lambert-solver calls
    divided by total program execution time. Axes are logarithmic.}
    \label{fig:app.astro_search_programs}
\end{figure}

The evolved programs do not uniformly maximize search breadth
(\Cref{fig:app.astro_search_programs}). They expand the set of encounter
sequences substantially for Galileo and Cassini, retain a narrow sequence set
for Mariner~10, and search fewer encounter sequences than the initial
permutation generator for Voyager~2 and the prospective Jupiter case. Their
code stages systematic and stochastic sequence generation, prunes or retains
candidate sequences according to early costs, and redirects the remaining time
toward promising sequences. Evolution therefore selects task-dependent search
schedulers rather than simply making every search wider or deeper.

The middle panel captures a second discrete decision that remains after an
encounter sequence has been selected. Each transfer leg may admit short- and
long-path, prograde and retrograde, and multi-revolution Lambert solutions, and
the useful combinations depend on the continuous encounter epochs. In each of
the five tasks for which the initial program reaches this stage, the evolved
program explores more branch types. Cassini nevertheless exhibits the
complementary strategy: it evaluates many encounter sequences while using only
two Lambert branch types. After this screening, the programs combine global
exploration with local refinement of encounter epochs and maneuver variables.
This division of effort avoids spending the full budget on either exhaustive
sequence enumeration or repeated polishing of one transfer basin.

The right panel measures an engineering property of the complete search
implementation. All valid evolved replays execute Lambert-solver calls at a
higher observed rate than the initial program. The archived programs use
caching and candidate pruning to avoid repeated work, while the measured rate
also reflects solver choice and other program overhead. A higher throughput
allows more transfer candidates to be tested within a wall-clock budget, but is
not itself a measure of solution quality; this comparison is not an isolated
speed ablation. It instead shows the computational capacity available to the
evolved search strategy.
Together with adaptive budget allocation, broad branch coverage where useful,
and complementary global and local optimization, this capacity allows more
effective use of finite computation.

The replay describes the behavior of the final executable programs, not the
provenance of every heuristic they contain. Several programs include
mission-specific sequence templates or incumbent solutions that may have been
introduced during evolution or inferred from pretrained knowledge. The figure
therefore supports claims about the resulting search policies, but does not by
itself establish how independently each encounter-sequence heuristic was
discovered.

\FloatBarrier
\clearpage
\subsection{GPU Kernel Optimization}\label{app.gpu}
\paragraph{Task definitions and benchmark settings.} 
TriMul follows the triangular multiplicative update operator used in AlphaFold3-class protein-structure models~\citep{abramson2024alphafold3,byte2025protenix}, instantiated through the official GPUMode TriMul benchmark~\citep{gpumode_trimul_leaderboard}. 


\begin{taskdefinition}[TriMul]
\label{prob:trimul}
Given an input tensor $x\in\mathbb{R}^{B\times N\times N\times C}$, an
optional mask tensor $m\in\{0,1\}^{B\times N\times N}$, and the operator
weights $W$, implement the forward pass of the TriMul operator to produce
$y\in\mathbb{R}^{B\times N\times N\times C}$.

The reference operator consists of the following stages:
\begin{enumerate}
    \item Apply input normalization to $x$ to obtain $\hat{x}$.
    \item Compute the left, right, and output gates from $\hat{x}$.
    \item Compute the left and right projections from $\hat{x}$ and, when
    provided, apply the mask together with the corresponding gates.
    \item Perform a triangle multiplicative interaction by aggregating over
    the shared index.
    \item Apply output normalization, output gating, and a final projection
    to produce $y$.
\end{enumerate}
An implementation is valid only if its output matches the reference within
the prescribed numerical tolerance of $2\times10^{-2}$ on every evaluated
setting. Among valid implementations, the objective is to minimize the
geometric mean runtime across the prescribed benchmark settings; equivalently,
the search reward is the reciprocal of that geometric mean.
\end{taskdefinition}

The evaluated benchmark settings are summarized in \Cref{tab:app.trimul-shapes}.

\begin{table}[h]
\centering
\small
\caption{
TriMul benchmark settings from the GPUMode competition. 
Here $B$ is the batch size, $N$ is the sequence length, $C$ is the pair-feature dimension, and $H$ is the projection hidden dimension. 
The final two columns report the input distribution and whether masking is enabled.
}
\label{tab:app.trimul-shapes}
\begin{tabular}{lccccccc}
\toprule
 & S1 & S2 & S3 & S4 & S5 & S6 & S7 \\
\midrule
$B$ & 2 & 1 & 2 & 1 & 1 & 1 & 1 \\
$N$ & 256 & 1024 & 256 & 512 & 768 & 768 & 1024 \\
$C$ & 128 & 128 & 384 & 128 & 128 & 384 & 384 \\
$H$ & 128 & 128 & 128 & 128 & 128 & 128 & 128 \\
Distribution & normal & cauchy & normal & normal & cauchy & normal & normal \\
Mask & no & no & yes & no & no & yes & no \\
\bottomrule
\end{tabular}
\end{table}

Asymmetric matrix multiplication follows the KernelBench Level~1 tail skinny matrix multiplication task~\citep{ouyang2025kernelbench}, extended with additional highly unbalanced shapes that arise in modern AI workloads. The evaluated benchmark settings are summarized in \Cref{tab:app.asymmetric-matmul-shapes}.

\begin{taskdefinition}[Asymmetric Matrix Multiplication]
\label{prob:asymmetric_matmul}
Given input matrices
\[
A\in\mathbb{R}^{M\times K},
\qquad
B\in\mathbb{R}^{K\times N},
\]
compute
\[
C=AB,
\qquad
C\in\mathbb{R}^{M\times N},
\]
for highly asymmetric shapes in which the matrix dimensions are strongly
unbalanced.
The reference uses PyTorch FP32 matrix multiplication. An implementation is
valid only if its output matches the reference within a numerical tolerance
of $10^{-2}$ on every evaluated setting. Among valid implementations, the
objective is to minimize the geometric mean runtime across the prescribed
benchmark settings.
\end{taskdefinition}

\begin{table}[h]
\centering
\small
\caption{
Asymmetric matrix multiplication benchmark settings. 
Each setting specifies the dimensions $(M,K,N)$ for multiplying 
$A\in\mathbb{R}^{M\times K}$ and $B\in\mathbb{R}^{K\times N}$.
}
\label{tab:app.asymmetric-matmul-shapes}
\setlength{\tabcolsep}{8pt}
\begin{tabular}{lccccc}
\toprule
 & S1 & S2 & S3 & S4 & S5 \\
\midrule
$M$ & 32 & 32 & 32768 & 32768 & 32768 \\
$K$ & 4096 & 11008 & 16 & 32 & 64 \\
$N$ & 12288 & 4096 & 32768 & 32768 & 32768 \\
\bottomrule
\end{tabular}
\end{table}

Batched cumsum follows the KernelBench Level~1 cumsum task~\citep{ouyang2025kernelbench}, extended with additional LLM sampling-inspired shapes. The evaluated benchmark settings are summarized in \Cref{tab:app.batched-cumsum-shapes}.

\begin{taskdefinition}[Batched Cumsum]
\label{prob:batched_cumsum}
Given an input tensor $x\in\mathbb{R}^{B\times N}$, compute the inclusive
cumulative sum along the last dimension to produce
\[
y\in\mathbb{R}^{B\times N},
\qquad
y_{b,i}=\sum_{j=0}^{i}x_{b,j}.
\]
Inputs and outputs use float32. An implementation is valid only if its output
matches the reference within a numerical tolerance of $10^{-4}$ on every
evaluated setting. Among valid implementations, the objective is to minimize
the geometric mean runtime across the prescribed benchmark settings;
equivalently, the search reward is its reciprocal.
\end{taskdefinition}

\begin{table}[h]
\centering
\small
\caption{
Batched cumsum benchmark settings. Each setting specifies an input tensor shape $(B,N)$, where $B$ is the batch size and $N$ is the sequence length along which the cumulative sum is computed.
}
\label{tab:app.batched-cumsum-shapes}
\setlength{\tabcolsep}{7pt}
\begin{tabular}{lcccccc}
\toprule
 & S1 & S2 & S3 & S4 & S5 & S6 \\
\midrule
$B$ & 16 & 64 & 16 & 64 & 96 & 32768 \\
$N$ & 32000 & 32000 & 262208 & 262208 & 201088 & 32768 \\
\bottomrule
\end{tabular}
\end{table}

\paragraph{TriMul evaluation alignment.}

\Cref{tab:app.trimul_align} reports the TriMul comparison on H100 under Triton~3.4.0 and Triton~3.6.0. We report both versions because TriMul latency is noticeably affected by the Triton compiler version. Triton~3.4.0 is retained for comparability with prior AI-discovery system baselines, including TTT-Discover and K-Search. 
However, in our local evaluation, Triton~3.6.0 reproduces the public GPUMode leaderboard times more closely. We therefore use Triton~3.6.0 as the leaderboard-aligned setting for the main-text comparison, while reporting Triton~3.4.0 as an additional reference. Each local measurement is repeated three times and reported as mean $\pm$ standard deviation. Under both Triton versions, the kernel discovered by \method achieves the lowest latency among all compared AI methods and public GPUMode Triton submissions.

\begin{table}[h]
\centering
\small
\caption{TriMul H100 comparison under Triton-version alignment. Time is reported in milliseconds, and lower is better. We report the GPUMode leaderboard runtimes together with local evaluations under Triton~3.4.0 and Triton~3.6.0. The upper block lists the top-5 public GPUMode Triton submissions, excluding TTT-Discover, which is reported with the AI-discovery system baselines.}
\label{tab:app.trimul_align}
\begin{tabular}{llccc}
\toprule
\textbf{Method} & \textbf{Model} & \textbf{GPUMode} & \multicolumn{2}{c}{\textbf{Local evaluation}} \\
& & \textbf{Leaderboard} & \textbf{Triton 3.4.0} & \textbf{Triton 3.6.0} \\
\midrule
Zeyu Shen         & - & 1.140 & 1.293 $\pm$ 0.011 & 1.131 $\pm$ 0.009 \\
POLARIS AGENT     & - & 1.295 & 1.660 $\pm$ 0.013 & 1.298 $\pm$ 0.009 \\
davidberard & - & 1.371 & 1.394 $\pm$ 0.011 & 1.345 $\pm$ 0.022 \\
Waqar             & - & 2.368 & 2.349 $\pm$ 0.022 & 2.284 $\pm$ 0.014 \\
Arseni Ivanov     & - & 2.546 & 2.835 $\pm$ 0.017 & 2.435 $\pm$ 0.003 \\
\midrule
TTT-Discover      & gpt-oss-120b w/RL & 1.161 & 1.229 $\pm$ 0.005 & 1.164 $\pm$ 0.004 \\
Aster             & Multi-model agent & -- & 1.232 $\pm$ 0.007 & 1.212 $\pm$ 0.005 \\
K-Search          & GPT-5.2 + Gemini-3-pro & -- & 1.169 $\pm$ 0.012 & 1.154 $\pm$ 0.014 \\
\midrule
\method           & gpt-oss-120b & -- & \textbf{1.137 $\pm$ 0.017} & \textbf{1.122 $\pm$ 0.008} \\
\bottomrule
\end{tabular}
\end{table}

\paragraph{Cross-device performance generalization.}

\Cref{tab:app.trimul_cross_hardware} evaluates whether the TriMul kernel discovered on NVIDIA H200 transfers to other GPU architectures. The same Triton program is evaluated directly on NVIDIA A100, NVIDIA H100, and AMD MI300, without any platform-specific re-tuning or additional search. Despite being discovered only on H200, the kernel achieves lower latency than the compared AI-discovery system baselines and the top-3 GPUMode Triton submissions available for each target GPU. 

\begin{table}[h]
\centering
\small
\caption{Cross-hardware comparison of TriMul performance. Time is reported in milliseconds, and lower is better. $^{\dagger}$ denotes the GPUMode leaderboard time from TTT-Discover.}
\label{tab:app.trimul_cross_hardware}
\begin{tabular}{lcccc}
\toprule
\multirow{2}{*}{\textbf{Method}} & \multicolumn{4}{c}{\textbf{Time}} \\
\cmidrule(lr){2-5}
& \textbf{A100} & \textbf{H100} & \textbf{H200} & \textbf{MI300} \\
\midrule
1st submission   & 2.198$^{\dagger}$ & 1.140 & - & 2.657 \\
2nd submission   & 2.370 & 1.161$^{\dagger}$ & - & 5.364 \\
3rd submission   & 4.532 & 1.294 & - & 5.648 \\
\midrule
TTT-Discover   & 2.194 $\pm$ 0.004 & 1.164 $\pm$ 0.004 &  1.064 $\pm$ 0.006 & 1.382 $\pm$ 0.006 \\
Aster       & 2.151 $\pm$ 0.004 & 1.212 $\pm$ 0.005 & 1.101 $\pm$ 0.007 & 1.665 $\pm$ 0.009 \\
K-Search    & 2.169 $\pm$ 0.013 & 1.154 $\pm$ 0.014 & 1.075 $\pm$ 0.003 & 1.486 $\pm$ 0.006 \\
\midrule
\method & \textbf{2.135 $\pm$ 0.006}  & \textbf{1.122 $\pm$ 0.008} & \textbf{1.020 $\pm$ 0.001} & \textbf{1.352 $\pm$ 0.004} \\
\bottomrule
\end{tabular}
\end{table}

\subsection{Scaling Law Discovery}
\label{app.sld}
\paragraph{Overview.}
Scaling law discovery~\citep{lin2026sldagent,lin2024selecting} studies how machine learning performance changes with scale and aims to identify compact symbolic laws that extrapolate from small-scale experiments to larger regimes. In foundation-model development, such laws are used to predict quantities such as training loss, downstream error, or task-specific metrics from variables including model size, dataset size, vocabulary size, learning rate, batch size, and architectural choices. The central challenge is not merely to fit a curve, but to recover a concise analytic form whose structure is not known \emph{a priori} and that generalizes across related experimental settings. Each trial consists of a set of input variables, a target quantity, and a control index identifying the experimental context, such as a model family, dataset, or domain. The control index makes it possible to search for a shared symbolic form across settings while allowing the coefficients to vary within each setting. This makes scaling law discovery especially suitable for \method{}: candidate laws can be represented as executable programs, evaluated automatically on held-out extrapolation performance, and compared across a large, structured design space involving symbolic form, asymptotic behavior, parameter sharing, and fitting procedures. Compared with standard regression, the emphasis here is on symbolic structure and extrapolation; compared with broader agentic ML benchmarks, the objective is not to engineer an end-to-end workflow, but to discover a compact law that transfers across settings and remains accurate in unseen scaling regimes. We formalize scaling law discovery as follows.



\begin{taskdefinition}[Scaling Law Discovery]
\label{prob:scaling_law}

Let
\[
\mathcal{D}_{\mathrm{train}} = \{(x_i, j_i, y_i)\}_{i=1}^m
\]
be a collection of observed trials, where each $x_i \in \mathbb{R}^n$ is a vector of feature variables, $j_i \in \mathcal{C}$ is a control index denoting the experimental setting, and $y_i \in \mathbb{R}^k$ is the target quantity to be predicted. For each setting $j \in \mathcal{C}$, let
\[
\mathcal{D}_{\mathrm{train}}^{(j)} = \{(x_i, y_i) : j_i = j\}
\]
denote the subset of trials belonging to that setting.

The goal is to discover:
\begin{enumerate}
    \item a symbolic law $f_\theta : \mathbb{R}^n \to \mathbb{R}^k$ parameterized by coefficients $\theta$, and
    \item a fitting procedure that produces, for each control setting $j \in \mathcal{C}$, a parameter vector $\theta_j$ from $\mathcal{D}_{\mathrm{train}}^{(j)}$,
\end{enumerate}
such that the instantiated predictors $f_{\theta_j}$ extrapolate accurately to unseen inputs drawn from larger-scale or otherwise held-out regions of the input space. In benchmark form, this objective is evaluated on hidden extrapolation sets $\{\mathcal{D}_{\mathrm{test}}^{(j)}\}_{j \in \mathcal{C}}$ by maximizing the average coefficient of determination
\[
\frac{1}{|\mathcal{C}|}\sum_{j \in \mathcal{C}}
R^2\!\left(
\{y : (x,y)\in \mathcal{D}_{\mathrm{test}}^{(j)}\},
\{f_{\theta_j}(x) : (x,y)\in \mathcal{D}_{\mathrm{test}}^{(j)}\}
\right),
\]
or, equivalently, by minimizing extrapolation error on the unseen test points.
\end{taskdefinition}

A canonical example is pretraining scaling, where the inputs are model size $N$ and dataset size $D$, the target is training loss $L$, and the objective is to discover a law of the form $L \approx f_\theta(N, D)$. The same framework also covers scaling with domain mixture, learning rate, batch size, and related variables. What distinguishes this task from ordinary regression is that success is measured primarily by extrapolation beyond the fitted regime, not by interpolation within it.

\paragraph{Experimental setting.}
We report results on the four-task SLDBench subset: \texttt{parallel} (36 seen / 12 unseen)~\citep{chen2025parallel}, \texttt{domain\_mix} (80 seen / 24 unseen)~\citep{ye2024data}, \texttt{lr\&bsz} (2,702 seen / 117 unseen)~\citep{li2025predictable}, and \texttt{u\_shape} (389 seen / 127 unseen)~\citep{wu2024u}. Following the original SLDBench protocol, the unseen split is always constructed as an extrapolation test set by holding out the largest model sizes, compute regimes, or other extreme settings, rather than by using random interpolation-style splits. The execution environment also follows SLDBench: agents operate in a sandbox terminal with a minimal Python stack (\texttt{scikit-learn}, \texttt{pandas}, and \texttt{datasets}) and no network access, and must implement the discovered law and its parameter-fitting subroutine under the required function signature. Final performance is measured by test-set $R^2$, clipped to $[-1,1]$, on the hidden extrapolation split, where higher is better. We set $C=16$, $L=20$, and $K=16$ for \method{} on this task.

To ensure a controlled comparison, our method uses the same initialization, task instruction, and evaluator as SLDAgent. In particular, the initial program is the same baseline program pair used by SLDAgent: a naive power-law \texttt{scaling\_law\_func} together with a standard BFGS-based \texttt{fit\_scaling\_law} optimizer. We also adopt the same task-specific instruction template, which asks the model to evolve both the symbolic expression and the fitting routine from this baseline, while emphasizing extrapolation accuracy, cross-setting generalization, parameter efficiency, and numerical/theoretical stability; the instruction additionally provides task context, function signatures, and data characteristics such as feature definitions and value ranges. The evaluator is likewise kept identical to the SLDAgent setting. This alignment is important because SLDAgent is itself an evolution-based method built on top of the OpenEvolve framework, using iterative mutation and evaluation of candidate programs; accordingly, our comparison isolates the effect of the evolutionary strategy rather than differences in initialization, prompting, or evaluation pipeline.

\begin{table}[t]
\centering
\caption{SLDBench results. Scores are test-set \(R^2\) values, where higher is better, averaged over 5 random seeds.}
\resizebox{0.7\textwidth}{!}{%
\begin{tabular}{ccccccc}
\toprule
Agent & Model & parallel & {domain\_mix} & {lr\&bsz} & {u\_shape} & {Avg. \(R^2\)} \\
\midrule
Aider & GPT-5 & 0.991 & 0.514 & -0.659 & -0.474 & 0.093 \\
Terminus-2 & GPT-5 & \best{1.000} & 0.502 & -0.754 & -0.604 & 0.036 \\
Mini-SWE-Agent & GPT-5 & 0.997 & 0.873 & -0.269 & -0.491 & 0.277 \\
OpenCode & GPT-5 & \best{1.000} & 0.960 & -0.368 & -0.480 & 0.278 \\
OpenHands & GPT-5 & \best{1.000} & 0.899 & -0.909 & -0.278 & 0.178 \\
CodeX & GPT-5 & \second{0.999} & 0.933 & -0.039 & -0.740 & 0.288 \\
Goose & GPT-5 & \best{1.000} & 0.944 & 0.280 & \second{-0.232} & 0.498 \\
SLDAgent & GPT-5 & \best{1.000} & 0.988 & 0.604 & -0.305 & \second{0.572} \\
Human & -- & \best{1.000} & 0.671 & -0.076 & -1.000 & 0.149 \\
\midrule
Gemini-CLI & Gemini-2.5-Flash & 0.200 & 0.530 & -0.873 & -0.794 & -0.234 \\
SLDAgent & Gemini-2.5-Flash & \best{1.000} & \best{0.991} & -0.871 & -0.758 & 0.090 \\
Gemini-CLI & Gemini-3-Pro-Preview & 0.600 & 0.978 & -0.332 & -0.847 & 0.100 \\
SLDAgent & Gemini-3-Pro-Preview & \best{1.000} & 0.984 & 0.513 & -1.000 & 0.374 \\
Claude Code & Claude-Haiku-4.5 & \best{1.000} & 0.905 & -0.511 & -1.000 & 0.099 \\
SLDAgent & Claude-Haiku-4.5 & \best{1.000} & 0.980 & -0.657 & -0.754 & 0.142 \\
Claude Code & Claude-Sonnet-4.5 & 0.998 & 0.971 & -0.846 & -1.000 & 0.031 \\
SLDAgent & Claude-Sonnet-4.5 & \best{1.000} & 0.985 & -0.514 & -0.522 & 0.237 \\
CodeX & o4-mini & \best{1.000} & 0.553 & -0.773 & -1.000 & -0.055 \\
SLDAgent & o4-mini & \best{1.000} & \second{0.989} & \second{0.611} & -1.000 & 0.400 \\
\midrule
\method{} & gpt-oss-120b & \best{1.000} & \best{0.991} & \best{0.712} & \best{-0.008} & \best{0.674} \\
\bottomrule
\end{tabular}%
}
\label{tab:sldbench_combined_with_simpleevolve}
\end{table}

\paragraph{Results analysis.}
\Cref{tab:sldbench_combined_with_simpleevolve} shows that scaling law discovery remains a strong discriminator of agent capabilities even on this reduced four-task subset. \method{} achieves the best overall average score of 0.674, outperforming the strongest baseline, \textsc{SLDAgent} with GPT-5, which attains 0.572 on the same subset. The tasks also reveal a useful spectrum of difficulty. \texttt{parallel} is close to saturated, with many methods reaching near-perfect extrapolation, while \texttt{domain\_mix} is broadly tractable but still differentiates top-performing methods at the margin. By contrast, \texttt{lr\&bsz} and especially \texttt{u\_shape} remain challenging: many agents still obtain negative test $R^2$, indicating that symbolic forms that interpolate well in-range often fail to extrapolate reliably into the held-out regime. \method{}'s advantage is most pronounced on these harder extrapolative settings, where it attains the best score on \texttt{lr\&bsz} and the best score on \texttt{u\_shape}, while also tying for best on \texttt{parallel} and \texttt{domain\_mix}.

\paragraph{Case study.}
The \texttt{lr\&bsz} task is a particularly informative case study because it tests whether a discovered scaling law can support an actual decision, rather than merely fitting observed losses. The goal is to predict validation loss as a function of learning rate, batch size, dataset size, and model size, and then use the predicted surface to select a good hyperparameter configuration in an extrapolated regime. Because the held-out split consists of larger-scale configurations, success requires recovering the geometry of the loss basin beyond the observed range.

Our method discovers an explicit symbolic law for the full loss surface,
\[
\begin{aligned}
\hat{L}_{\text{\method}}(\mathrm{lr}, \mathrm{bsz}, D, N)
={}&
6.567185 \,\mathrm{lr}^{-0.0131}\,\mathrm{bsz}^{0.0096}\,D^{0.0346}\,N^{-0.1173} \\
&+ 0.408323 \,\mathrm{lr}^{0.2807}\,\mathrm{bsz}^{-0.4838}\,D^{0.0499}\,N^{0.0995} \\
&+ 26.840067 \,\mathrm{lr}^{-0.0657}\,\mathrm{bsz}^{0.0595}\,D^{-0.2274}\,N^{0.0474} \\
&+ 66071022.857 \,\mathrm{lr}^{0.0844}\,\mathrm{bsz}^{1.8723}\,D^{-1.4189}\,N^{-0.0982} \\
&+ 11.145293 \,\mathrm{lr}^{0.9783}\,\mathrm{bsz}^{0.1254}\,D^{-0.4142}\,N^{0.4786} \\
&- 0.1294720520.
\end{aligned}
\]
Given a target regime \((D,N)\), we evaluate this law on the full admissible hyperparameter grid \(\mathcal{G}\) and choose the point with the minimum predicted loss:
\[
(\mathrm{lr}^{\dagger}, \mathrm{bsz}^{\dagger})
=
\arg\min_{(\mathrm{lr},\mathrm{bsz})\in\mathcal{G}}
\hat{L}_{\text{\method}}(\mathrm{lr}, \mathrm{bsz}, D, N).
\]
This matches the actual deployment setting: we score every feasible grid point directly, rather than optimizing in a continuous space and then rounding back to the nearest grid point.

For the extrapolation example in~\Cref{fig:ai_for_ai}e, corresponding to a \(1\)B-parameter model trained on \(100\)B tokens, this procedure selects the blue-star configuration \((1.953\times 10^{-3}, 512)\). Its realized validation loss is \(2.0774\), while the true best grid point is the red-star configuration \((1.381\times 10^{-3}, 512)\) with loss \(2.0762\). Thus, the configuration selected by \method{} is only \(0.058\%\) above the optimum. For comparison, SLDAgent selects the gray-star point \((1.953\times 10^{-3}, 384)\), whose realized loss is \(2.0776\), or \(0.067\%\) above the optimum. Although the numerical gap is small, \method{} more accurately identifies the optimal region of the extrapolated loss surface and lands slightly closer to the empirical optimum.

This comparison is meaningful because the three marked points all lie in the same narrow low-loss valley, so the main challenge is fine-grained recovery of the learning-rate/batch-size trade-off rather than coarse localization. Modeling the full loss surface, rather than only the optimum coordinates, provides richer supervision and makes the discovered law directly usable as a practical selector in unseen regimes.

More broadly, this example shows why high held-out \(R^2\) matters in practice: the value of a scaling law lies not only in predictive accuracy, but also in its ability to support extrapolative decisions. On the \texttt{lr\&bsz} task, \method{} turns a symbolic loss law into a grid-level hyperparameter rule and selects a point that is essentially optimal on the true evaluation landscape.

\subsection{Whole-brain neural-activity forecast on ZAPBench}

\paragraph{Problem definition and benchmark design.}
ZAPBench evaluates neural-activity forecasting at whole-brain, single-neuron resolution. Structural maps of neural wiring are essential for understanding circuit organization, but they do not by themselves specify the time-varying circuit dynamics that transform sensory input and internal state into neural responses. A functional account therefore requires predicting how activity evolves across the brain over time. ZAPBench provides a quantitative benchmark for this problem by asking models to forecast the activities of all recorded neurons from a short temporal context, enabling direct comparison of forecasting models and analysis of how whole-brain dynamics unfold across stimulus conditions. It is built from a two-hour light-sheet fluorescence recording of a larval zebrafish that was exposed to a sequence of visual stimuli designed to elicit a range of behaviors; the raw volumetric video was aligned, motion-stabilized, and segmented into activity traces, yielding a data matrix of $71{,}721$ neurons over $7{,}879$ time steps (a sampling interval of roughly $914$\,ms per step) spanning nine stimulus conditions (\emph{gain}, \emph{dots}, \emph{flash}, \emph{taxis}, \emph{turning}, \emph{position}, \emph{open loop}, \emph{rotation}, and \emph{dark}).

\paragraph{Forecasting task and data splits.}
The task is to predict the joint future activity of all $71{,}721$ neurons up to $32$ steps ahead given only $4$ steps of past context---a high-dimensional multivariate forecasting problem. Because the traces are derived from the volumetric recording, a forecaster may operate either on the extracted per-neuron traces or on the 3D volumes directly: the benchmark's strongest hand-designed baseline is a volumetric U-Net that predicts the brain video, exploiting spatial structure that is lost when the volumes are reduced to traces, but at a much higher computational cost---the video model is reported to train in roughly $36$ hours on $16$ A100 GPUs, versus under four hours on a single 16GB GPU for trace-based forecasters. Following ZAPBench, each stimulus condition is split $70/10/20$ into contiguous train, validation, and test segments.

\begin{taskdefinition}[Whole-brain neural-activity forecast on ZAPBench]
\label{prob:zapbench}
Let $X\in\mathbb{R}^{T\times N}$ denote the neural-activity traces for
$N$ neurons over $T$ time steps, and let
$\mathcal{D}_{\mathrm{tr}}$, $\mathcal{D}_{\mathrm{val}}$, and
$\mathcal{D}_{\mathrm{test}}$ denote the sets of valid forecast origins in
the train, validation, and test splits, respectively. For a forecast horizon
$H\in\mathbb{N}$, find a forecasting program
\[
\operatorname{forecast}_H:\mathbb{R}^{4\times N}
\to \mathbb{R}^{H\times N}
\]
that maps four context steps of whole-brain activity to predictions for all
neurons over the next $H$ time steps. For a forecast task at a specific time step origin $t$, the program produces
\[
\hat{X}_{t+1:t+H,:}
=
\operatorname{forecast}_H(X_{t-3:t,:}).
\]

For any split $\mathcal{D}\in
\{\mathcal{D}_{\mathrm{tr}},\mathcal{D}_{\mathrm{val}},
\mathcal{D}_{\mathrm{test}}\}$, performance at horizon $H$ is measured by
horizon-averaged mean absolute error,
\[
\mathrm{MAE}_{\mathcal{D},H}(\operatorname{forecast}_H)
=
\frac{1}{|\mathcal{D}|\,H\,N}
\sum_{t\in\mathcal{D}}
\sum_{h=1}^{H}
\sum_{n=1}^{N}
\left|
\hat{X}_{t+h,n}-X_{t+h,n}
\right|.
\]
The candidate programs are trained within
$\mathcal{D}_{\mathrm{tr}}$ and selected by validation MAE,
\[
R_{\mathrm{search}}(\operatorname{forecast}_H)=
\begin{cases}
-\mathrm{MAE}_{\mathcal{D}_{\mathrm{val}},H}(\operatorname{forecast}_H),
& \text{if the prediction is valid},\\
-\infty, & \text{otherwise}.
\end{cases}
\]
Final performance is reported as
$\mathrm{MAE}_{\mathcal{D}_{\mathrm{test}},H}$ for each evaluated forecast
horizon $H$.
\end{taskdefinition}

\paragraph{Evaluation and search protocol.}
Candidate programs are scored by the mean absolute error (MAE) between predicted and recorded activity, averaged over all neurons and over the full $32$-step (or $1$, $4$, $8$, $16$ steps in different tasks) horizon. The validation split serves two purposes: it drives model selection including hyperparameter choices and early stopping; and it supplies the single scalar (validation MAE) that guides the search. The test split is never seen during search or model selection. \method is initialized from one neural starter program, a shared-weight per-neuron forecaster (Fig.~\ref{fig:zapbench}f), and edits the body of its \texttt{forecast} routine. Each candidate is trained and scored in an isolated evaluator.

\subsection{Single-Cell RNA-Seq Denoising}
\label{app:denoising}

Single-cell RNA sequencing (scRNA-seq) resolves gene expression in individual cells~\citep{macosko2015highly,zheng2017massively,luecken2019current}, but its measurements are corrupted by low capture efficiency and stochastic dropout~\citep{vallejos2017normalizing}, motivating dedicated denoising algorithms~\citep{van2018recovering}. Denoising has no ground truth to check against; quality is assessed by the molecular cross-validation framework of \citet{batson2019molecular}, which partitions observed UMI counts by binomial subsampling and underpins the OpenProblems benchmark~\citep{luecken2025defining}. The core difficulty is a tension between two objectives: mean-squared error in log-normalized space rewards faithful recovery of relative expression, while the Poisson negative log-likelihood rewards count consistency after library-size rescaling, and aggressively optimizing one degrades the other.

\begin{taskdefinition}[Single-Cell RNA-Seq Denoising]
\label{prob:scrna_denoising}
Given $X_{\mathrm{tr}}\in\mathbb{Z}_{\geq0}^{C\times G}$, the training split
of the Pancreas dataset, find a function
\[
\operatorname{denoise}:\mathbb{Z}_{\geq0}^{C\times G}
\to\mathbb{R}_{\geq0}^{C\times G}
\]
that produces $\hat{X}=\operatorname{denoise}(X_{\mathrm{tr}})$ with
non-negative finite entries,
$\hat{X}_{\max}\leq\lVert X_{\mathrm{tr}}\rVert_1$, and
$\mathrm{Pois\_norm}(\hat{X})\geq0.97$, all within a 400-second time limit.
During evolution, the selection reward is
\[
R_{\mathrm{search}}(\hat{X})=
\begin{cases}
\mathrm{MSE\_norm}(\hat{X}), & \text{if $\hat{X}$ is valid},\\
0, & \text{otherwise}.
\end{cases}
\]
For final cross-dataset reporting, we follow the comparison-benchmark
convention and report the mean of the normalized MSE and normalized Poisson
metric scores. Generalization is assessed by applying the same discovered
program to the held-out PBMC and Tabula Muris Senis Lung datasets, without
further evolution or manual dataset-specific tuning.
\end{taskdefinition}

We follow the molecular cross-validation protocol~\citep{batson2019molecular}, splitting the Pancreas dataset into a training matrix the algorithm observes and a held-out test matrix it does not. Candidates are scored on two normalized metrics, log-space MSE and Poisson negative log-likelihood, with the Poisson term acting as a hard constraint that rejects any solution failing it, and the final score is the mean of the two. Following \citet{yuksekgonul2026learning}, we initialize \method with MAGIC~\citep{van2018recovering}, cap each candidate at 400 seconds, and run evolution exclusively on Pancreas; PBMC and Tabula Muris Senis Lung are withheld entirely and used only for final evaluation, so the reported numbers measure cross-dataset generalization to held-out datasets.

Across the held-out datasets in Table~\ref{tab:denoising_results}, \method reaches a Tabula score of $0.74$, surpassing the prior state-of-the-art method TTT-Discover~\citep{yuksekgonul2026learning} ($0.73$) and matching it on PBMC ($0.71$), while outperforming MAGIC~\citep{van2018recovering}, ALRA~\citep{linderman2022zero}, and every other baseline on both. Because evolution never saw PBMC or Tabula, these gains are direct evidence that the discovered algorithm generalizes rather than fitting the search distribution.

\begin{table}[h]
\centering
\small
\begin{tabular}{lcccccc}
\toprule
& \multicolumn{3}{c}{PBMC} & \multicolumn{3}{c}{Tabula} \\
\cmidrule(lr){2-4} \cmidrule(lr){5-7}
Method & Score\,$\uparrow$ & MSE\,$\downarrow$ & Poisson\,$\downarrow$
       & Score\,$\uparrow$ & MSE\,$\downarrow$ & Poisson\,$\downarrow$ \\
\midrule
MAGIC            & 0.42 & 0.19 & 0.16 & 0.40 & 0.18 & 0.12 \\
MAGIC (A)        & 0.42 & 0.19 & 0.16 & 0.40 & 0.18 & 0.12 \\
MAGIC (R)        & 0.64 & 0.19 & 0.05 & 0.64 & 0.18 & 0.03 \\
MAGIC (A, R)     & 0.64 & 0.19 & 0.05 & 0.64 & 0.18 & 0.03 \\
ALRA (S, RN)     & 0.50 & 0.26 & 0.05 & 0.47 & 0.27 & 0.03 \\
Best-of-25600    & 0.62 & 0.20 & 0.05 & 0.65 & 0.18 & 0.03 \\
OpenEvolve       & 0.70 & 0.16 & 0.05 & 0.71 & 0.15 & 0.03 \\
TTT-Discover     & 0.71 & 0.15 & 0.05 & 0.73 & 0.14 & 0.03 \\
\midrule
\method (ours)   & \textbf{0.71} & \textbf{0.15} & \textbf{0.05}
                 & \textbf{0.74} & \textbf{0.13} & \textbf{0.03} \\
\bottomrule
\end{tabular}
\caption{Denoising results on held-out PBMC and Tabula Muris Senis Lung
datasets. Score is the mean of normalized MSE and Poisson scores (higher is better).
MAGIC (A) = MAGIC approximate; MAGIC (R) = MAGIC with reversed normalization;
MAGIC (A, R) = MAGIC approximate with reversed normalization.
ALRA (S, RN) = ALRA with square-root norm and reversed normalization.
All non-\method results are taken from \citet{yuksekgonul2026learning}.}
\label{tab:denoising_results}
\end{table}

Inspection shows that \method departs structurally from the baselines. Where MAGIC builds a single diffusion operator and TTT-Discover elaborates that template with variance-stabilizing transform ensembling and low-rank refinement, the discovered program constructs several independent denoisers, including multi-scale diffusion under both correlation- and Euclidean-based graph operators, log-space diffusion, neighbor averaging, PCA imputation, and NMF reconstruction, and scores each internally on the training split alone. It then blends them into a data-driven ensemble, weighting candidates by their inverse Poisson and MSE losses under the Poisson constraint, and applies a gene-mean calibration that clips per-gene scaling near $1.0$ to preserve count fidelity. None of these choices are dataset-specific: the ensemble weights are fit from the training split at runtime, which is what lets the same program generalize to the unseen PBMC and Tabula distributions.

\subsection{Lasso Regularization Path}
\label{app:lasso}

\paragraph{Overview.}
The lasso regularization path is a core computational primitive in high-dimensional
statistics, arising naturally in cross-validation and model selection across domains
from genomics to finance. Computing the full path of solutions across a grid of
regularization values is orders of magnitude more efficient than solving each problem
independently, but demands careful algorithmic design to exploit warm starts and
sparsity structure. The de facto standard solver, \texttt{glmnet}~\citep{friedman2010regularization},
represents decades of expert engineering. This task asks whether \method{} can discover
a solver that is faster than \texttt{glmnet} while maintaining the same float64
precision and correctness guarantees.

\begin{taskdefinition}[Lasso Regularization Path]
\label{prob:lasso_path}
Given a feature matrix $X\in\mathbb{R}^{n\times p}$, response
$y\in\mathbb{R}^n$, and a decreasing sequence
$\lambda_1>\cdots>\lambda_K$, define
\[
F_k(w):=\frac{1}{2n}\lVert y-Xw\rVert_2^2
+\lambda_k\lVert w\rVert_1,
\qquad
w_k^\star\in\argmin_{w\in\mathbb{R}^p}F_k(w).
\]
A candidate solver returns approximate coefficients
$\widetilde{W}=(\widetilde{w}_1,\ldots,\widetilde{w}_K)$. It passes the
benchmark's objective-value check if
\[
F_k(\widetilde{w}_k)
\leq F_k(w_{k,\mathrm{sklearn}})+10^{-6}
\qquad\text{for every }k,
\]
where correctness is checked on fresh instances distinct from the timing
instances. If any required check fails, the search score is zero. Otherwise,
letting $\mathcal{I}$ denote the timing instances and $t_i$ the time required
to compute the complete regularization path on instance $i$, the search score
is
\[
R_{\mathrm{search}}
=\left(\prod_{i\in\mathcal{I}}t_i\right)^{-1/|\mathcal{I}|}.
\]
\end{taskdefinition}

\paragraph{Experimental setting.}
Each candidate is a self-contained C++ program that reads a binary-encoded problem from standard input and writes the complete coefficient path to standard output. Correctness is evaluated on a separate fresh problem that is not used for timing, preventing caching or test-set overfitting. The surrogate metric optimized during evolution is the geometric mean of solve times over 17 synthetic problem sizes. These instances span the axes that govern lasso-path performance: sample-to-feature ratio \(n/p\), design-matrix sparsity, active-set density, and feature correlation from near-independent to highly correlated Toeplitz structure.

This evaluation surface is deliberately heterogeneous: a solver specialized for wide problems with \(p\gg n\) can fail on tall problems with \(n\gg p\), and vice versa. The 11 real-world datasets in \Cref{tab:lasso_results} are not used during evolution and therefore test whether the discovered solver generalizes beyond the synthetic search distribution. We initialize \method{} from a faithful C++ port of \texttt{glmnet}'s Gaussian lasso path, including the covariance update for \(p<500\), the naive residual update for \(p\geq500\), warm starts, sequential strong-rule screening~\citep{tibshirani2012strong}, active-set inner loops, and KKT verification.

\begin{table}[htbp]
\centering
\caption{Lasso path solver performance on real-world datasets. All solutions
pass the correctness check (per-\(\lambda\) objective gap \(\le 10^{-6}\) vs.\
\texttt{sklearn}). Times are mean wall-clock times (ms). Each per-dataset
speedup is the corresponding baseline time divided by the \method{} time; the
average speedups are arithmetic means of the per-dataset speedups.}
\label{tab:lasso_results}
\resizebox{\textwidth}{!}{%
\begin{tabular}{lrrrrr}
\toprule
\textbf{Dataset} & \textbf{glmnet (ms)} & \textbf{sklearn (ms)} & \textbf{\method (ms)} & \textbf{vs.\ glmnet} & \textbf{vs.\ sklearn} \\
\midrule
\multicolumn{6}{l}{\textit{Non-biological (libsvm)}} \\
Gisette ($5100 \times 5000$)              & 4282.7  & 4873.1   & 3141.9  & 1.36$\times$ & 1.55$\times$ \\
RCV1 ($17205 \times 19959$)              & 34521.0 & 107790.2 & 19625.6 & 1.76$\times$ & 5.49$\times$ \\
\midrule
\multicolumn{6}{l}{\textit{Biological (libsvm + TCGA/Kaggle)}} \\
DNA ($1700 \times 180$)                  & 152.3   & 40.5     & 15.9   & 9.56$\times$ & 2.54$\times$ \\
Leukemia ($38 \times 7129$)              & 20.4    & 107.8    & 15.5   & 1.32$\times$ & 6.95$\times$ \\
Colon Cancer ($62 \times 2000$)          & 15.2    & 115.5    & 11.6   & 1.31$\times$ & 9.92$\times$ \\
Duke Breast Cancer ($44 \times 7129$)    & 24.9    & 183.7    & 18.1   & 1.38$\times$ & 10.16$\times$ \\
TCGA BRCA ($500 \times 20238$)           & 4020.1  & 152175.5 & 2997.2 & 1.34$\times$ & 50.77$\times$ \\
TCGA Liver RNA ($422 \times 20168$)      & 546.4   & 3309.7   & 374.5  & 1.46$\times$ & 8.84$\times$ \\
TCGA Lung RNA ($500 \times 20258$)       & 652.2   & 5003.6   & 443.7  & 1.47$\times$ & 11.28$\times$ \\
TCGA Prostate RNA ($500 \times 20232$)   & 649.9   & 13132.4  & 438.9  & 1.48$\times$ & 29.92$\times$ \\
TCGA Thyroid RNA ($500 \times 20164$)    & 648.1   & 7701.4  & 442.4  & 1.46$\times$ & 17.41$\times$ \\
\midrule
\textbf{Average} & \textbf{4139.4} & \textbf{26766.7} & \textbf{2502.3} & \textbf{2.17$\times$} & \textbf{14.08$\times$} \\
\bottomrule
\end{tabular}%
}
\end{table}

\paragraph{Results analysis.}
All reported solutions pass the held-out correctness check. Across the 11 real-world datasets, \method{} achieves an average speedup of \(2.17\times\) over \texttt{glmnet} and \(14.08\times\) over \texttt{sklearn}. The gains are largest in tall regimes: on DNA (\(1700\times180\)), \method{} is \(9.56\times\) faster than \texttt{glmnet}. It also improves on the large sparse RCV1 design by \(1.76\times\) over \texttt{glmnet} and \(5.49\times\) over \texttt{sklearn}; the gains against \texttt{sklearn} are particularly large on high-dimensional biological datasets, including \(50.77\times\) on TCGA BRCA and \(10.16\times\) on Duke Breast Cancer.

\paragraph{Case study.}
The initialization uses coordinate descent throughout: the covariance method for \(p<500\) and the naive residual method for \(p\geq500\). The discovered program instead uses a geometry-aware switching rule. For moderate-dimensional, tall designs (\(p\leq2000\) and \(n\geq p/4\)), it replaces coordinate descent with an exact LARS homotopy solver that traces the regularization path analytically and updates the active set at each kink with rank-one inverse-Gram updates~\citep{efron2004least}. For wide or sparse designs, it retains coordinate descent with strong-rule
screening, an active-set inner loop, and KKT verification.

\subsection{Mathematical Scientific Discovery Tasks}
\label{app:math_tasks}

Extremal analysis and combinatorial construction ask for explicit mathematical solutions that optimize precisely defined quantities: a distribution, a witness function, an integer set, a geometric packing, or a sign matrix. These tasks are a natural setting for scientific discovery because progress depends on finding the right structural solution, and each candidate can be checked by a deterministic evaluator rather than by a noisy empirical proxy.

\begin{figure}[p]
    \centering
    \includegraphics[width=\textwidth]{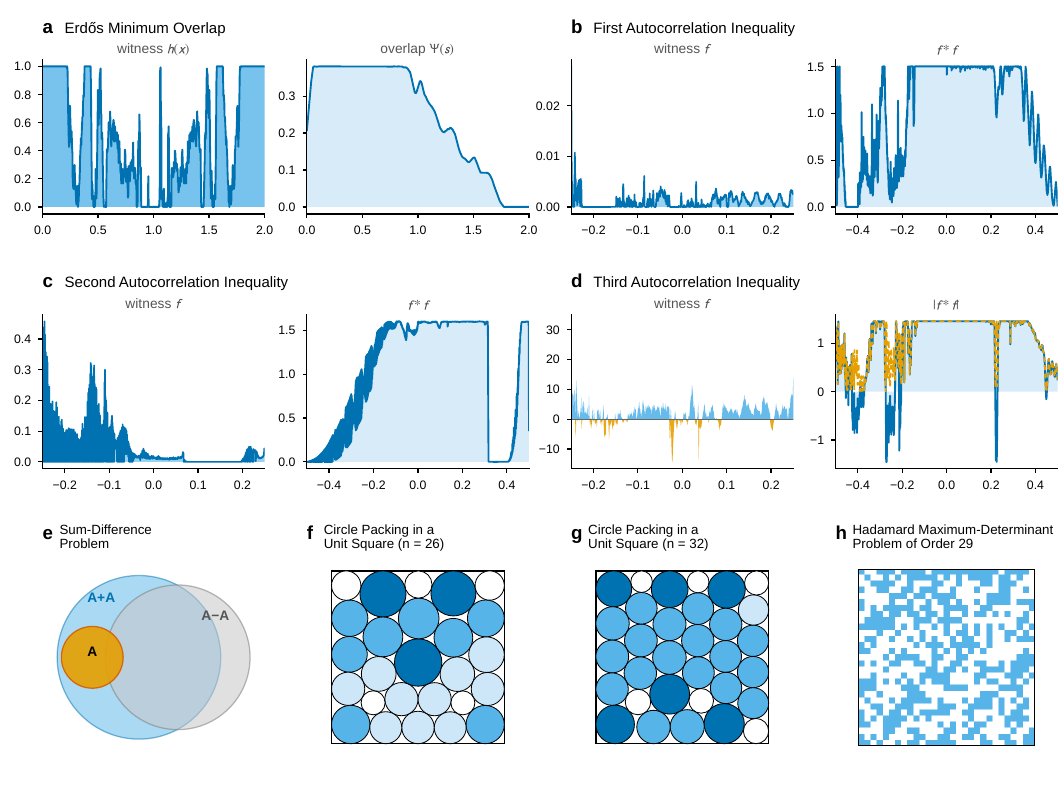}
    \caption{\textbf{Mathematical solutions discovered by \method.}
    \textbf{(a)} Erd\H{o}s Minimum-Overlap Problem: witness \(h(x)\) and its
    translated-overlap profile \(F_h(s)\).
    \textbf{(b)} First Autocorrelation Inequality: non-negative witness \(f\)
    and its autoconvolution \(f\!*\!f\).
    \textbf{(c)} Second Autocorrelation Inequality: witness \(f\) and its
    autoconvolution \(f\!*\!f\), whose near-flat plateau drives the norm ratio.
    \textbf{(d)} Third Autocorrelation Inequality: signed witness \(f\) and the absolute autoconvolution envelope \(\lvert f\!*\!f\rvert\).
    \textbf{(e)} Sum--Difference Problem: set \(A\), drawn as the
    area-proportional Euler diagram of \(A\), \(A+A\), and \(A-A\)
    \((\lvert A\rvert=506, \lvert A+A\rvert=3551, \lvert A-A\rvert=2775)\).
    \textbf{(f)} Hadamard Maximum-Determinant Problem of Order \(29\): an
    order-29 \(\{-1,1\}\) sign matrix.
    \textbf{(g)} Circle Packing in a Unit Square (\(n=26\)).
    \textbf{(h)} Circle Packing in a Unit Square (\(n=32\)).}
    \label{fig:math_discovered_solutions}
\end{figure}
\FloatBarrier

\subsubsection{Extremal analysis}
\label{app:math_extremal_analysis}

Extremal analysis tasks search over constrained functions and optimize overlap or autoconvolution functionals. Candidate solutions are represented as discretized step functions or signed piecewise-constant witnesses, with feasibility enforced by support, range, and unit-mass constraints.

\paragraph{Erd\H{o}s minimum-overlap problem.}
Erd\H{o}s' minimum-overlap problem asks how to distribute a unit amount of mass over $[0,2]$ so that the distribution has as little one-sided overlap as possible with the complements of its shifted copies.

\begin{taskdefinition}[Erd\H{o}s Minimum-Overlap Problem]
\label{prob:erdos_min_overlap}
In the step-function formulation, the task searches for $h:[0,2]\to[0,1]$ satisfying
\[
    \int_0^2 h(x)\,dx = 1,
\]
where $h$ is extended by zero outside $[0,2]$. The objective is to minimize
\[
    \Psi(h) := \sup_{s\in[0,2]} \int_0^2 h(x)\bigl(1-h(x+s)\bigr)\,dx.
\]
\end{taskdefinition}
The evaluator reports the translated-overlap objective through $1/\Psi(h)$ during search, while the standard value $\Psi(h)$ is reported for comparison.

\textbf{Best result and discovered solution.}
The discovered witness is near-binary: most grid cells are driven close to either zero or one, rather than remaining near the uniform initialization. This structure arranges the unit mass so that the translated-overlap profile is flattened across the high-risk shift range, reducing the worst shift rather than only improving the average shift behavior. The solution is therefore interpretable as an equioscillation-style witness: the best solution balances multiple competing translated-overlap constraints instead of optimizing a single shift locally (\Cref{fig:math_discovered_solutions}a).

\textbf{Case Analysis.}
The strongest program uses a coarse-to-fine optimization pipeline. It first searches over coarse discretizations to identify promising mass layouts, then refines them with constrained local optimization and projected polishing steps that preserve the exact discrete mass constraint and the box constraint $0\leq h\leq1$. This code structure is important because naive local perturbations around the uniform witness tend to improve only a small subset of shifts, whereas coarse-to-fine refinement can reorganize the global mass pattern before polishing the final overlap profile.

\taskboxreserve{52}
\paragraph{Autocorrelation Inequalities.}
The First, Second, and Third Autocorrelation Inequalities form a family of extremal autoconvolution problems.

\taskboxreserve{48}
\begin{taskdefinition}[Autocorrelation Inequalities]
\label{prob:autocorrelation_inequalities}
For an integrable function $f:\mathbb{R}\to\mathbb{R}$, define the autoconvolution
\[
    (f*f)(t) := \int_{\mathbb{R}} f(t-x)f(x)\,dx,
    \qquad t\in[-\tfrac12,\tfrac12].
\]

\medskip
\noindent\textbf{First Autocorrelation Inequality.}
Find a non-negative integrable function $f:\mathbb{R}\to\mathbb{R}$ supported on $[-\tfrac14,\tfrac14]$ such that
\[
    \int_{-1/4}^{1/4} f(x)\,dx = 1.
\]
The objective is to minimize
\[
    \Phi_1(f) := \max_{t\in[-1/2,1/2]} (f*f)(t).
\]

\medskip
\noindent\textbf{Second Autocorrelation Inequality.}
Find a non-negative integrable function $f:\mathbb{R}\to\mathbb{R}$ supported on $[-\tfrac14,\tfrac14]$ such that
\[
    \int_{-1/4}^{1/4} f(x)\,dx = 1.
\]
The objective is to maximize
\[
    \Phi_2(f) := \frac{\|f*f\|_2^2}{\|f*f\|_1\|f*f\|_\infty}.
\]

\medskip
\noindent\textbf{Third Autocorrelation Inequality.}
Find an integrable (possibly signed) function $f:\mathbb{R}\to\mathbb{R}$ supported on $[-\tfrac14,\tfrac14]$ such that
\[
    \int_{-1/4}^{1/4} f(x)\,dx = 1.
\]
The objective is to minimize
\[
    \Phi_3(f) := \max_{t\in[-1/2,1/2]} |(f*f)(t)|.
\]
\end{taskdefinition}
During search, the evaluators cast the tasks as maximization by using $1/\Phi_1$, $\Phi_2$, and $1/\Phi_3$, while the standard objectives are reported in the result tables.

\begin{table}[htbp]
\centering
\caption{Autocorrelation inequality results. An em dash (--) indicates that the corresponding source did not report a directly aligned value. Entries marked with $^\dagger$ use ThetaEvolve's corrected Third Autocorrelation Inequality verifier and are informative but not perfectly head-to-head with the original AlphaEvolve result.}
\label{tab:math_autocorrelation}
\footnotesize
\setlength{\tabcolsep}{1.5pt}
\renewcommand{\arraystretch}{1.08}\begin{tabular*}{\linewidth}{@{\extracolsep{\fill}}
>{\raggedright\arraybackslash}m{0.29\linewidth}
>{\raggedright\arraybackslash}m{0.19\linewidth}
*{3}{>{\centering\arraybackslash}m{0.145\linewidth}}@{}}
\toprule
\textbf{Method} & \textbf{Model} &
\textbf{First Autocorrelation Inequality (\(\downarrow\))} &
\textbf{Second Autocorrelation Inequality (\(\uparrow\))} &
\textbf{Third Autocorrelation Inequality (\(\downarrow\))} \\
\midrule
Best human~\citep{matolcsi2010improved,boyer2026autoconvolution,vinuesa2010generalized} & -- & 1.509730 & 0.901500 & 1.458100 \\
AlphaEvolve~\citep{novikov2025alphaevolve} & Gemini-2.0 Pro + Flash & 1.505300 & 0.896200 & 1.455700 \\
AlphaEvolve V2~\citep{georgiev2025mathexplore} & Gemini-2.0 Pro + Flash & 1.503170 & 0.961000 & -- \\
ThetaEvolve~\citep{wang2025thetaevolve} & Distill-Qwen3-8B & 1.503133 & 0.946900 & 1.493000$^\dagger$ \\
OpenEvolve~\citep{openevolve_github} & gpt-oss-120b & 1.507190 & 0.944900 & -- \\
OpenEvolve~\citep{openevolve_github} & Gemini-3.0-Pro & -- & -- & 1.460000 \\
TTT-Discover~\citep{yuksekgonul2026learning} & gpt-oss-120b & \underline{1.502870} & 0.959100 & -- \\
Together AI~\citep{togetherai2026einsteinarena} & Mixed & \textbf{1.502862} & \underline{0.961206} & \underline{1.454555} \\
ShinkaEvolve~\citep{lange2025shinkaevolve} & Gemini-3.0-Pro & -- & -- & 1.457800 \\
EvoX~\citep{liu2026evox} & Gemini-3.0-Pro & -- & -- & 1.455800 \\
AlphaResearch~\citep{yu2025alpharesearch} & o4-mini & -- & -- & 1.546000 \\
\midrule
\method & gpt-oss-20b & 1.506791 & 0.950494 & 1.455324 \\
\method & gpt-oss-120b & 1.503871 & \textbf{0.962694} & \textbf{1.453675} \\
\bottomrule
\end{tabular*}
\end{table}

\textbf{Best results and discovered solutions.}
The three autocorrelation tasks lead to different witness structures. For AC1, the strongest discovered witness concentrates mass near the support boundary, reducing the peak of the non-negative autoconvolution but still trailing the leading public result. For AC2, the best witness is sparse and produces a long near-flat plateau in the autoconvolution profile, a structure that improves the norm ratio by avoiding a narrow peak. For AC3, the signed witness becomes oscillatory and uses cancellation to control the largest absolute autoconvolution value. These solutions show that the same evaluation-driven loop can search across non-negative and signed witness classes without a task-specific symbolic derivation (\Cref{fig:math_discovered_solutions}b--d).

\textbf{Case Analysis.}
The best programs implement different search strategies for the three objectives. AC1 uses a simplex-projected mass-transfer search over a 1024-bin witness, repeatedly moving mass while preserving non-negativity and unit mass. AC2 combines FFT-based convolution evaluation with L-BFGS-B refinement, making it efficient to evaluate many candidate witnesses and polish sparse structures. AC3 uses a discrete cosine transform parameterization that directly searches signed oscillatory patterns and optimizes the cancellation structure of $f\ast f$. The code-level adaptation is therefore objective-specific even though the outer \method procedure is unchanged.

\subsubsection{Combinatorial construction}
\label{app:math_combinatorial_construction}

Combinatorial construction tasks search over finite structured solutions: integer sets, unequal-circle packings, and $\{-1,1\}$ sign matrices. Their evaluators are exact or feasibility-based, using integer arithmetic for sumsets, geometric non-overlap checks for packings, and exact determinant computation for sign matrices.

\paragraph{Sum--Difference Problem.}
The Sum--Difference Problem comes from classical additive combinatorics and is closely related to more-sums-than-differences sets and Ruzsa-type inequalities~\citep{martin2007many,hegarty2007explicit,ruzsa1978cardinality,ruzsa1996sums}.

\begin{taskdefinition}[Sum--Difference Problem]
\label{prob:sums_diffs}
The Sum--Difference Problem asks for a finite set $A\subset\mathbb{Z}$ whose normalized sumset is large relative to its normalized difference set. The objective is
\[
    \Gamma(A) :=
    \frac{\log\!\left(|A+A|/|A|\right)}{\log\!\left(|A-A|/|A|\right)},
\]
where
\[
    A+A := \{a+a':a,a'\in A\},
    \qquad
    A-A := \{a-a':a,a'\in A\}.
\]
\end{taskdefinition}
Candidate sets contain at most 512 integers with elements bounded in $[-10^6,10^6]$. The evaluator computes $|A+A|$ and $|A-A|$ exactly using integer arithmetic.

\textbf{Best results and discovered solutions.}
The discovered set has a highly regular arithmetic backbone rather than an irregular collection of integers. Most consecutive gaps follow a long progression pattern, while a small number of fringe positions receive local corrections. These sparse perturbations enlarge the sumset more efficiently than the difference set, producing a construction that is both high-scoring and structurally interpretable (\Cref{fig:math_discovered_solutions}e).

\textbf{Case Analysis.}
The strongest program maintains exact sum and difference multiplicity tables during search. It alternates aggressive pruning, greedy additions, and local replacements, allowing each candidate edit to be scored by its exact effect on the two set expansions. This is more efficient than repeatedly rebuilding $A+A$ and $A-A$ from scratch, and it makes the search sensitive to local fringe edits that would be difficult to identify from the final scalar score alone.

\paragraph{Circle Packing in a Unit Square.}
Unequal-circle packing in a square is a classical problem in continuous optimization and discrete geometry~\citep{peikert1992review,szabo2007new,hifi2009review}.

\begin{taskdefinition}[Circle Packing in a Unit Square]
\label{prob:circle_packing}
For $n\in\{26,32\}$, the circle-packing task asks for centers $(x_i,y_i)\in[0,1]^2$ and radii $r_i\geq0$ such that every circle lies inside the unit square and no two circles overlap:
\[
    r_i \leq x_i \leq 1-r_i,
    \qquad
    r_i \leq y_i \leq 1-r_i,
\]
\[
    (x_i-x_j)^2+(y_i-y_j)^2 \geq (r_i+r_j)^2,
    \qquad 1\leq i<j\leq n.
\]
The objective is to maximize $\sum_i r_i$.
\end{taskdefinition}
Invalid configurations receive zero score.


\begin{table}[htbp]
\centering
\caption{Circle Packing in a Unit Square results. The OpenEvolve result is derived from the comparative result reported in CodeEvolve.}
\label{tab:math_circle_packing}
\footnotesize
\setlength{\tabcolsep}{1.5pt}
\renewcommand{\arraystretch}{1.08}
\begin{tabular*}{\linewidth}{@{\extracolsep{\fill}}
>{\raggedright\arraybackslash}p{0.30\linewidth}
>{\raggedright\arraybackslash}p{0.24\linewidth}
>{\centering\arraybackslash}p{0.18\linewidth}
>{\centering\arraybackslash}p{0.18\linewidth}@{}}
\toprule

\multirow{2}{*}{\textbf{Method}} & \multirow{2}{*}{\textbf{Model}} &
\multicolumn{2}{c}{\textbf{Circle Packing in a Unit Square (\(\uparrow\))}} \\
\cmidrule(lr){3-4} 
& & \textbf{\(n=26\)} & \textbf{\(n=32\)} \\
\midrule
AlphaEvolve~\citep{novikov2025alphaevolve} & Gemini-2.0 Pro + Flash & 2.635862 & 2.937944 \\
AlphaEvolve V2~\citep{georgiev2025mathexplore} & Gemini-2.0 Pro + Flash & \textbf{2.635983} & \textbf{2.939572} \\
ShinkaEvolve~\citep{lange2025shinkaevolve} & Mixed & \underline{2.635982} & -- \\
ThetaEvolve~\citep{wang2025thetaevolve} & Distill-Qwen3-8B & \textbf{2.635983} & -- \\
TTT-Discover~\citep{yuksekgonul2026learning} & Qwen3-8B & \textbf{2.635983} & \textbf{2.939572} \\
CodeEvolve~\citep{assumpcao2025codeevolve} & Qwen3-Coder-30B & 2.635980 & \underline{2.939560} \\
OpenEvolve~\citep{openevolve_github} & Qwen3-Coder-30B & -- & 2.931560 \\
\midrule
\method & gpt-oss-20b & \textbf{2.635983} & \textbf{2.939572} \\
\method & gpt-oss-120b & \textbf{2.635983} & \textbf{2.939572} \\
\bottomrule
\end{tabular*}
\end{table}

\textbf{Best results and discovered solutions.}
The $n=26$ and $n=32$ solutions recover the strongest known packing values under the same evaluator, but the two instances have visibly different structure. The $n=26$ construction contains a dominant central circle, large boundary anchors, and a near-triangular ring of medium circles. The $n=32$ construction is more homogeneous and resembles a fixed six-row quasi-hexagonal layout. These structures show that the search does not only adjust radii locally; it recovers global geometric organization suitable for each instance size (\Cref{fig:math_discovered_solutions}g,h).

\textbf{Case Analysis.}
The evolved $n=26$ solver uses a coarse-to-fine strategy: broad center-set exploration first identifies promising layouts, and exact radius optimization then solves the radius subproblem under boundary and pairwise-distance constraints. The $n=32$ solver simplifies to a more direct routine by reusing an incumbent or fixed quasi-hexagonal layout, then solving for radii through linear programming. The two programs therefore adapt the search granularity to the instance: broader geometric exploration for $n=26$ and more efficient radius refinement for $n=32$.

\taskboxreserve{22}
\paragraph{Hadamard maximum-determinant problem of order $29$.}
The Hadamard maximum-determinant problem of order $29$ is a classical benchmark in extremal matrix theory and $D$-optimal design.

\begin{taskdefinition}[Hadamard Maximum-Determinant Problem of Order $29$]
\label{prob:hadamard_29}
The Hadamard maximum-determinant problem of order $29$ searches for a sign matrix
\[
    A\in\{-1,1\}^{29\times29}
\]
that maximizes
\[
    \Lambda(A):=|\det A|.
\]
\end{taskdefinition}
Order 29 lies outside the Hadamard regime, so the best known solutions are near-extremal sign matrices rather than exact Hadamard matrices. The evaluator computes the determinant exactly using the Bareiss algorithm and reports both the raw determinant and a normalized score.

\begin{table}[htbp]
\centering
\footnotesize
\caption{Hadamard Maximum-Determinant Problem of Order \(29\).}
\label{tab:math_hm29}
\begin{tabular}{>{\raggedright\arraybackslash}p{0.20\linewidth} >{\raggedright\arraybackslash}p{0.17\linewidth} c c}
\toprule
Method & Model & Normalized score $\uparrow$ & Determinant $\uparrow$ \\
\midrule
Best human~\citep{orrick2003newlower} & -- & \textbf{0.935673} & \textbf{$320\cdot7^{12}\cdot2^{28}$} \\
ThetaEvolve~\citep{wang2025thetaevolve} & ProRL-1.5B-v2 & 0.563500 & -- \\
ThetaEvolve~\citep{wang2025thetaevolve} & Distill-Qwen3-8B & \underline{0.576400} & -- \\
\midrule
\method & gpt-oss-20b & \textbf{0.935673} & \textbf{$320\cdot7^{12}\cdot2^{28}$} \\
\method & gpt-oss-120b & \textbf{0.935673} & \textbf{$320\cdot7^{12}\cdot2^{28}$} \\
\bottomrule
\end{tabular}
\end{table}


\textbf{Best results and discovered solutions.}
The discovered sign matrices recover the long-standing classical lower-bound record. Their structure is best inspected through the Gram matrix rather than raw signs, since row and column permutations and sign flips can change the visible matrix while preserving equivalence. The high-determinant matrices exhibit near-orthogonal Gram structure, with off-diagonal interactions controlled so that rows remain as close to mutually orthogonal as possible under the order-29 constraint (\Cref{fig:math_discovered_solutions}f).

\textbf{Case Analysis.}
The strongest solver expands the initial search beyond a single quadratic-residue construction. It builds a diverse seed pool containing incumbent, quadratic-residue circulant, orthogonal-sign, and cropped Sylvester-style matrices, then alternates inverse-guided multi-flip hill climbing with simulated annealing. This combination allows the program to make coordinated sign changes rather than relying on isolated single-entry flips, which is essential for escaping low-determinant basins while preserving the global Gram structure needed for a near-extremal matrix.

\clearpage

\end{document}